\documentclass[10pt]{article}
\usepackage{setspace}
\usepackage{cprotect}
\usepackage{amsmath, amsfonts, amsthm, amssymb, bm, bbm}
\usepackage[noend]{algorithmic}
\usepackage[ruled,vlined]{algorithm2e}
\usepackage[bookmarks=true,
                    bookmarksopen=true,
                    bookmarksnumbered=true,
                    colorlinks=true,
                    linkcolor=blue,
                    citecolor=ao,
                    ]{hyperref}
\usepackage{fullpage}
\usepackage{makeidx}
\usepackage{enumerate}
\usepackage{graphicx,float,psfrag,epsfig}
\usepackage{epstopdf}
\usepackage{color}
\usepackage{enumitem}
\usepackage{subfig}
\usepackage{caption}
\usepackage{bigints}
\usepackage{mathtools}
\usepackage[mathscr]{euscript}
\usepackage{cite}

\DeclareSymbolFont{rsfs}{U}{rsfs}{m}{n}
\DeclareSymbolFontAlphabet{\mathscrsfs}{rsfs}

\numberwithin{equation}{section}

\newtheoremstyle{myexample} 
    {\topsep}                    
    {\topsep}                    
    {\rm }                   
    {}                           
    {\bf }                   
    {.}                          
    {.5em}                       
    {}  

\newtheoremstyle{myremark} 
    {\topsep}                    
    {\topsep}                    
    {\rm}                        
    {}                           
    {\bf}                        
    {.}                          
    {.5em}                       
    {}  

\newtheorem{claim}{Claim}[section]
\newtheorem{lemma}[claim]{Lemma}

\newtheorem{theorem}{Theorem}

\newtheorem{corollary}[claim]{Corollary}
\newtheorem{definition}[claim]{Definition}
\newtheorem*{theorem*}{Theorem}

\def\btheta{{\boldsymbol \theta}}
\def\bTheta{{\boldsymbol \Theta}}
\def\b0{{\boldsymbol 0}}

\def\E{{\mathbb{E}}}
\DeclarePairedDelimiter\floor{\lfloor}{\rfloor}
\DeclareMathOperator*{\argmin}{arg\,min}
\definecolor{ao}{rgb}{0.0, 0.5, 0.0}
\newcommand{\rev}[1]{\textcolor{black}{#1}}

\title{Mean-field Analysis of Piecewise Linear Solutions \\ for Wide ReLU Networks
}
\date{}
\author{%
  Alexander Shevchenko\thanks{IST Austria, \texttt{alex.shevchenko@ist.ac.at}}
  \and
   Vyacheslav Kungurtsev\thanks{Czech Technical University in Prague, \texttt{vyacheslav.kungurtsev@fel.cvut.cz}}
   \and
  Marco Mondelli\thanks{IST Austria, \texttt{marco.mondelli@ist.ac.at}}
 }

\begin{document}

\maketitle

\begin{abstract}
 Understanding the properties of neural networks trained via stochastic gradient descent (SGD) is at the heart of the theory of deep learning.
In this work, we take a mean-field view, and consider a two-layer ReLU network trained via \rev{noisy-}SGD for a univariate regularized regression problem.
Our main result is that SGD \rev{with vanishingly small noise injected in the gradients} is biased towards a simple solution: at convergence, the ReLU network implements a piecewise linear map of the inputs, and the number of ``knot'' points -- i.e., points where the tangent of the ReLU network estimator changes -- between two consecutive training inputs is at most three. In particular, as the number of neurons of the network grows, the SGD dynamics is captured by the solution of a gradient flow and, at convergence, the distribution of the weights approaches the unique minimizer of a related free energy, which has a Gibbs form. Our key technical contribution consists in the analysis of the estimator resulting from this minimizer: we show that its second derivative vanishes everywhere, except at some specific locations which represent  the ``knot'' points. We also provide empirical evidence that knots at locations distinct from the data points might occur, as predicted by our theory.
\end{abstract}

\section{Introduction}\label{sec:intro}

Neural networks are the key ingredient behind many recent advances in machine learning. They achieve state-of-the-art performance on various practical tasks, such as image classification \cite{he2016deep} and synthesis  \cite{brock2018large}, natural language processing \cite{vaswani2017attention} and reinforcement learning \cite{silver2016mastering}. However, these results would not be possible without computational advances which enabled the training of highly overparameterized models with billions of weights. Such complex networks are capable of extracting more sophisticated patterns from the data than their less parameter-heavy counterparts.  Nonetheless, in the view of classical learning theory, models with a large number of parameters are prone to over-fitting \cite{von2011statistical}. Contrary to~the conventional statistical wisdom, overparameterization turns out to be a rather desirable property for neural networks. \rev{This was even observed in a classical paper by \cite{bart1998}, which demonstrated that in the overparameterized setting, the size of the network is less important than the magnitude of the weights.} More recently, phenomena such as double descent \cite{belkin2018reconciling,spigler2018jamming,nakkiran2019deep} and benign overfitting \cite{bartlett2020benign,li2021towards,bartlett2021deep} suggest that understanding the generalization properties of overparameterized models lies beyond the scope of the usual control of capacity via the size of the parameter set \cite{neyshabur2014search}. 

One way to explain the generalization capability of large neural networks lies in characterizing the properties of solutions found by stochastic gradient descent (SGD). In other words, the question is whether the optimization procedure is implicitly selective, i.e., it finds the functionally simple solutions that exhibit superior generalization ability in comparison to other candidates with roughly the same value of the empirical risk. For instance, \cite{chizat2020implicit} consider shallow networks minimizing the logistic loss, and show that SGD converges to a max-margin classifier on a certain functional space endowed with the variation norm. 
In the machine learning literature, it has been suggested that large margin classifiers inherently exhibit better performance on unseen data \cite{bartlett2021deep,cortes1995support}.

Constraints on the functional class of network solutions can also be imposed explicitly, e.g., via $\ell_2$ regularization or by adding label noise. In some cases, it has been shown that the presence of parameter penalties or noise results in surprising implications. Depending on the regime, it biases optimization to find smooth solutions \cite{sahs2020shallow,jin2020implicit,savarese2019infinite} or piecewise linear functions \cite{blanc2020implicit,ergen2020convex}. The study by \cite{balestriero2018spline} proposes an alternative to conventional $\ell_p$ regularization inspired by max-affine spline operators. It enforces a neural network to learn orthogonal representations, which significantly improves the performance and does not require any modifications of the network architecture.

In this work, we develop a novel approach towards understanding the implicit bias of gradient descent methods applied to overparameterized neural networks. In particular, we focus on the following key questions:

\vspace{0.4em}
\begin{center}
\begin{minipage}{0.85\textwidth}
\textit{
\hspace{-0.5em}Once stochastic gradient descent has converged, how does the distribution of the weights of the neural network look like?  What functional properties of the resulting solution are induced by this stationary distribution? Can we quantitatively characterize the trade-off between the complexity of the solution and the size of the training data in the overparameterized regime?
}
\end{minipage}
\end{center}
\vspace{0.4em}

To answer these questions, we consider training a wide two-layer ReLU (rectified linear unit) network for univariate regression, and we focus on the mean-field regime \cite{mei2018mean,rotskoff2018neural,chizat2018global,sirigano2020mean}. In this regime, the idea is that, as the number of neurons of the network grows, the weights obtained via SGD are close to i.i.d. samples coming from the solution of a certain Wasserstein gradient flow. As a consequence, the output of the neural network approaches the following quantity:
$$
y^{\sigma^{*}}_{\rho}(x) = \int \sigma^{*}(x,\btheta) \rho(\btheta)\mathrm{d}\btheta.
$$
Here, $x$ is the input, $\sigma^{*}$ denotes the activation function, and $\rho$ is
the solution of the Wasserstein gradient flow minimizing the free energy
\begin{equation}\label{eq:fenergy}
\mathcal{F}(\rho) = \frac{1}{2}\hspace{0.2em}\E_{(x,y)\sim\mathbb{P}}\big\{(y - y^{\sigma^{*}}_{\rho}(x))^2\big\} + \frac{\lambda}{2} \int \|\btheta\|_2^2 \rho(\btheta)\mathrm{d}\btheta + \beta^{-1} \int \rho(\btheta)\log \rho(\btheta)\mathrm{d}\btheta.
\end{equation}
The first term corresponds to the expected squared loss (under the data distribution $\mathbb P$); the second term comes from the $\ell_2$ regularization; the differential entropy term is linked to the noise introduced into the SGD update,
and it penalizes non-uniform solutions. The coefficient $\beta$ is often referred to as \textit{inverse temperature}. In \cite{mei2018mean}, it is also shown that the minimizer of the free energy, call it $\rho^{*}$, has a Gibbs form for a sufficiently regular activation function $\sigma^{*}$. We review the connection between the dynamics of gradient descent and the solution $\rho$ of the Wasserstein gradient flow in Section \ref{section:background}.

\begin{figure}[t!]
    \subfloat[]{\includegraphics[width=.34\columnwidth]{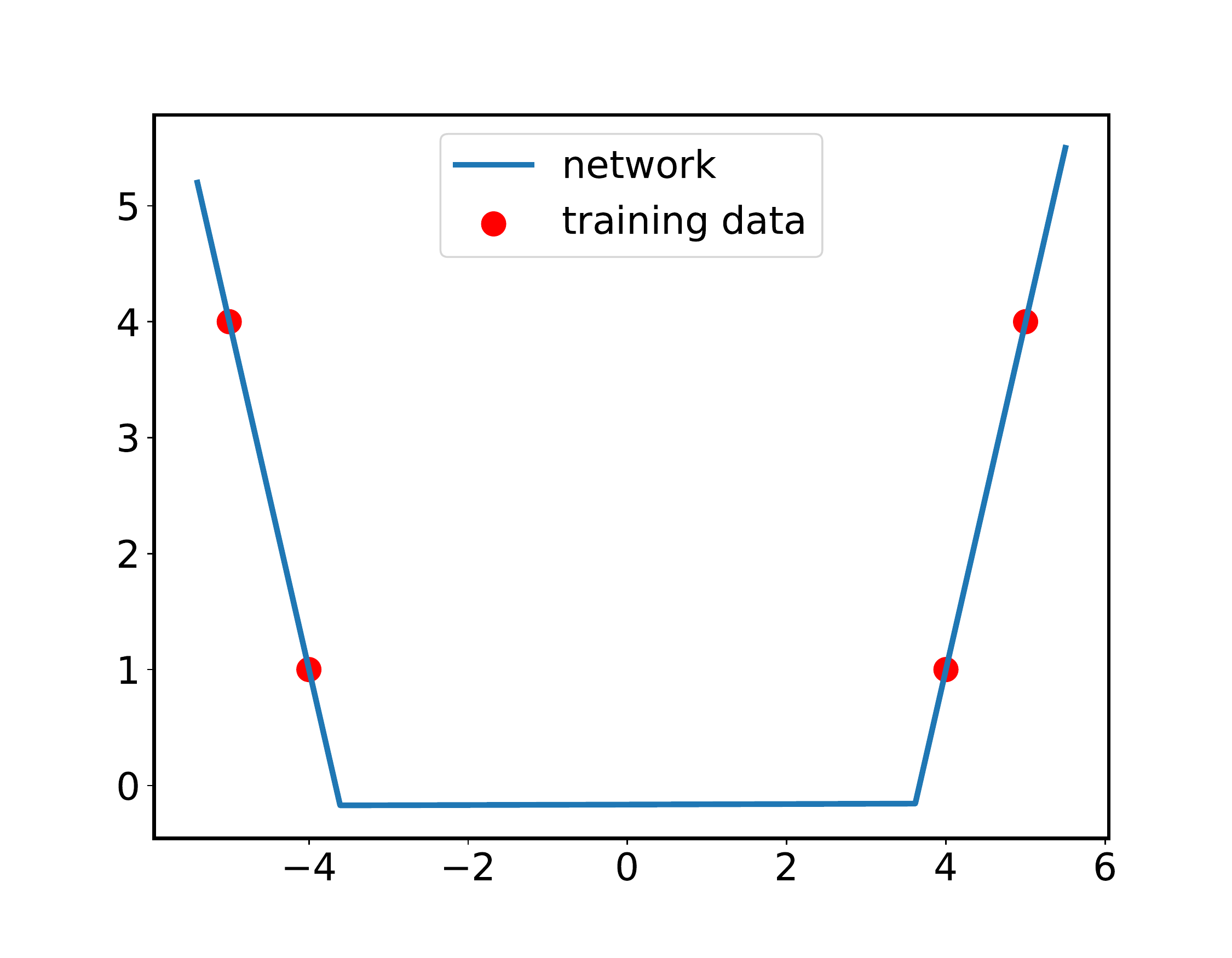}\label{f1:a}}\
    \subfloat[]{\includegraphics[width=.34\columnwidth]{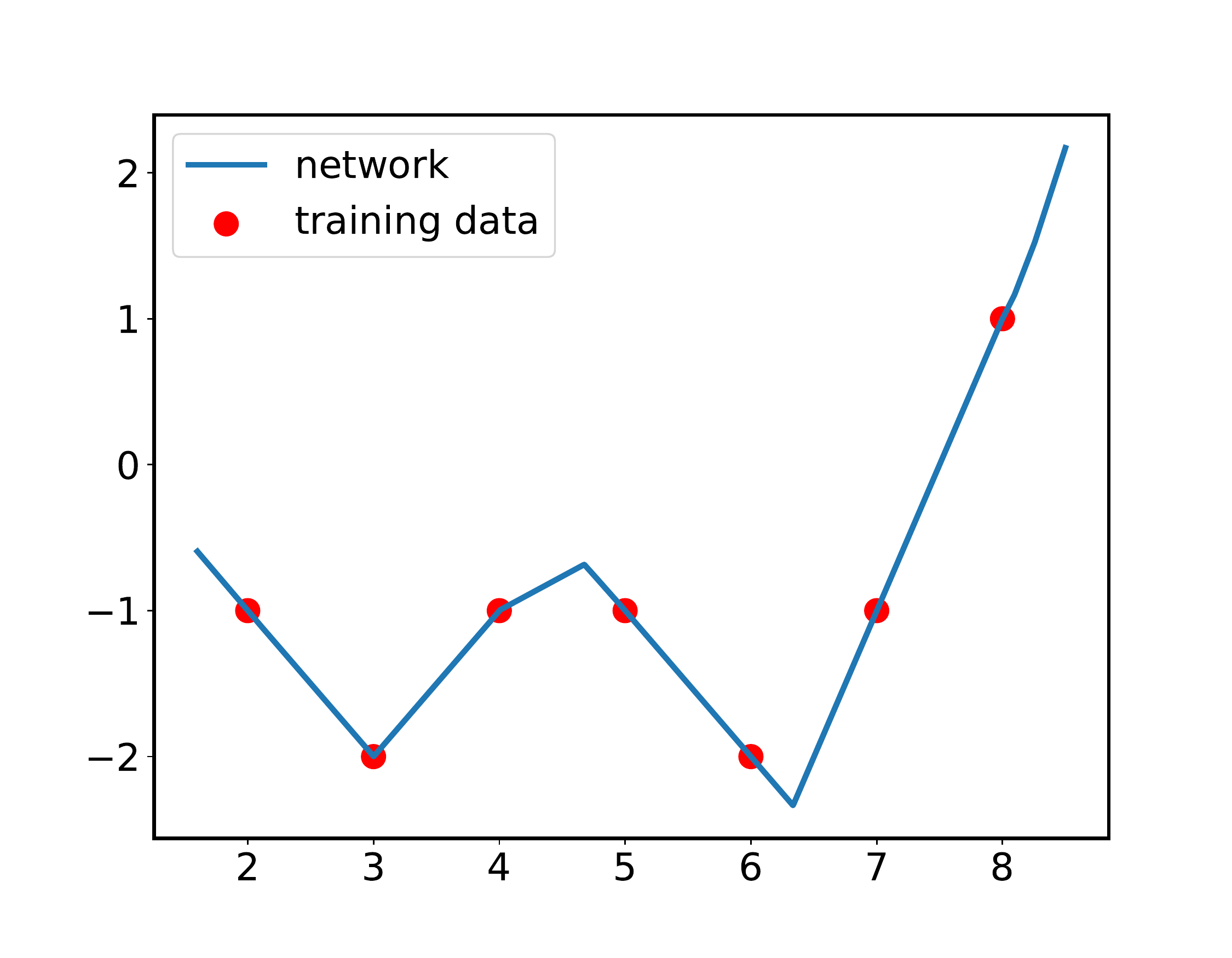}\label{f1:b}}
    \subfloat[]{\includegraphics[width=.34\columnwidth]{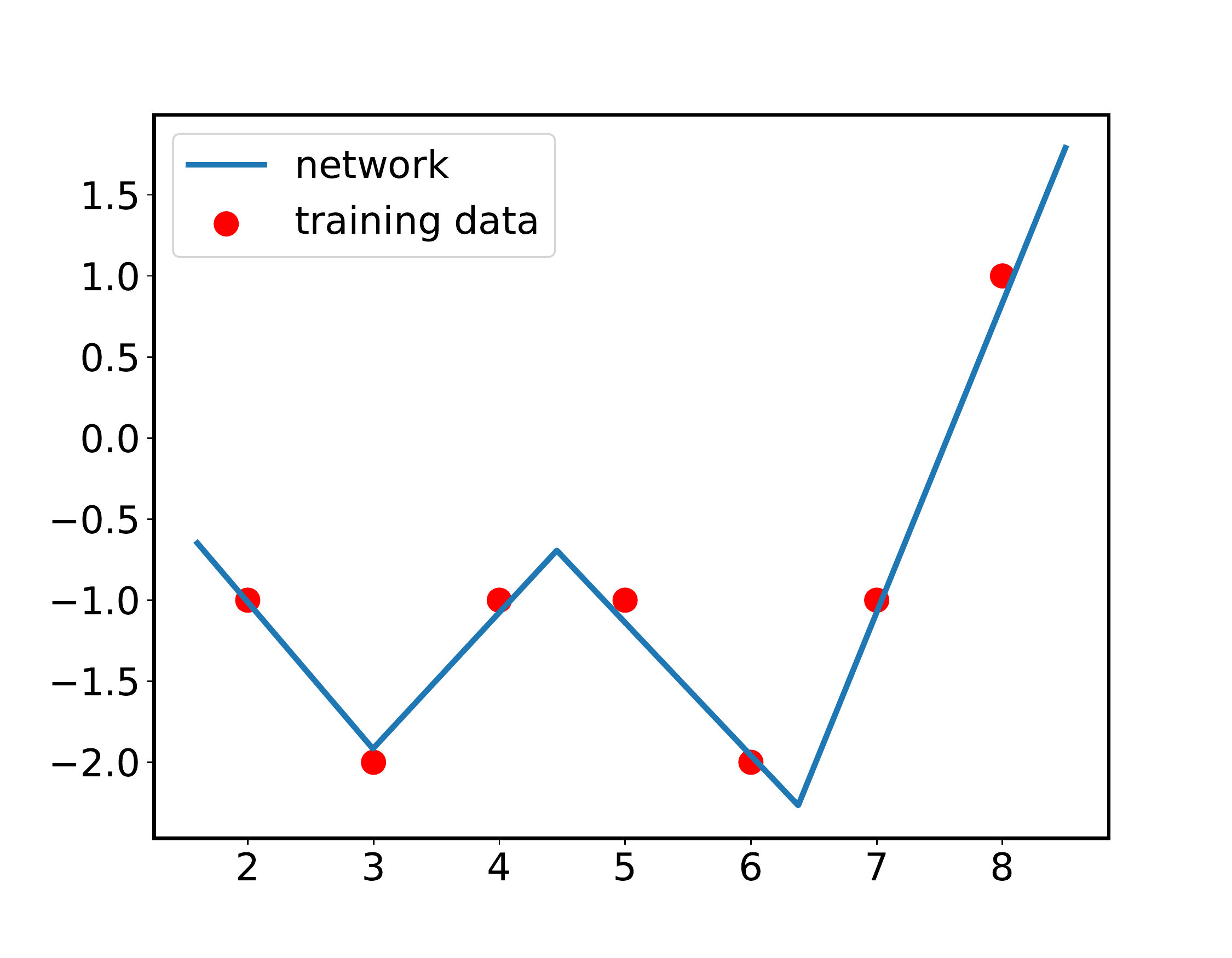}\label{f1:c}}
\caption{Example of functions learnt by a two-layer ReLU network with $N=1000$ neurons on different training data. Solutions (a)-(b) are obtained with \textit{no} regularization and label noise, i.e., $\lambda=0$ and $\beta=+\infty$, while in (c) we have a sufficiently large regularization coefficient, which does not allow the network to fit the training data perfectly. Note that the piecewise linear solution exhibits tangent changes also at points different from the training data. Furthermore, the number of ``knot'' points may differ from the minimum required to fit the data: for instance, in (a) the minimum amount of tangent changes is $1$, but the solution has two of them.}\label{fig:intro}
\vspace{-3mm}
\end{figure}

A number of works has exploited this connection to provide a rigorous justification to various phenomena attributed to neural networks. \cite{mei2018mean,mei2019mean} give global convergence guarantees for two-layer networks by studying the energy dissipation along the trajectory of the flow. The paper by \cite{chizat2018global} takes a different route and exploits a lifting property enabled by a certain type of initialization and regularization, and \cite{javanmard2020analysis} put forward an argument based on displacement convexity. \cite{nguyen2020rigorous} and \cite{araujo2019mean} tackle the multi-layer case, and, in particular, \cite{nguyen2020rigorous} establish convergence guarantees for a three-layer network. \cite{fang2021modeling} introduce a mean-field dynamics capturing the evolution of the features (instead of the network parameters) and show global convergence of ResNet type of architectures. 
\cite{shevchenko2020landscape} prove two properties commonly observed in practice (see e.g. \cite{garipov2018loss,draxler2018essentially,kuditipudi2019explaining}), namely dropout stability and mode connectivity, for multi-layer networks trained under the mean-field regime. 
{\cite{de2020quantitative} consider different scalings of the step size of SGD, and identify two regimes under which different mean-field limits are obtained.}
\cite{williams2019gradient} show that the gradient flow for unregularized objectives forces the neurons of a two-layer ReLU network to concentrate around a subset of the training data points.

In this paper, we take a mean-field view to show that SGD is biased towards functionally simple solutions, namely, piecewise linear functions. Our idea is to analyze the stationary distribution $\rho^{*}$ minimizing the free energy \eqref{eq:fenergy}. We show that, in the low temperature regime ($\beta\rightarrow\infty$), the estimator's curvature vanishes everywhere except for a certain \textit{cluster set}. More precisely, for each interval between two consecutive training inputs, aside for a set of small measure, the second derivative vanishes, i.e.,
$$
\frac{\partial^2}{\partial x^2}\hspace{0.3em} y^{\sigma^{*}}_{\rho^{*}}(x) \rightarrow 0,\quad \mbox{ as }\beta\to\infty.
$$
Furthermore, we provide a characterization of the cluster set and show that its measure vanishes while it concentrates around at most 3 points per interval. Ultimately, this analysis guarantees that, in the regime of decreasing temperature (corresponding to a small noise injected in the gradient updates), the solution found by SGD is piecewise linear. Our main contribution can be summarized in the following informal statement:

\begin{theorem*}[Informal] Under the low temperature regime, i.e., $\beta\rightarrow\infty$, the estimator obtained by training a two-layer ReLU network via noisy-SGD converges to a piecewise linear solution. Furthermore, the number of ``knot'' points -- i.e., points at which distinct linear pieces connect -- between two consecutive training inputs is at most 3.
\end{theorem*}

Let us remark on a few important points. In the overparameterized regime, the number of neurons $N$ is significantly larger than the number of training samples $M$, i.e., $N \gg M$. The output of the two-layer ReLU network is a linear combination of $N$ ReLU units, hence the function implemented by the network is clearly piecewise linear with $\mathcal{O}(N)$ knot points. Here, we show that the number of knot points is actually $\mathcal{O}(M) \ll \mathcal{O}(N)$.
Our analysis applies for both constant ($\lambda\to\bar{\lambda}>0$) and vanishing  ($\lambda\rightarrow 0$) regularization, and it does not require a specific form for the initialization of the parameters of the networks (as long as some mild technical conditions are satisfied).

In a nutshell, we establish a novel technique that accurately characterizes the solution to which gradient descent methods converge, when training overparameterized two-layer ReLU networks.
Our analysis unveils a behaviour which is qualitatively different from that described in recent works \cite{williams2019gradient, blanc2020implicit, ergen2020convex} (see also a detailed comparison in Section \ref{sec:discussion}): knot points are not necessarily allocated at the training points, or in a way that results in a function with the minimum number of tangent changes required to fit the data. We provide also numerical simulations to validate our findings (see Section \ref{section:numsim} and Figure \ref{fig:intro} above).
We suggest that this novel behaviour is likely due to the difference in settings and the additional $\ell_2$ regularization (including of the bias parameters). 

\textbf{Organization of the paper.} The rest of the paper is organized as follows. In Section \ref{section:related}, we review the related work and a more detailed comparison is deferred to Section \ref{sec:discussion}. In Section \ref{sec:prel}, we provide some preliminaries, including a background on the mean-field analysis in Section \ref{section:background}. Our main results are stated in Section \ref{section:mainres} and proved in Section \ref{section:proofsketch}. In Section \ref{section:knotsarethere}, we provide an example of a dataset for which the estimator found by SGD has a knot at a location different from the training inputs. We validate our findings with numerical simulations for different regression tasks in Section \ref{section:numsim}. We conclude and discuss some future directions in Section \ref{conclusion}.
Some of the technical lemmas and the corresponding proofs are deferred to Appendix \ref{techlemmas}.

\textbf{Notation.} We use bold symbols for vectors $\boldsymbol{a}, \boldsymbol{b}$, and plain symbols for real numbers $a,b$. We use capitalized bold symbols to denote matrices, e.g., $\bTheta$.  We denote the $\ell_2$ norm of vectors $\boldsymbol{a}, \boldsymbol{b}$ by $\|\boldsymbol{a}\|_2, \|\boldsymbol{b}\|_2$. Given an integer $N$, we denote $[N] = \{1,\dots,N\}$. Given a discrete set $\mathcal{A}$, $|\mathcal{A}|$ is its cardinality. Similarly, given a Lebesgue measurable set $\mathcal{B} \subset \mathbb{R}^d$ its Lebesgue measure is given by $|\mathcal{B}|$. Given a sequence of distributions $\{\rho_n\}_{n\geq 0}$, we write $\rho_n \rightharpoonup \rho$ to denote the weak $L_1$ convergence of the corresponding measures. For a sequence of functions $\{f_n\}_{n\geq 0}$ we denote by $f_n \rightarrow f$ the \textit{pointwise} convergence to a function $f$. Given a real number $x\in\mathbb{R}$, the closest integer that is not greater than $x$ is defined by $\floor*{x}$.

\section{Related Work}\label{section:related}

The line of works \cite{williams2019gradient,jin2020implicit} shows that, in the lazy training regime \cite{chizat2018lazy,jacot2018neural} and for a uniform initialization, SGD converges to a cubic spline interpolating the data. Furthermore, for multivariate regression \rev{in the lazy training regime, \cite{jin2020implicit} proved} that the optimization procedure is biased towards solutions minimizing the 2-norm of the Radon transform of the fractional Laplacian. 
Similar results (although without the connection to the training dynamics) are obtained in \cite{savarese2019infinite,ongie2019function}, which analyze the solutions with zero loss and minimum norm of the parameters. \cite{ergen2020convex} develop a convex analytic framework to explain the bias towards simple solutions. In particular, an explicit characterization of the minimizer is provided, which implies that an optimal set of parameters yields linear spline interpolation for regression problems involving one dimensional or rank-one data. {\cite{cao2021towards} show that, for overparameterized models, the lower degree spherical harmonics are easier to learn. This observation comes from the fact that, in the lazy training regime, the convergence occurs faster along the directions given by the top eigenfunctions of the neural tangent kernel.}
Classification with linear networks on separable data is considered in \cite{soudry2018implicit}, where it is shown that gradient descent converges to the max-margin solution. This max-margin behavior is demonstrated in \cite{chizat2020implicit} for non-linear wide two-layer networks using a mean-field analysis. In particular, in the mean-field regime, optimizing the logistic loss is equivalent to finding the max-margin classifier in a certain functional space. The paper by \cite{zhang2020type} focuses on the lazy training regime, and it shows that the optimization procedure finds a solution that fits the data perfectly and is closest to the starting point of the dynamics in terms of Euclidean distance in the parameter space. { \cite{wu2021direction} characterize the directional bias of GD and SGD in the case of moderate (but annealing) learning rate.}

The behavior of SGD with label noise near the zero-loss manifold is studied in \cite{blanc2020implicit}. Here, it is shown that the training algorithm implicitly optimizes an auxiliary objective, namely, the sum of squared norms of the gradients evaluated at each training sample. This allows the authors of \cite{blanc2020implicit} to show that SGD with label noise for a two-layer ReLU network with skip-connections is biased towards a piecewise linear solution. In particular, this piecewise linear solution has the minimum number of tangent changes required to fit the data. \cite{williams2019gradient} consider the Wasserstein gradient flow on a certain space of reduced parameters (in polar coordinates), and show that 
the points where the solution changes tangent are concentrated around a subset of training examples. A trade-off between the scale of the initialization and the training regime is also provided in \cite{williams2019gradient, sahs2020shallow}. \cite{maennel2018gradient} prove that the gradient flow enforces the weight vectors to concentrate at a small number of directions determined by the
input data. Through the lens of spline theory, \cite{parhi2020role} explain that a number of best practices used in deep learning, such as weight decay and path-norm, are connected to the ReLU activation and its smooth counterparts. \cite{neyshabur2018towards} suggest a novel complexity measure for neural networks that provides a tighter generalization for the case of ReLU activation.

\section{Preliminaries}\label{sec:prel}

\subsection{Mean-field Background}\label{section:background}

We consider a two-layer neural network with $N$ neurons and one-dimensional input $x\in\mathbb{R}$:
\begin{equation}\label{eq:NN}
\hat{y}_N(x,\bTheta) = \frac{1}{N}\sum_{i=1}^N \sigma^{*}(x, \btheta_i), 
\end{equation}
where $\hat{y}_N(x,\bTheta)\in\mathbb{R}$ is the output of the network, $\bTheta = \left(\btheta_1,\dots,\btheta_N\right)\in\mathbb{R}^{D\times N}$, with $\btheta_i \in \mathbb{R}^{D}$, are the parameters of the network, $D$ is the dimension of parameters of each neuron, and $\sigma^{*}: \mathbb{R}\times\mathbb{R}^{D}\rightarrow \mathbb{R}$ represents the activation function. A typical example is $\sigma^{*}(x,\btheta) = a(wx + b)_{+}$, where $\btheta=(a,w,b)\in \mathbb{R}^3$ and $(\cdot)_{+}:\mathbb{R}\rightarrow\mathbb{R}$ is a rectified linear unit activation.

We consider a regression problem for a dataset $\{(x_j, y_j)\}_{j=1}^M$ containing $M$ points, and we aim to minimize the following expected squared loss with $\ell_2$ regularization:
\begin{equation}\label{eq:loss}
\E \Big\{(\hat{y}_N(x,\bTheta) - y)^2\Big\} +  \frac{\lambda}{N} \sum\limits_{i=1}^N \|\btheta_i\|^2_{2} = \frac{1}{M}\sum_{j=1}^M (\hat{y}_N(x_j,\bTheta) - y_j)^2 +  \frac{\lambda}{N} \sum\limits_{i=1}^N \|\btheta_i\|^2_{2}.
\end{equation}
On the LHS of \eqref{eq:loss}, the expectation is taken over $(x,y) \sim \mathbb P$, with $\mathbb P = M^{-1}\sum_{j=1}^M \delta_{(x_j,y_j)}$ and $x_j < x_{j+1}\ \forall{j\in[M-1]}$. Here, $\delta_{(a,b)}$ stands for a delta distribution centered at $(a,b)\in\mathbb{R}^2$.

We are given samples $(\tilde{x}_k,\tilde{y}_k)_{k\geq 0}\sim_{\rm i.i.d.} \mathbb P$, and we learn the network parameters $\bTheta$ via stochastic gradient descent (SGD) with step size $s_k$ and additive Gaussian noise scaled by a factor $\beta^{-1} > 0$ (often referred to as a temperature):
\begin{equation}\label{eq:SGD}
{\boldsymbol{\theta}}_{i}^{k+1}=(1-2\lambda s_k){\boldsymbol{\theta}}_{i}^{k}+2 s_k \big(\tilde{y}_{k}-\hat{y}_N(\tilde{x}_k,\bTheta^k)\big) \nabla_{{\boldsymbol{\theta}}_{i}} \big(\sigma^{*}(\tilde{x}_k,\btheta^k_i)\big) + \sqrt{2s_k/\beta} \boldsymbol{g}_i^k,
\end{equation}
where $\bTheta^k$ stands for the network parameters after $k$ steps of the optimization procedure, $\{\boldsymbol{g}_i^k\}_{i\in[N], k\geq 0} \sim_{\textrm{i.i.d.}} \mathcal{N}(0, \boldsymbol{I}_{D})$, and the term $-2\lambda s_k \boldsymbol{\theta}_{i}^{k}$ corresponds to $\ell_2$ regularization. The parameters are initialized independently according to a given distribution $\rho_0$, i.e., $\{\btheta_i^0\}_{i\in[N]} \sim_{\textrm{i.i.d.}} \rho_0$.

For some $\varepsilon > 0$, we assume that the step size of the noisy SGD update \eqref{eq:SGD} is given by $s_k=\varepsilon\xi(\varepsilon k)$, where $\xi: \mathbb{R}_{\geq 0}\rightarrow\mathbb{R}_{\geq 0}$ is a sufficiently regular function. Let $\hat{\rho}^{N}_k := \frac{1}{N} \sum_{i=1}^N \delta_{\btheta_i^k}$ denote the empirical distribution of weights after $k$ steps of noisy SGD. Then, in \cite{mei2018mean}, it is proved that the evolution of $\hat{\rho}^{N}_k$ is well approximated by a certain distributional dynamics. In formulas,
$$
\hat{\rho}^{(N)}_{\floor*{t/\varepsilon}} \rightharpoonup \rho_t,
$$    
almost surely along any sequence $(N\rightarrow\infty, \varepsilon_N\rightarrow 0)$ such that $N/\log\left(N/\varepsilon_N\right)\rightarrow\infty$ and $\varepsilon_N\log\left(N/\varepsilon_N\right)\rightarrow 0$. (Here, we have put the subscript $N$ in $\varepsilon_N$ to emphasize that the choice of the learning rate depends on $N$.) The distribution $\rho_t$ is the solution of the following partial differential equation (PDE):
\begin{align}\label{eq:PDE}
    \partial_t \rho_t &= 2 \xi(t) \nabla_{\btheta} \cdot (\rho_t \nabla_{\btheta}\Psi_{\lambda}(\btheta, \rho_t)) + 2\xi(t)\beta^{-1}\Delta_{\btheta} \rho_t,\nonumber \\
    \Psi_{\lambda}(\btheta, \rho):&= \frac{1}{M} \sum_{i=1}^M \big(y^{\sigma^{*}}_{\rho}(x_i)-y_i\big) \cdot \sigma^{*}(x_i,\btheta) + \frac{\lambda}{2} \|\btheta\|_{2}^2.
\end{align}
Here, $\nabla_{\theta} \cdot \boldsymbol{v}(\btheta)$ stands for the divergence of the vector field $\boldsymbol{v}(\btheta)$, and $\Delta_{\btheta}f(\btheta)=\sum_{j=1}^{D} \partial^2_{\theta_j}f(\btheta)$ is the Laplacian of the function $f: \mathbb{R}^D\rightarrow\mathbb{R}^D$. To describe the next result, we first introduce a few related quantities. Define the infinite-width network with activation $\sigma^{*}: \mathbb{R}\times\mathbb{R}^D \rightarrow \mathbb{R}$ and weight distribution $\rho: \mathbb{R}^{D} \rightarrow [0,+\infty)$ as follows:
$$
y^{\sigma^{*}}_{\rho}(x) = \int \sigma^{*}(x,\btheta) \rho(\btheta)\mathrm{d}\btheta,
$$
where the integral is taken over the support of $\rho$. For the forthcoming analysis, a certain regularity is required for the weight distribution $\rho$. In particular, the weight distribution is restricted to a set of admissible densities
$$
{\mathcal{K} := \left\{\rho: \mathbb{R}^D \rightarrow [0,+\infty) \text{ measurable: } \int \rho(\btheta)\mathrm{d}\btheta = 1, \ M(\rho) < \infty\right\},}
$$
where $M(\rho) = \int \|\btheta\|^2_2\rho(\btheta)\mathrm{d}\btheta$. The expected risk attained on the distribution $\rho$ by the infinite-width network with activation $\sigma^{*}$ is defined by
$$
R^{\sigma^{*}}(\rho) := \frac{1}{M}\sum_{i=1}^M \big(y^{\sigma^{*}}_{\rho}(x_i) - y_i\big)^2.
$$
The quantity
$$
H(\rho) := -\int \rho(\btheta)\log \rho(\btheta)\mathrm{d}{\btheta}
$$
stands for the differential entropy of $\rho$, which is equal to $-\infty$ if the distribution $\rho$ is singular. In this view, the distributional dynamics (\ref{eq:PDE}) is the Wasserstein gradient flow that minimizes the free energy
\begin{align}\label{eq:free_energy}
    \mathcal{F}^{\sigma^{*}}(\rho) &= \frac{1}{2}R^{\sigma^{*}}(\rho) + \frac{\lambda}{2}M(\rho) - \beta^{-1}H(\rho),
\end{align}
over the set of admissible densities $\mathcal{K}$. Furthermore, this free energy has a unique minimizer and the solution of (\ref{eq:PDE}) converges to it as $t\rightarrow\infty$:
\begin{align*}
    \rho_t \rightharpoonup \rho^{*}_{\sigma^{*}} \in \argmin_{\rho'\in\mathcal{K}}\mathcal{F}^{\sigma^{*}}(\rho'), \quad \mbox{ as }t\to\infty.
\end{align*}
The unique minimizer $\rho^{*}_{\sigma^{*}}$ is absolutely continuous, and it has the Gibbs form
\begin{equation}\label{eq:gibbs}
    \rho^{*}_{\sigma^{*}}(\btheta) = \frac{\exp\big\{-\beta \Psi_{\lambda}(\btheta, \rho^{*}_{\sigma^{*}})\big\}}{Z_{\sigma^{*}}(\beta,\lambda)} ,
\end{equation}
where $Z_{\sigma^{*}}(\beta,\lambda)$ is the normalization constant, also referred to as partition function.

\subsection{Approximation of the ReLU Activation}\label{section:truncationact}

Let us elaborate on the properties which $\sigma^{*}$ should satisfy so that the results of Section \ref{section:background} hold. First, the distributional dynamic (\ref{eq:PDE}) is known to be well-defined for a smooth and bounded potential $\Psi_{\lambda}$. In particular, it suffices to choose a bounded, Lipschitz $\sigma^{*}$ with Lipschitz gradient, see assumptions A2-A3 in \cite{mei2018mean}. Furthermore, the minimizer of the free energy (\ref{eq:free_energy}) exists and has a Gibbs form even for non-smooth potentials and, in particular, it suffices that $\sigma^{*}$ is bounded and Lipschitz (this allows the first derivative to be discontinuous), see Lemmas 10.2-10.4 in \cite{mei2018mean}.

In the case of a ReLU activation, the corresponding $\sigma^{*}$ has the following form
$$
\sigma^{*}(x,\btheta) = a (wx + b)_{+} = a \max\{0, wx+b\},\quad \btheta = (a,w,b)\in\mathbb{R}^3,
$$
which does not satisfy some of the aforementioned conditions. The first salient problem is the lack of continuity of the derivative at zero. This issue can be dealt with by considering a soft-plus activation with scale $\tau$:
$$
(x)_{\tau} := \frac{\log(1 + e^{\tau x})}{\tau}.
$$
Notice that, as $\tau$ grows large, we have that $(\cdot)_{\tau} \rightarrow (\cdot)_{+}$. Another issue is that the function $\sigma^{*}(x,\btheta)$ is not Lipschitz in the parameters $\btheta$, and it is unbounded. This problem can be solved by an appropriate truncation applied to the parameter $a$ of the activation. The truncation should be Lipschitz and smooth for the dynamics to be well-defined.

In this view, we now provide the details of the approximation of the ReLU activation. For a parameter $v \in \mathbb{R}$, we denote by $v^m$ its $m$-truncation defined as
$$
v^m := \mathbbm{1}_{\{|v| > m\}} \cdot m\cdot\mathrm{sign}(v) + \mathbbm{1}_{\{|v| \leq m\}} \cdot v.
$$
Notice that the function $f(v)=v^m$ is Lipschitz continuous and bounded. For a parameter $v \in \mathbb{R}$, we denote by $v^{\tau,m}$ its $\tau$-smooth $m$-truncation defined as follows: $v^{\tau,m}$ converges pointwise to $v^m$ as $\tau \rightarrow \infty$, $v^{\tau,m}=v$ inside the ball $\{v\in\mathbb{R}: |v|<m-\frac{1}{\tau}\}$, and the map $v \mapsto v^{\tau,m}$ is odd and belongs to $C^4(\mathbb{R})$. For a visualization of $v^m$ and $v^{\tau, m}$, see Figure \ref{subfig:vm}. 

We define the smooth $m$-truncation $(\cdot)_{+}^{m}$ of the ReLU activation as
$$
(x)_{+}^{m} := \mathbbm{1}_{\{x\leq m^2\}} (x)_{+} + \mathbbm{1}_{\{x > m^2\}}\phi_{m}(x),
$$
where $\phi_m$ is chosen so that the following holds: $(x)_{+}^{m}\in C^4(\mathbb{R})$, $(x)^m_{+} \leq (x)_{+}$ for all $x \in \mathbb{R}$, and $|\phi''_m(x)| \leq \frac{1}{m^2}$ for $x > m$. Note that these conditions imply that $\phi_m(m^2) = m^2$ and $\phi_m'(m^2) = 1$. Furthermore, in order to enforce the bound on $\phi''_m$, we pick $\phi_m$ so that $\lim_{x\rightarrow+\infty}\phi_m(x) =2m^2$, and $\lim_{x\rightarrow+\infty}\phi_m'(x) = \lim_{x\rightarrow+\infty}\phi_m''(x) = 0$. For a visualization of $(\cdot)_{+}^{m}$, see Figure \ref{subfig:plus}.

Finally, we define the smooth $m$-truncation $(\cdot)_{\tau}^{m}$ of the softplus activation as
\begin{equation}\label{act_taum}
    (x)_{\tau}^{m} := \mathbbm{1}_{\{x\leq x_m\}} (x)_{\tau} + \mathbbm{1}_{\{x > x_m\}}\phi_{\tau,m}(x),
\end{equation}
where $x_m \in \mathbb{R}$ is such that $(x_m)_{\tau} = m^2$. As in the truncation of ReLU, we choose $\phi_{\tau,m}$ so that $(x)_{\tau}^{m}\in C^{4}(\mathbb{R})$ and $
|\phi''_{\tau,m}(x)| \leq \frac{1}{m^2}$ for $x > x_m$. Furthermore, we require that $(\cdot)_{\tau}^{m}$ converges pointwise to $(\cdot)_{+}^{m}$ as $\tau\to\infty$ (which we can guarantee since $(\cdot)_{\tau} \rightarrow (\cdot)_{+}$, as $\tau\rightarrow\infty$). To enforce these conditions, we pick $\phi_{\tau,m}$ so that $\phi_{\tau,m}(x_m) = m^2$, $\phi_{\tau,m}'(x_m) = (x)'_{\tau}\big|_{x=x_m}$, $\lim_{x\rightarrow+\infty}\phi_{\tau,m}(x) =2m^2$, and $\lim_{x\rightarrow+\infty}\phi_{\tau,m}'(x) = \lim_{x\rightarrow+\infty}\phi_{\tau,m}''(x) = 0$. For a visualization of $(\cdot)_{\tau}^{m}$, see Figure \ref{subfig:mplus}. 

Notice that, for $\tau\geq 1$, the soft-plus activation can be sandwiched as follows:
$$
(x)_{+} - \frac{1}{\tau} \leq (x)_{\tau} \leq (x)_{+} + \frac{1}{\tau}.
$$
In order to establish the continuity of a certain limit and smoothness properties, we also pick $\phi_{\tau,m}$ such that the smooth $m$-truncation of soft-plus activation satisfies a similar bound:
\begin{equation}\label{eq:sandwich}
    (x)_{+} - \frac{1}{\tau} \leq (x)_{\tau}^m \leq (x)_{+} + \frac{1}{\tau}.
\end{equation}

\begin{figure}[t!]
    \centering
    \subfloat[$v^m$ and $v^{\tau,m}$\label{subfig:vm}]{\includegraphics[width=.33\columnwidth]{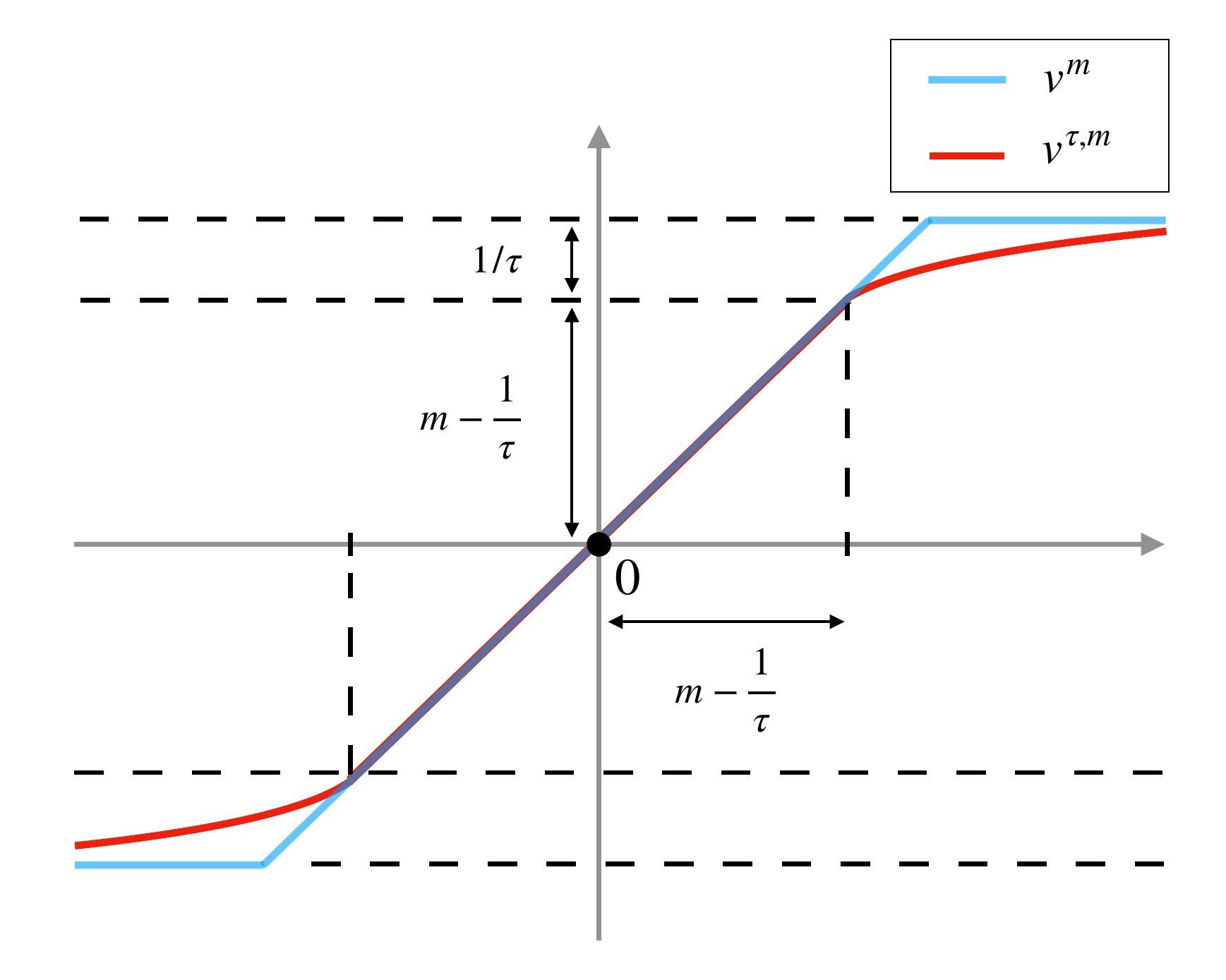}}
    \subfloat[$(\cdot)_+$ and $(\cdot)^{m}_{+}$\label{subfig:plus}]{\includegraphics[width=.33\columnwidth]{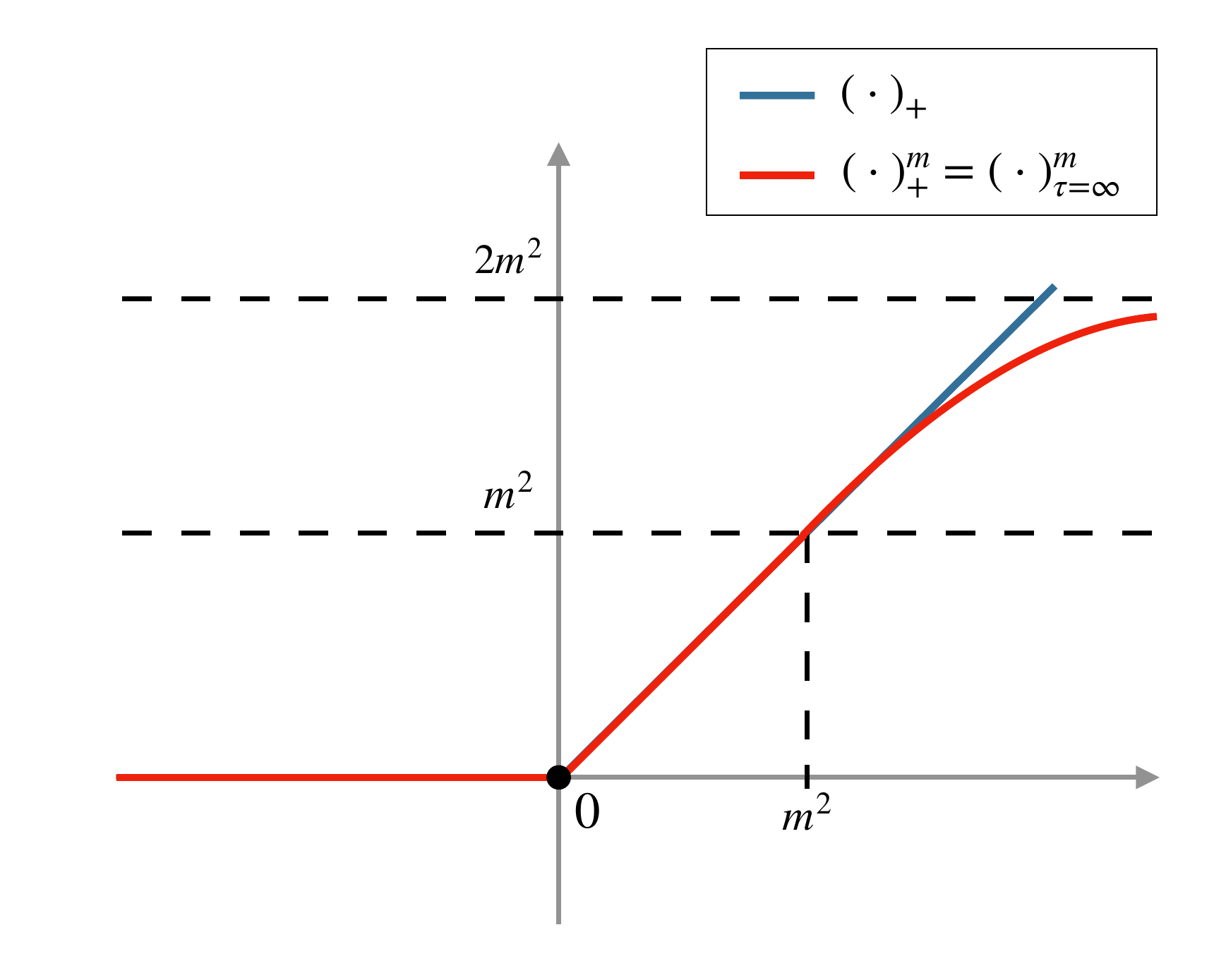}}
        \subfloat[$(\cdot)_+$ and $(\cdot)^{m}_{\tau}$\label{subfig:mplus}]{\includegraphics[width=.33\columnwidth]{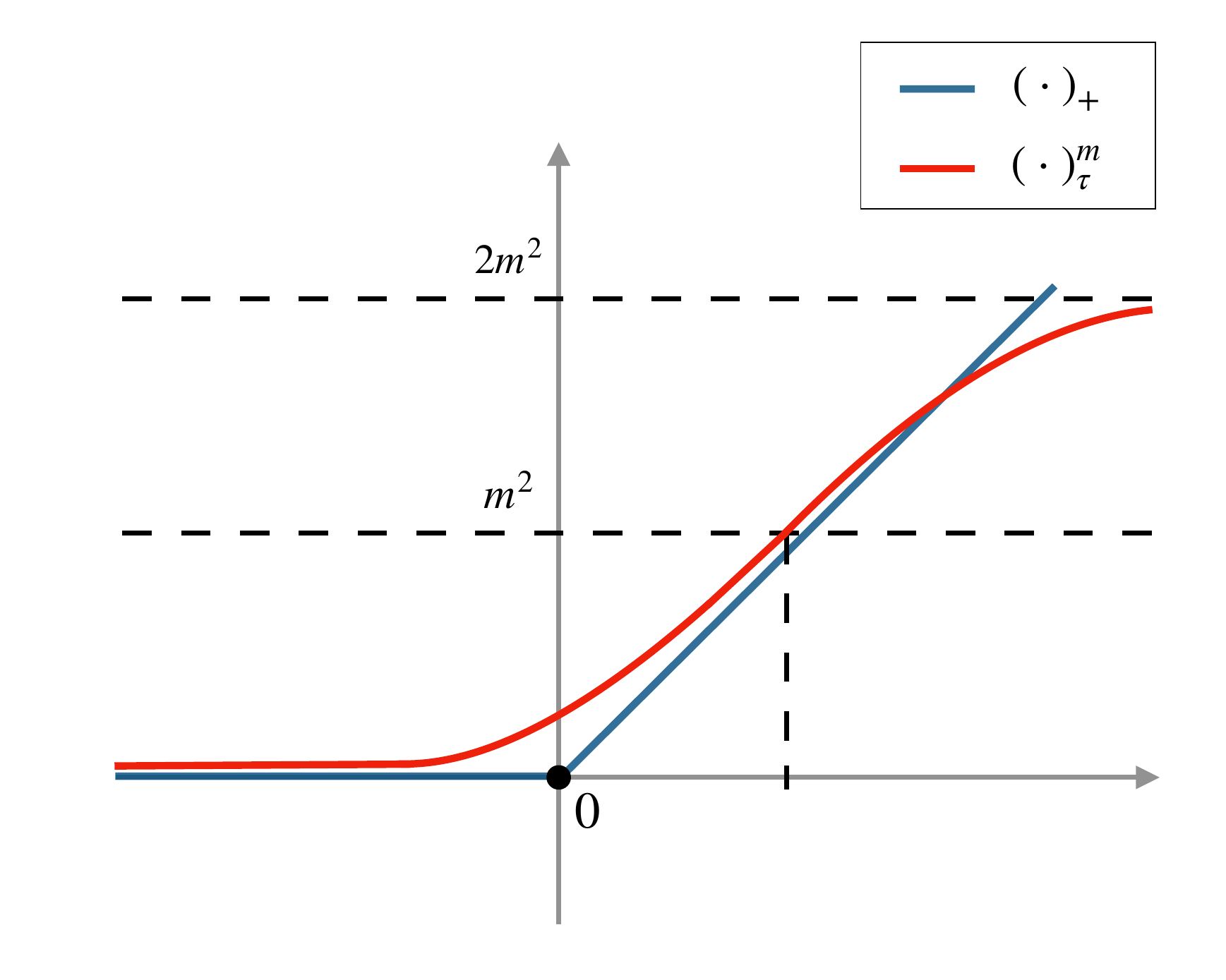}}
\caption{Visualization of the functions involved in the approximation of the ReLU activation.}\label{fig:approx}
\end{figure}

At this point, we remark that the activation $(\btheta,x) \mapsto a^{\tau,m}(w^{\tau,m} x + b)^m_{\tau}$ satisfies all the conditions necessary for the results of Section \ref{section:background} to hold. In what follows, we will also use the activation $(\btheta,x) \mapsto a^m(w^m x + b)^m_{\tau}$ as an auxiliary object. This map is not smooth, but it satisfies all the assumptions required for the existence of a free energy minimizer $\rho^{*}_{\sigma^{*}}$.
 We also note that the truncation of the parameter $w$ might seem unnatural (we are truncating the ReLU activation anyway), but it simplifies our analysis. In particular, it allows us to establish a connection between the derivatives (w.r.t. the input $x$) of the predictor implemented by the solution of the flow (\ref{eq:PDE}) and the same quantity evaluated on the minimizer, as $t$ grows large. 

We will use the following notation for the values of the risks corresponding to different activations
\begin{align*}
    R_i^{\tau,m}(\rho) &:= -\frac{1}{M}\left(y_i - \int a^{\tau,m}(w^{\tau,m}x_i +b)^{m}_{\tau}\rho(\btheta)\mathrm{d}\btheta\right),\quad R^{\tau,m}(\rho) := M \sum_{i=1}^M \left(R_i^{\tau,m}(\rho)\right)^2, \\ R_i^{m}(\rho) &:= -\frac{1}{M}\left(y_i - \int a^m(w^mx_i +b)^{m}_{+}\rho(\btheta)\mathrm{d}\btheta\right), \quad\quad\ \ \ \ R^{m}(\rho) := M \sum_{i=1}^M \left(R_i^{m}(\rho)\right)^2,
\end{align*}
and for the related free-energies
\begin{align*}
    \mathcal{F}^{\tau,m}(\rho) &:= \frac{1}{2}R^{\tau,m}(\rho) + \frac{\lambda}{2} M(\rho) - \beta^{-1} H(\rho),\\
    \mathcal{F}^{m}(\rho) &:= \frac{1}{2}R^{m}(\rho) + \frac{\lambda}{2} M(\rho) - \beta^{-1} H(\rho).
\end{align*}
Here, $R^{\tau,m}_i$ and $R^{m}_i$ represent the rescaled error on the $i$-th training sample, and $R^{\tau,m}$ and $R^{m}$ are the standard expected square losses. In this way, we can write the Gibbs minimizers in a compact form, namely, 
\begin{equation}\label{eq:rhodensitytm}
    \rho_{\tau,m}^{*}(\btheta) = Z^{-1}_{\tau,m}(\beta,\lambda)\exp\left\{-\beta\left[\sum\limits_{i=1}^M R^{\tau,m}_i(\rho_{\tau,m}^{*}) \cdot a^{\tau,m} (w^{\tau,m}x_i+b)_{\tau}^m + \frac{\lambda}{2}\|\bm\theta\|_2^2\right]\right\},
\end{equation}
\begin{equation}\label{eq:rhodensitytm2}
    \rho_{m}^{*}(\btheta) = Z^{-1}_{m}(\beta,\lambda)\exp\left\{-\beta\left[\sum\limits_{i=1}^M R^m_i(\rho_{m}^{*}) \cdot a^m (w^mx_i+b)_{+}^{m} + \frac{\lambda}{2}\|\bm\theta\|_2^2\right]\right\},
\end{equation}
where $Z_{\tau,m}(\beta,\lambda)$ and $Z_{m}(\beta,\lambda)$ denote the  partition functions.

\section{Main Results}\label{section:mainres}

Before presenting the main results, let us introduce the notion of a \textit{cluster set}. This set allows us to identify the locations of the knot points of an estimator function that is implemented by the neural network. In particular, we consider the second derivative of the predictor evaluated at the Gibbs distribution with activation $(\btheta,x) \mapsto a^{\tau,m}(w^{\tau,m} x + b)^m_{\tau}$, for large $\tau$, i.e.,
\begin{equation}\label{eq:secderq}
\lim_{\tau\rightarrow\infty}\frac{\partial^2}{\partial x^2}\int a^{\tau,m}(w^{\tau,m}x + b)^{m}_{\tau}\rho^{*}_{\tau,m}(\btheta)\mathrm{d}\btheta.
\end{equation}
Then, the cluster set is associated to the inputs on which the quantity \eqref{eq:secderq} might grow unbounded in absolute value, in the low temperature regime ($\beta^{-1}\rightarrow0$). Intuitively, this indicates that on some points of the cluster set, the tangent of the predictor changes abruptly, resulting in ``knots''. We denote the cluster set by $\Omega(m,\beta,\lambda)$, and we define it below.

Let $\mathcal{I}$ be the set of prediction intervals, i.e.,
$$
\mathcal{I} = \Big\{[x_0:=-L,x_1],[x_1,x_2],\cdots ,[x_{M-1},x_M],[x_M,x_{M+1}:=L]\Big\},
$$
where $L > \max\{|x_1|,\cdots,|x_M|\}$ is any fixed positive constant independent of $(\tau,m,\beta,\lambda)$. For each $I_j:=[x_j, x_{j+1}]\in\mathcal{I}$, the intersection of the cluster set with the prediction interval $I_j$ is denoted by $\overline{\Omega}_j(m,\beta,\lambda)$, i.e.,
\begin{equation}\label{eq:clustbeg}
\overline{\Omega}_j(m,\beta,\lambda) = \Omega(m,\beta,\lambda) \cap I_{\rev{j}}.
\end{equation}
Thus, in order to define the cluster set $\Omega(m,\beta,\lambda)$, it suffices to give the definition of $\overline{\Omega}_j(m,\beta,\lambda)$. To do so, consider the second-degree polynomials $f^j(x)$ and $f_j(x)$ given by \begin{equation}\label{eq:secdpoly}
    \begin{split}
        f^j(x) &:= 1 + x^2 - (A^jx - B^j)^2,\\
        f_j(x) &:= 1 + x^2 - (A_jx - B_j)^2,
    \end{split}
\end{equation}
with coefficients
\begin{figure}[t!]
    \subfloat[\label{subfig:1}]{\includegraphics[width=.33\columnwidth]{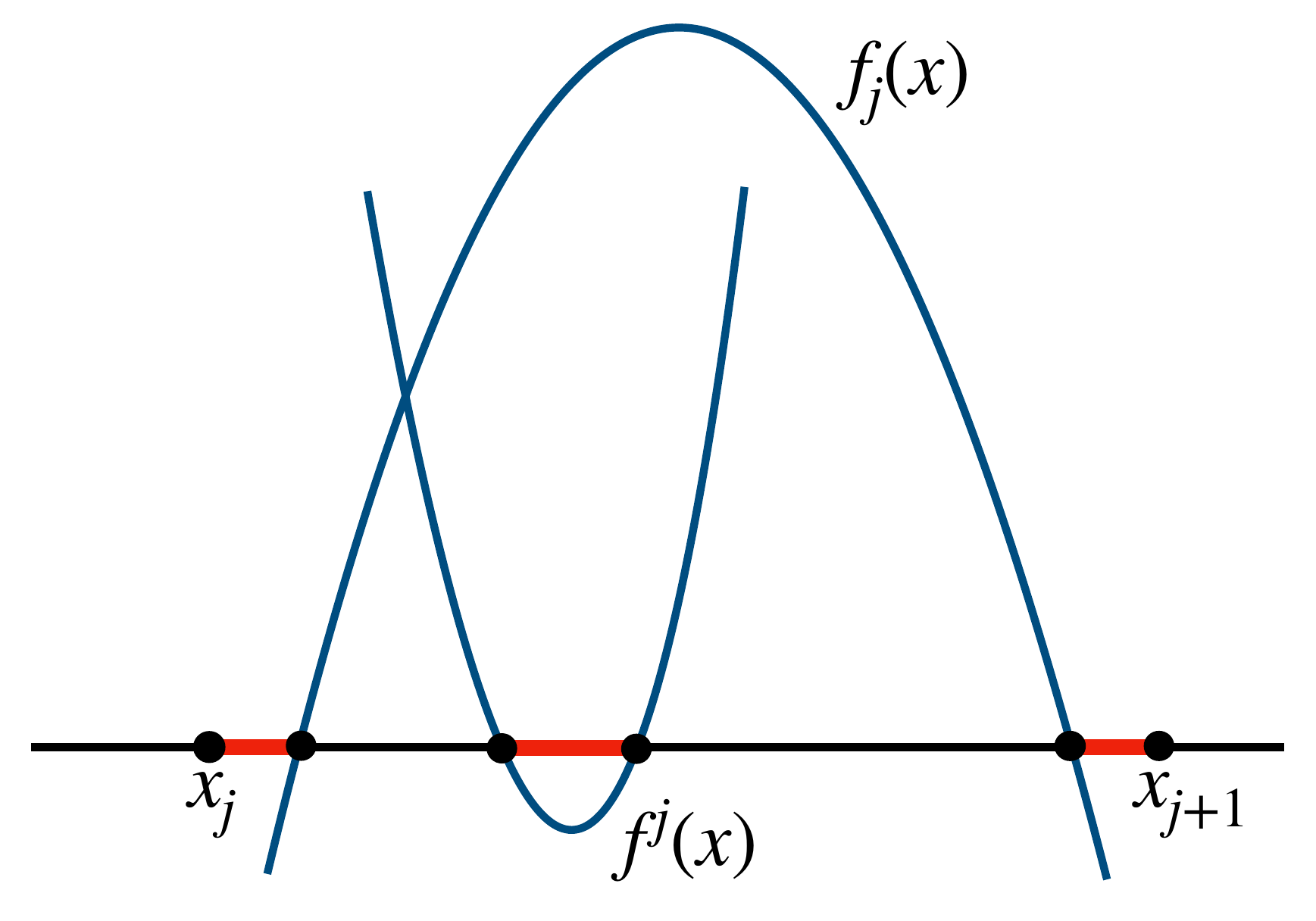}}\
    \subfloat[\label{subfig:2}]{\includegraphics[width=.33\columnwidth]{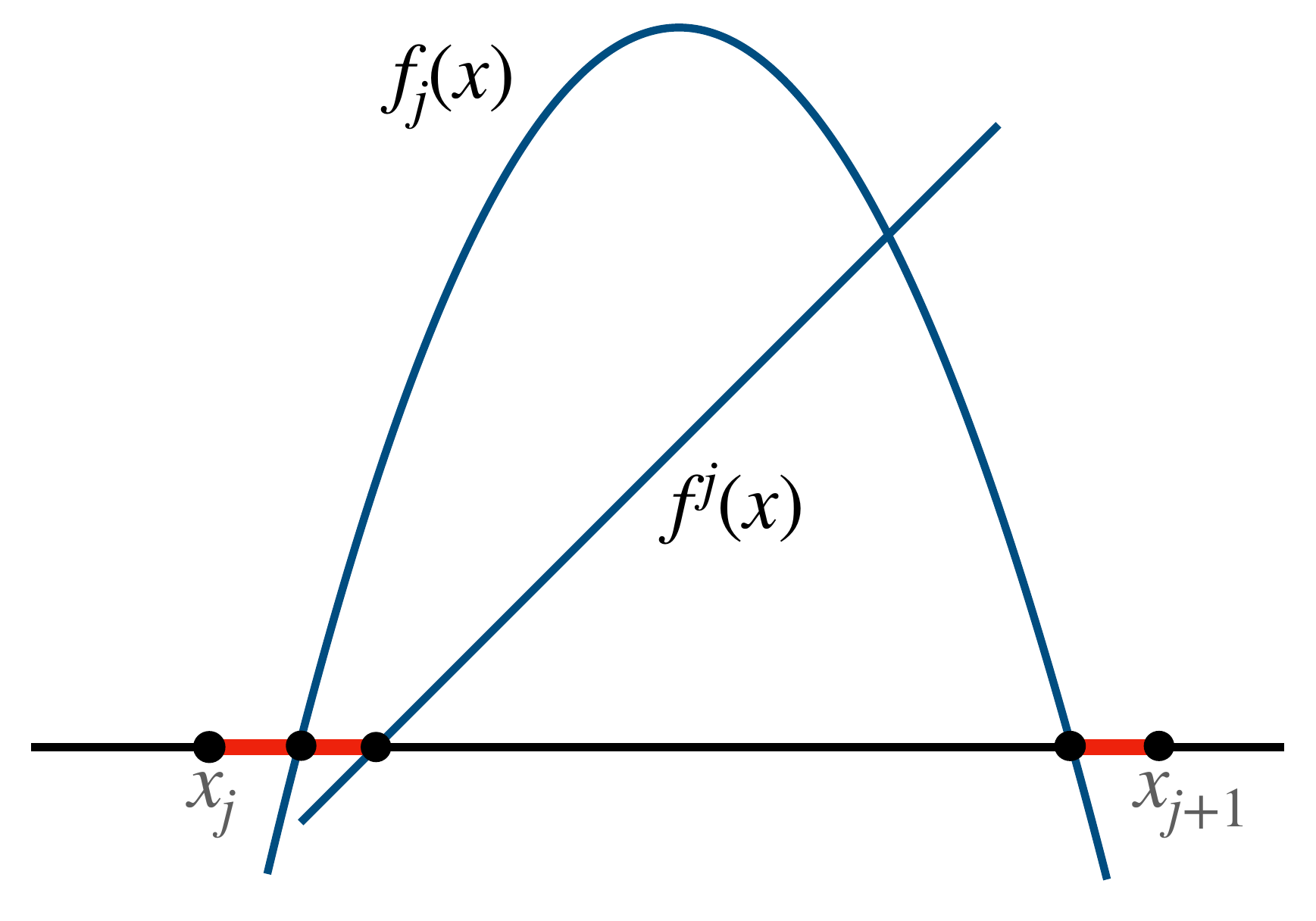}}
    \subfloat[\label{subfig:3}]{\includegraphics[width=.33\columnwidth]{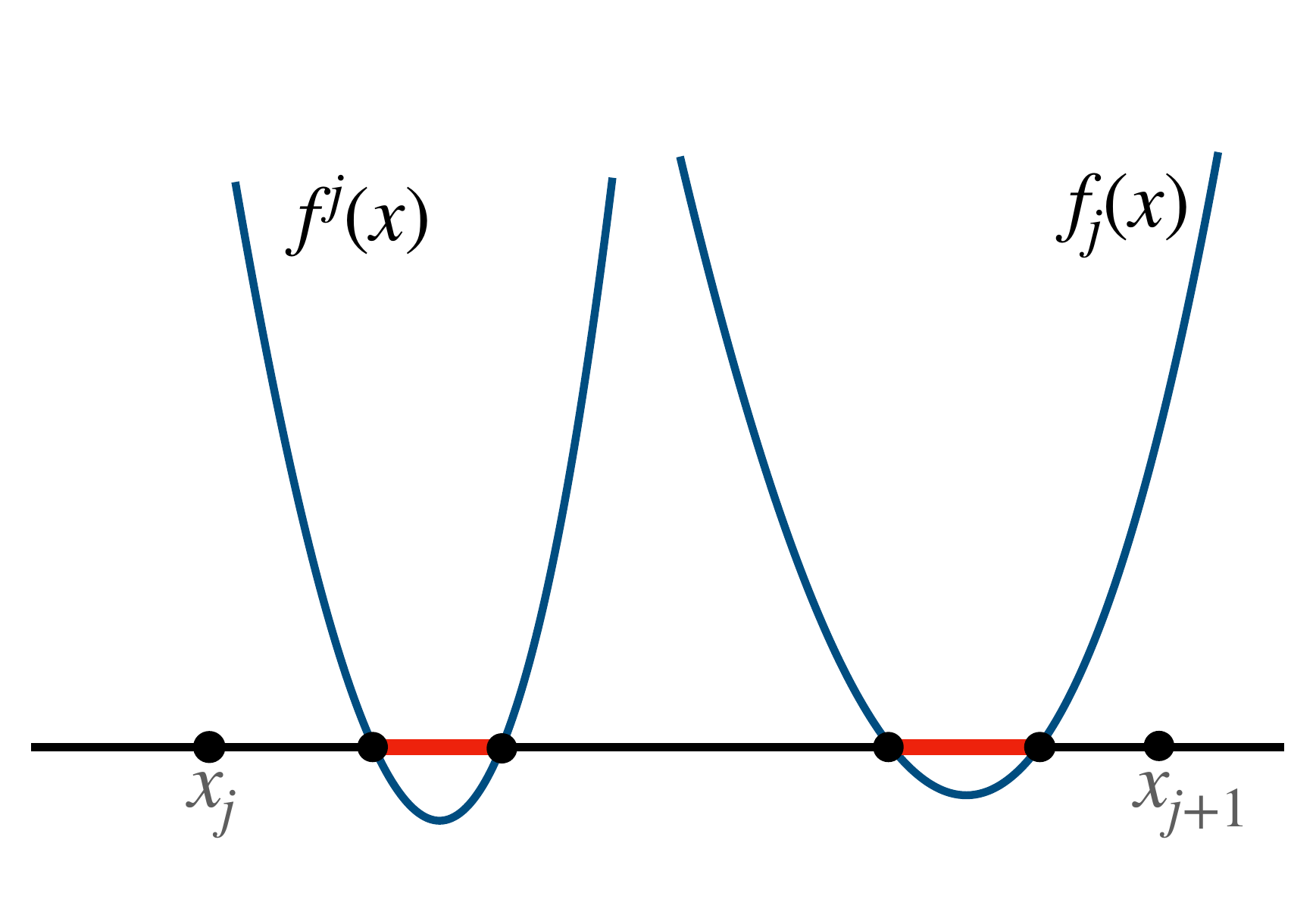}}
\caption{Three different configurations of the polynomials $f^j(x)$ and $f_j(x)$, together with the corresponding cluster set. The dark blue curves show the shape of polynomials, and the red bold intervals indicate the set on which polynomials attain non-positive value.}\label{fig:cluster_set_example}
\vspace{-3mm}
\end{figure}
\begin{align}\label{eq:defAB}
A^j &:= \frac{1}{\lambda}\sum_{i=j+1}^{M} R^m_i(\rho_m^{*}),\ A_j := \frac{1}{\lambda}\sum_{i=1}^{j} R^m_i(\rho_m^{*}),\nonumber\\  B^j &:= \frac{1}{\lambda}\sum_{i=j+1}^{M} R^m_i(\rho_m^{*})x_i,\ B_j := \frac{1}{\lambda}\sum_{i=1}^{j} R^m_i(\rho_m^{*})x_i.
\end{align}
Here, if the summation set is empty (e.g., for $A_0$), the corresponding coefficient is equal to zero. Then, the set $\overline{\Omega}_j(m,\beta,\lambda)$ is defined as the union 
of the non-positive sets of the second-degree polynomials $f^j(x)$ and $f_j(x)$:
\begin{equation}\label{clusterset}
    \overline{\Omega}_j(m,\beta,\lambda) = \Omega^j(m,\beta,\lambda) \cup \Omega_j(m,\beta,\lambda),
\end{equation}
where
\begin{align}\label{def:Om}
    &\Omega^j(m,\beta,\lambda) := \{x\in I_j: f^j(x)  \leq 0\},\nonumber \\
    &\Omega_j(m,\beta,\lambda) := \{x\in I_j: f_j(x)  \leq 0\}.
\end{align}
\rev{We now provide an informal explanation on how the non-positive sets of the second-degree polynomials $f^j(x)$ and $f_j(x)$ come into play. A central object of interest in our analysis is the second derivative of the estimator implemented by the neural network, and our strategy is to bound its magnitude by a particular Gaussian-like integral. This integral does not diverge as long as the corresponding covariance matrix is non-degenerate, i.e., it has strictly positive eigenvalues. In this view, the non-positive sets of the polynomials $f^j(x)$ and $f_j(x)$ precisely characterize the inputs $x$ for which this covariance matrix is degenerate. Hence, for such inputs $x$, this upper bound on the second derivative of the estimator diverges, which implies that the predictor may have a ``knot''.}

Since $f^j(x)$ and $f_j(x)$ are second-degree polynomials, the set $\overline{\Omega}_j(m,\beta,\lambda)$ can be always written as the union of at most $3$ intervals. Moreover, $\overline{\Omega}_j(m,\beta,\lambda)$ depends only on the errors of the estimator at the training points and on the penalty parameter $\lambda$. Thus, if one has access to the value of the errors at each training point for the optimal estimator, i.e., $R^{m}_i(\rho^{*}_m)$, an explicit expression for the cluster set can be readily obtained. Figure \ref{fig:cluster_set_example} shows three different configurations of the polynomials $f^j(x)$ and $f_j(x)$, together with the corresponding cluster set.

\begin{figure}[t!]
    \subfloat[]{\includegraphics[width=.33\columnwidth]{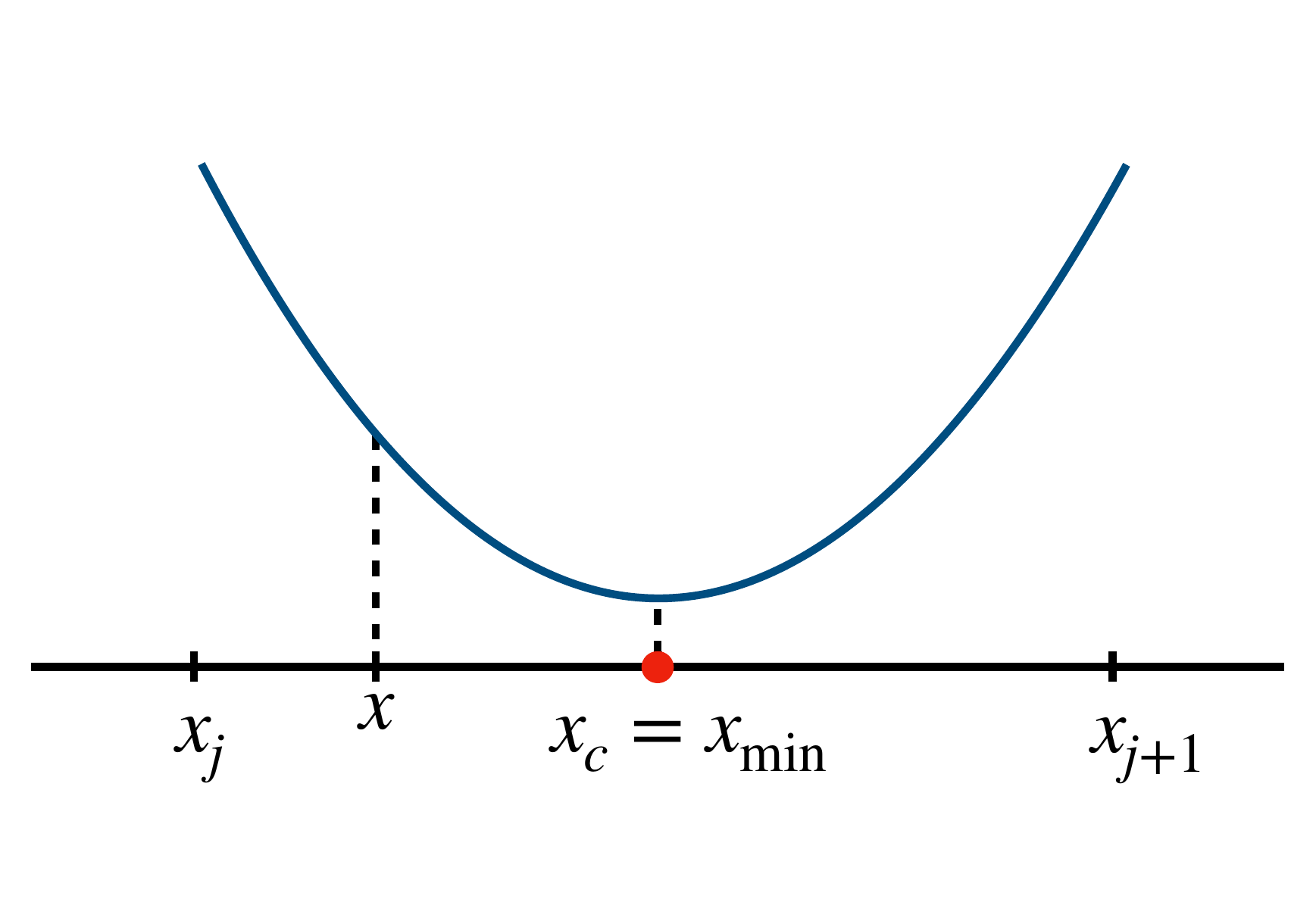}}\
    \subfloat[]{\includegraphics[width=.33\columnwidth]{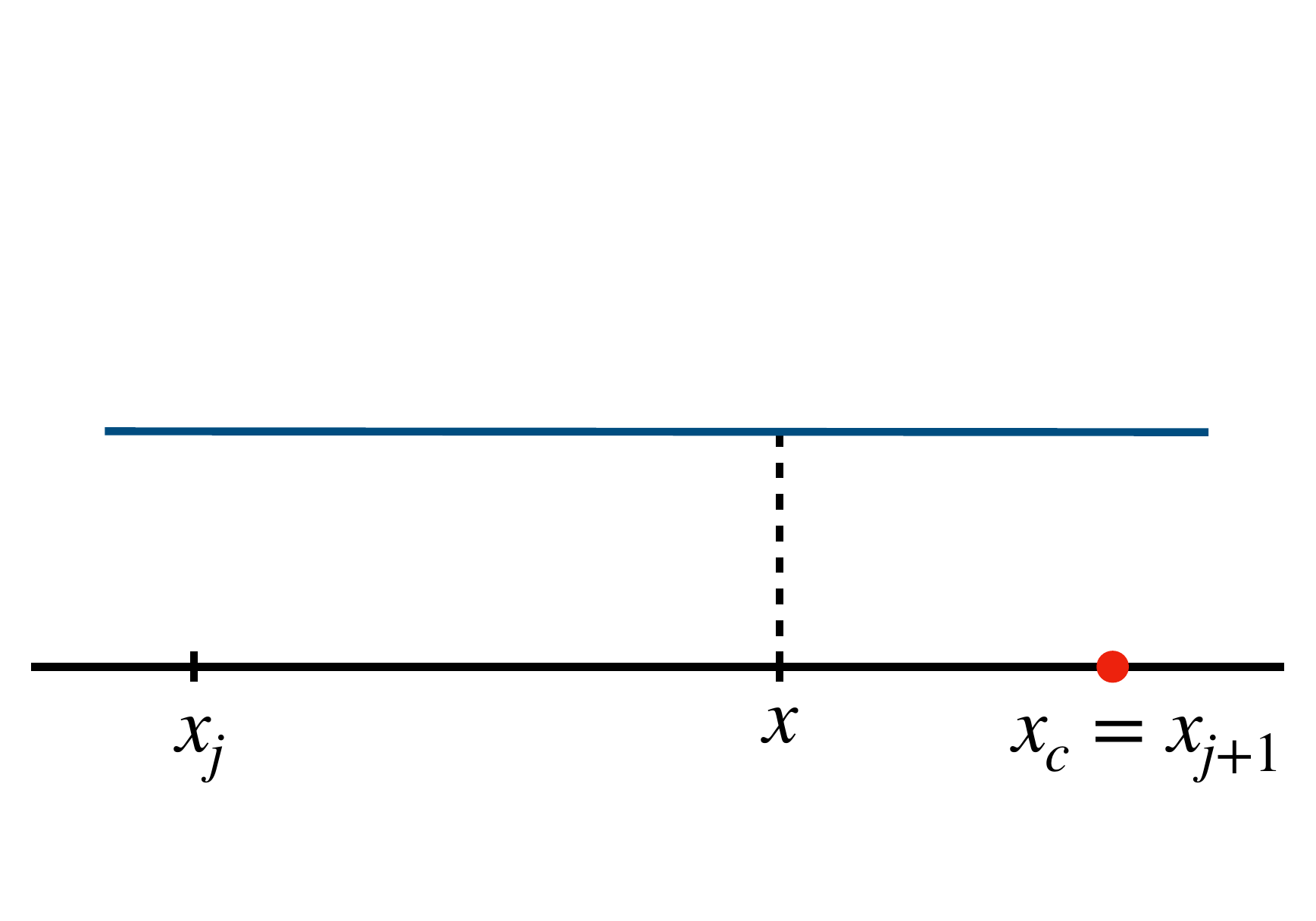}}
    \subfloat[]{\includegraphics[width=.33\columnwidth]{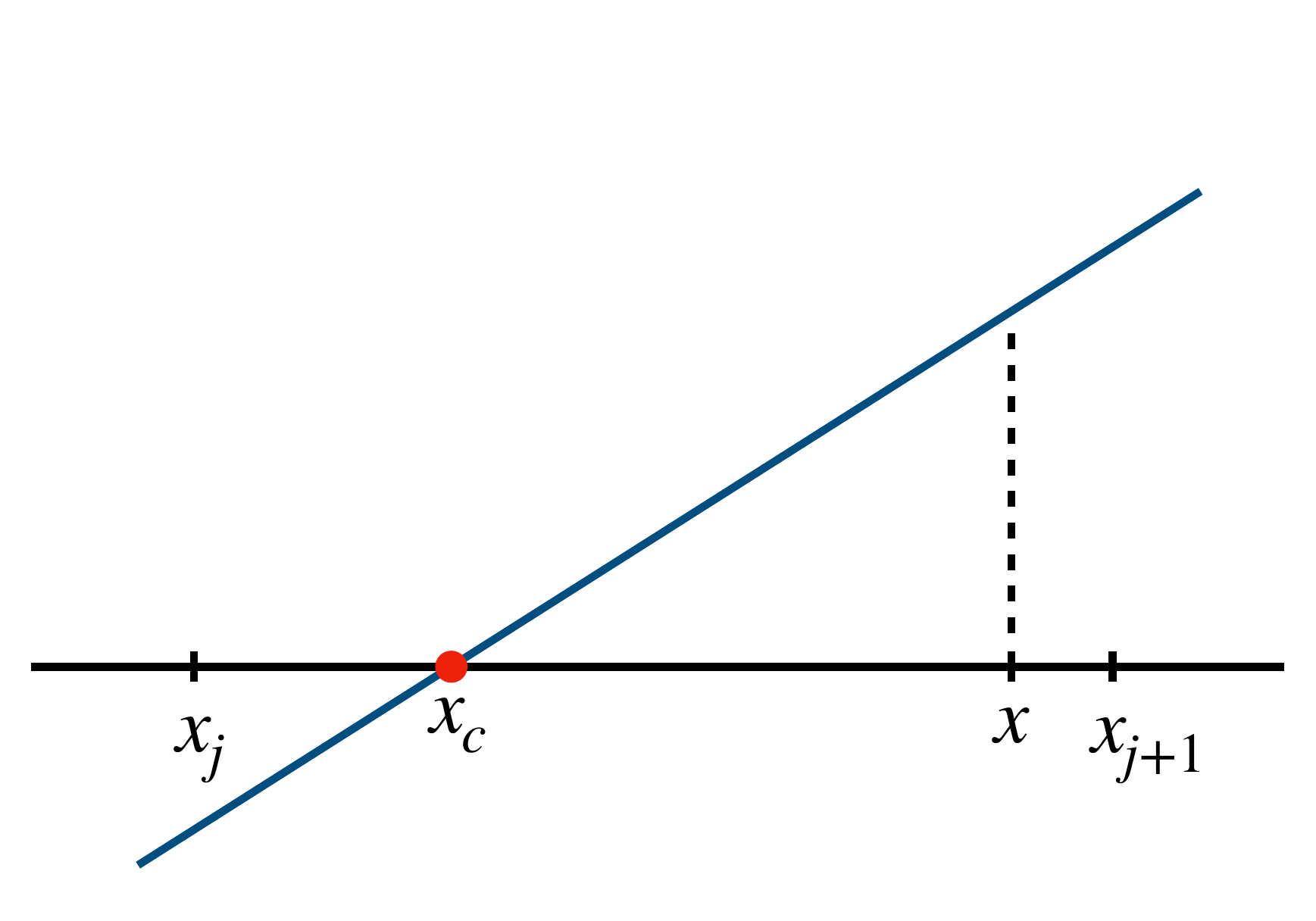}}
    
    \bigskip
    \subfloat[]{\includegraphics[width=.33\columnwidth]{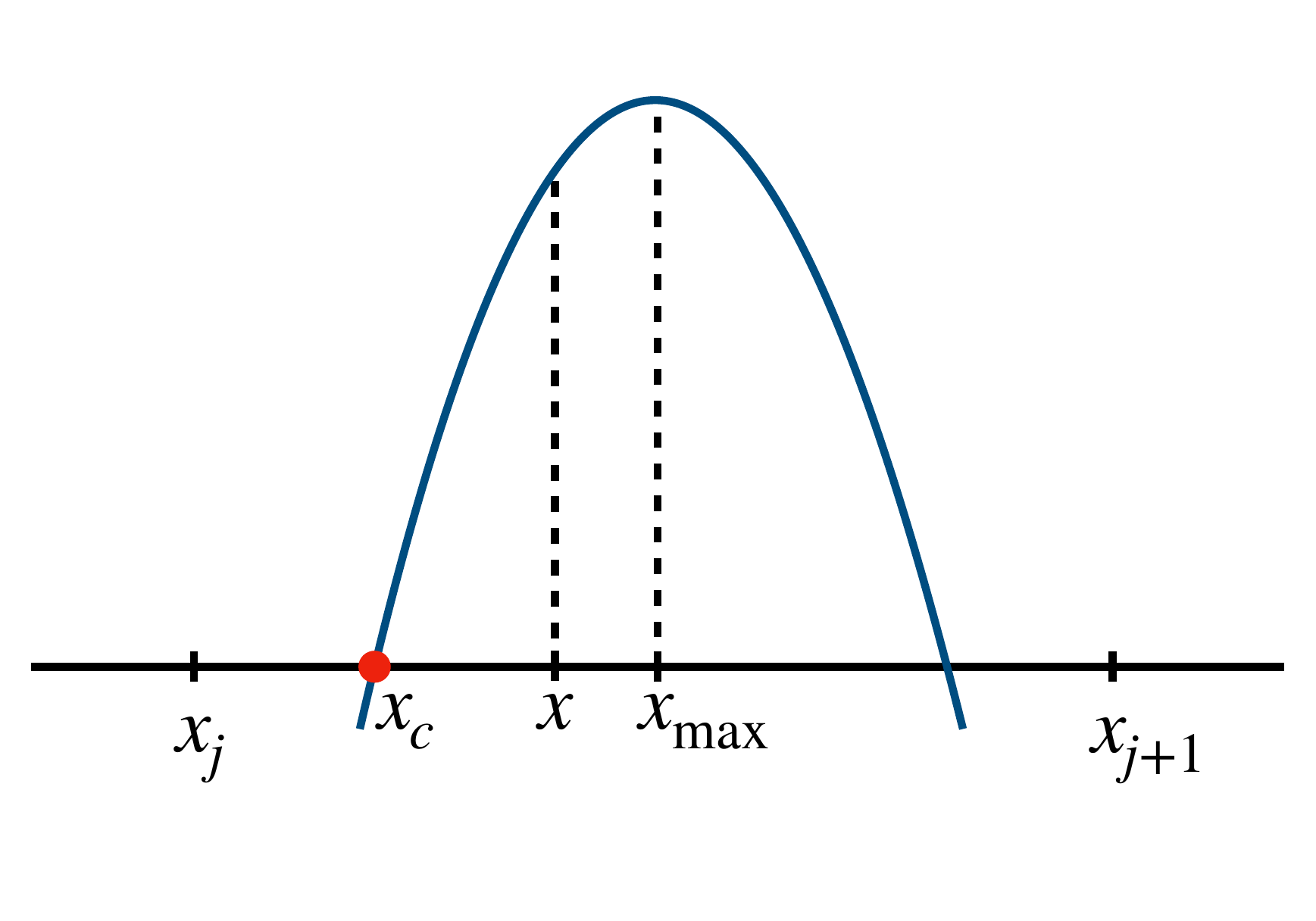}}\
    \subfloat[]{\includegraphics[width=.33\columnwidth]{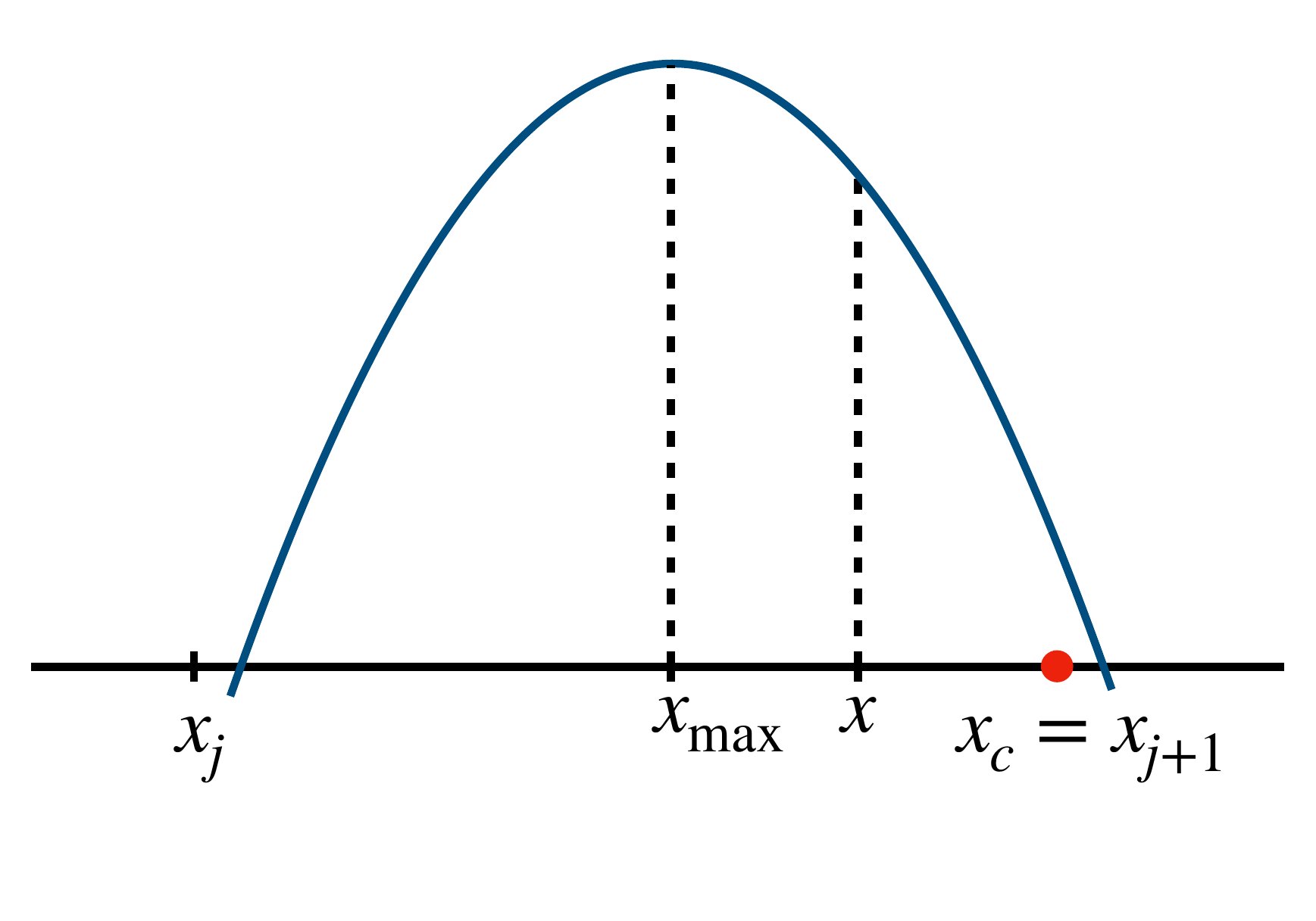}}
    \subfloat[]{\includegraphics[width=.33\columnwidth]{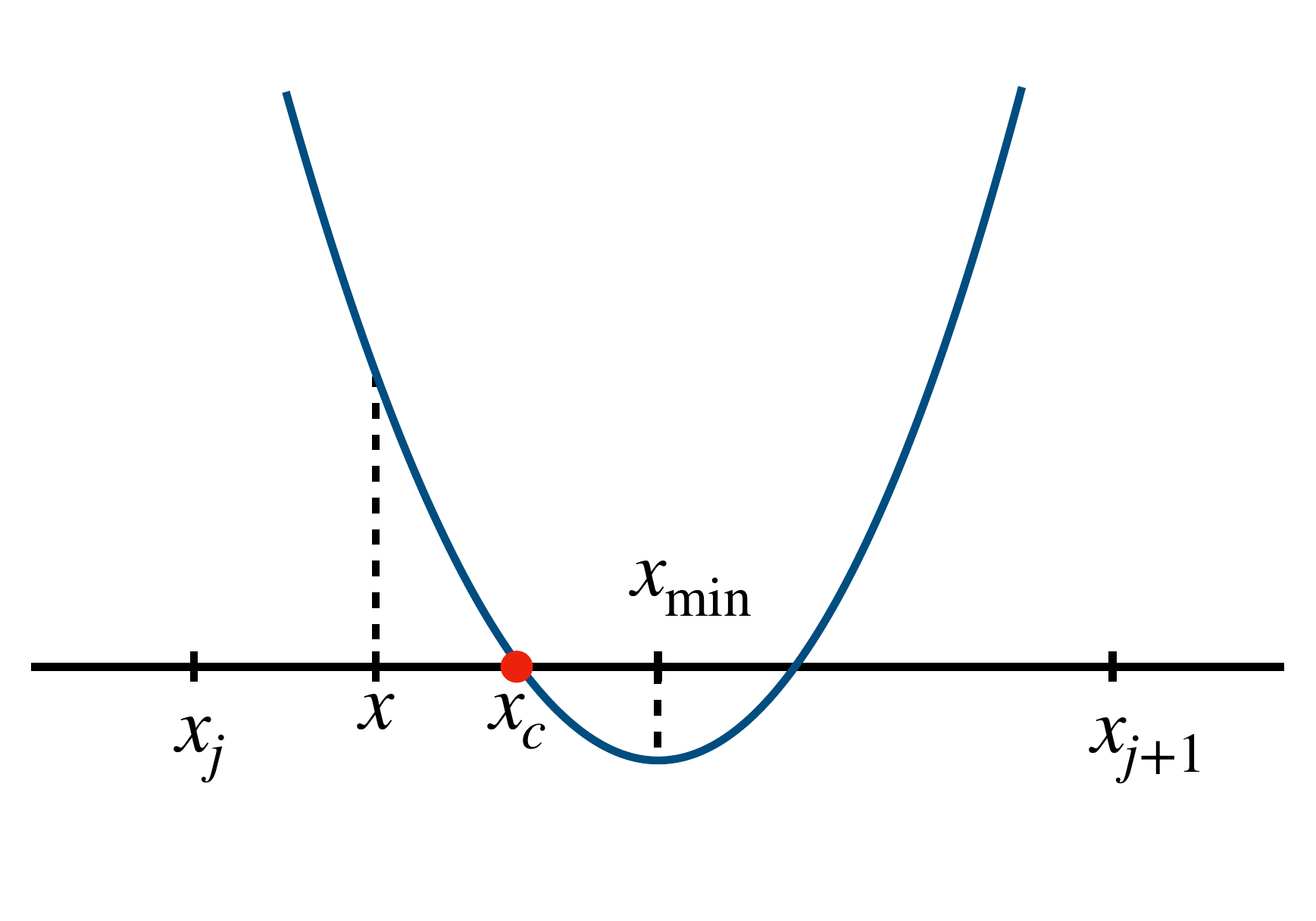}}
\caption{Representation of the critical point $x_c$ for different configurations of the polynomial $f^j$ and evaluation point $x$. The red dot indicates the location of the critical point. The dashed line indicates the value of $f^j$ attained at the corresponding point. The dark blue curve shows the shape of the polynomial $f^j$.}\label{fig:critical_example}
\vspace{-3mm}
\end{figure}

The size of the set $\overline{\Omega}_j(m,\beta,\lambda)$ can be controlled explicitly as a function of the parameters $(m,\beta,\lambda)$.  More formally, in Lemma \ref{worstcasebound}, we show that the Lebesgue measure of $\overline{\Omega}_j(m,\beta,\lambda)$ can be upper bounded as
\begin{equation}\label{eq:controlnorm_prev}
|\overline{\Omega}_j(m,\beta,\lambda)| \leq \frac{e^{C\beta}}{m^2},
\end{equation}
where $C>0$ denotes a numerical  constant independent of $(\tau,m,\beta,\lambda)$ and we have made the following assumption:
\begin{enumerate}[label=\textbf{A\arabic*.}]
  \item $\tau \geq 1$, $\beta \geq \max\Big\{C_1,\frac{1}{\lambda},\frac{1}{\lambda}\log\frac{1}{\lambda}\Big\}$, $m > C_2$ and $\lambda < C_3$ for some numerical constants $C_1,C_2,C_3 > 0$.
\end{enumerate}
In particular, \eqref{eq:controlnorm_prev} implies that the cluster set vanishes as $\beta\rightarrow\infty$ and $m=e^{\Theta(\beta)}$. Therefore, as $\overline{\Omega}_j(m,\beta,\lambda)$ is the union of at most $3$ intervals, the cluster set concentrates on at most 3 points per prediction interval. 

We note that our use of \textbf{A1} throughout the sequel is with the flexibility of $C_1$, $C_2$, and $C_3$ in mind; we are interested in the behavior as $m$ and $\beta$ grow large, so we permit liberty in the determination of the constants implying the formal statements we intend to show. 

A key step of our analysis (cf. Theorem \ref{mainT0}) consists in showing that, outside the cluster set, the absolute value of the second derivative vanishes. Our bound on this absolute value is connected to the speed of decay to zero of the polynomials $f^j(x)$ and $f_j(x)$, as the input $x$ approaches the cluster set. In order to establish a quantitative bound for such a decay, we introduce an auxiliary quantity, namely, a \emph{critical point}, that is associated to each input point outside of the cluster set. Given the polynomial $f^j(\cdot)$ and the input $x \in I_j \setminus \Omega^j(m,\beta,\lambda)$, the critical point $x_c$ associated to $x$ is defined below.

\begin{definition}[Critical point]\label{def:critical}
If $f_j(\tilde{x})=0$ has no solutions for $\tilde{x} \in \mathbb{R}$, then the critical point $x_c$ associated to $x$ and $I_j \setminus \Omega^j(m,\beta,\lambda)$ is defined to be the minimizer of $f_j(\cdot)$ on $I_j$, i.e., $x_c=\arg\min_{\tilde{x}\in I_j} f_j(\tilde{x})$. In case of multiple minimizers, e.g., $(a,b)=(1,0)$, we set $x_c=x_{j+1}$. If $f_j(\tilde{x})=0$ has at least one solution for $\tilde{x}\in \mathbb{R}$, then we let $x_r$ be the root of $f_j$ (in $\mathbb R$ and not necessarily in the segment $I_j$) which is the closest in Euclidean distance to $x$, and we define the \emph{critical point} $x_c$ to be the closest point to $x_r$ in $I_j$, i.e., $x_c=x_r$ if $x_r\in I_j$ and $x_c$ is one of the two extremes of the interval otherwise.
\end{definition}

Figure \ref{fig:critical_example} provides a visualization of the critical point associated to $x$ for several configurations of $f^j$. For the  polynomial $f_j(\cdot)$ and an input $x \in I_j \setminus \Omega_j(m,\beta,\lambda)$, the critical point $\bar{x}_c$ is defined in a similar fashion.
In this view, we show in Lemma \ref{welldefquad} that the following lower bounds on $f^j, f_j$ hold for $x \in I_j \setminus \overline{\Omega}(m,\beta,\lambda)$,
\begin{equation}\label{eq:totallbpoly}
     C^j(x) := \gamma_1 (x - x_c)^2 + \gamma_2 \leq f^j(x) , \quad C_j(x) := \gamma_3 (x - \bar{x}_c)^2 + \gamma_4 \leq f_j(x).
\end{equation}
The coefficients $\gamma_1,\gamma_2,\gamma_3,\gamma_4 > 0$ satisfy the following condition: either $\gamma_1 > \varepsilon$ or $\gamma_2 > \varepsilon$, and either $\gamma_3 > \varepsilon$ or $\gamma_4 > \varepsilon$,
where $\varepsilon>0$ is a numerical constant independent of the choice of $(m,\beta,\lambda)$. 

At this point, we are ready to state our upper bound on the second derivative outside the cluster set.

\begin{theorem}[Vanishing curvature]\label{mainT0} Assume that condition \textbf{A1} is satisfied and that $m>e^{K_1\beta}$ for some numerical constant $K_1>0$ independent of $(\tau,m,\beta,\lambda)$. Then, for each $x\in I_j\setminus\overline{\Omega}_j(m,\beta,\lambda)$, the following upper bound on the second derivative  holds
\begin{equation}\label{eq:decay}
    \lim_{\tau\rightarrow +\infty}  \Bigg|\frac{\partial^2}{\partial x^2}\int a^{\tau,m}(w^{\tau,m}x + b)^{m}_{\tau}\rho^{*}_{\tau,m}(\btheta)\mathrm{d}\btheta\Bigg| \leq \mathcal{O}\left(\frac{1}{m\lambda} + \frac{1}{\beta\lambda^{7/4}(\bar{C}^j(x))^2} \right),
\end{equation}
where the coefficient $\bar{C}^j(x)$ is defined as
\begin{equation}\label{eq:barCdef}
\bar{C}^j(x) = \min \left\{C^j(x), C_j(x), 1\right\},
\end{equation}
with $C^j(x)$ and $C_j(x)$ given by (\ref{eq:totallbpoly}). Furthermore, the following upper-bound on the size of the cluster set holds
\begin{equation}\label{eq:controlnorm}
    |\Omega(m,\beta,\lambda)| \leq \frac{\rev{K_2}}{m},
\end{equation}
for some numerical constant $K_2>0$ independent of $(\tau,m,\beta,\lambda)$.
\end{theorem}
Some remarks are in order. First, the inequality \eqref{eq:decay} shows that, in the low temperature regime, the curvature vanishes outside the cluster set, and it also provides a decay rate. Second, we will upper bound the measure of the cluster set as in \eqref{eq:controlnorm_prev}, thus the condition  $m>e^{K_1\beta}$ ensures that the upper bound \eqref{eq:controlnorm} holds.
Finally, the presence of the coefficient $\bar{C}^j(x)$ is due to the fact that the second derivative can grow unbounded for points approaching the cluster set. Let us highlight that this growth is solely dictated by the distance to the cluster set, and it does not depend on $(m, \beta, \lambda)$. In fact, \eqref{eq:totallbpoly} holds, where one of the coefficients in $\{\gamma_1, \gamma_2\}$ and in $\{\gamma_3, \gamma_4\}$ is lower bounded by a strictly positive constant independent of $(m, \beta, \lambda)$. 

From Theorem \ref{mainT0}, we conclude that, as $m\lambda \rightarrow \infty$ and $\beta\lambda^{7/4}\rightarrow\infty$, the second derivative vanishes for all $x\in I_j\setminus\overline{\Omega}_j(m,\beta,\lambda)$. Furthermore, for $m > e^{C\beta}$ and $\beta\rightarrow\infty$, the cluster set concentrates on at most 3 points per interval. Therefore, the estimator $\int a^{\tau,m}(w^{\tau,m}x + b)^{m}_{\tau}\rho^{*}_{\tau,m}(\btheta)\mathrm{d}\btheta$ is piecewise linear with ``knot'' points given by the cluster set (cf. Theorem \ref{mainT1}). To formalize this result, we define the notion of an \textit{admissible piecewise linear solution}.

\begin{figure}[t!]
    \subfloat[\label{fig:firstconf}]{\includegraphics[width=.33\columnwidth]{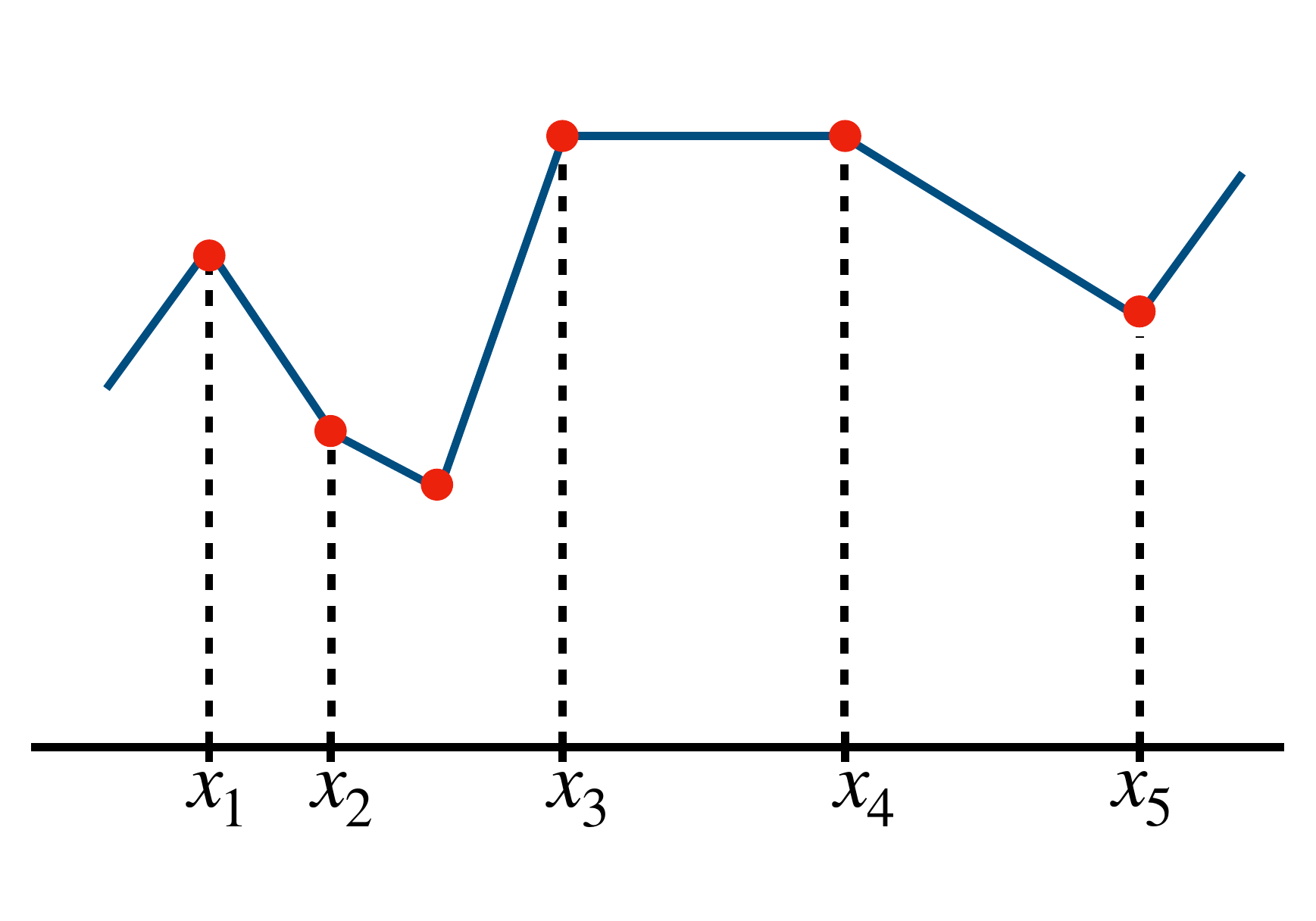}}\
    \subfloat[\label{fig:secondconf}]{\includegraphics[width=.33\columnwidth]{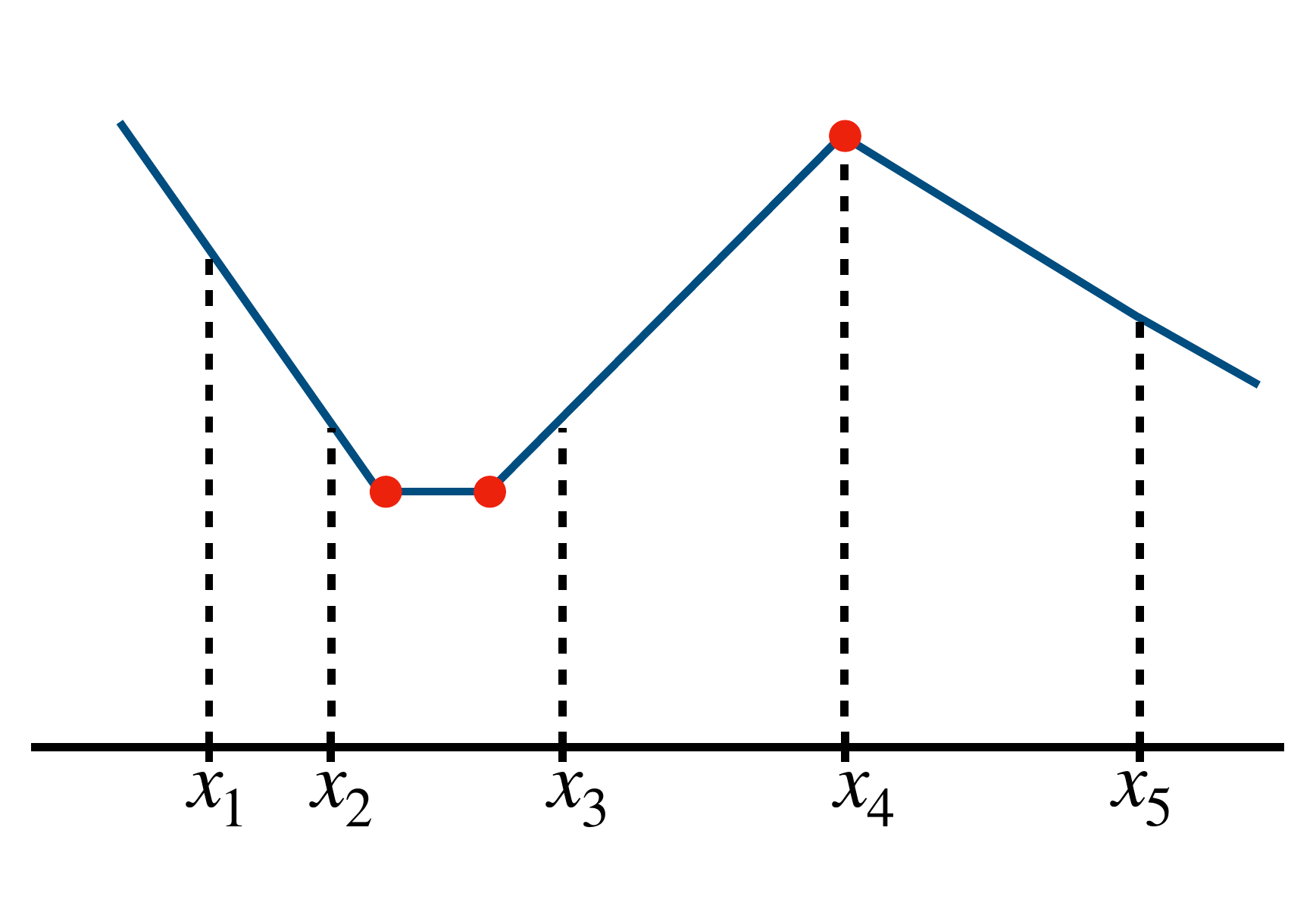}}
    \subfloat[\label{fig:thirdconf}]{\includegraphics[width=.33\columnwidth]{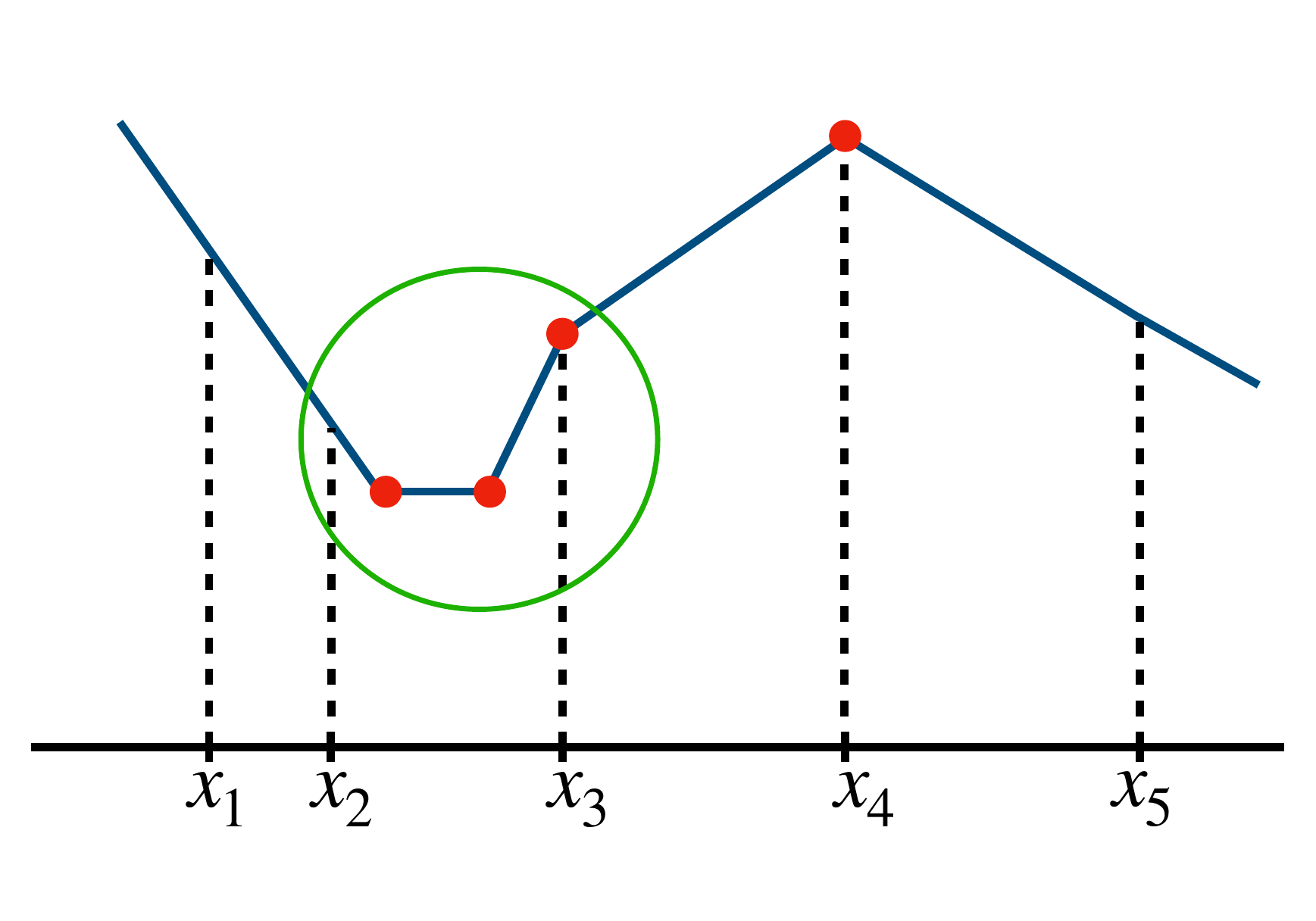}}
\caption{Three examples of piecewise linear functions that fit the data with zero squared error. Dashed black line indicates the $y$ value for each training input. Red dots are located at points where the function changes its tangent.
(a) and (b) illustrates two admissible piecewise linear solutions, while (c) is not admissible due to the location of break points on interval $[x_2,x_3]$.}\label{fig:pwl_example}
\vspace{-3mm}
\end{figure}

\begin{definition}[Admissible piecewise linear solution]\label{def:apwlsolution} Given a set of prediction intervals $\mathcal{I}$, a function $f: \mathbb{R}\rightarrow\mathbb{R}$
is an admissible piecewise linear solution if $f$ is continuous, piecewise linear and has at most 3 knot points (i.e., the points where a change of tangent occurs) per prediction interval $I_j \in \mathcal{I}$.
Moreover, the only configuration possible for 3 knots to occur is the following: two knots are located strictly at the end points of the interval, and the remaining point lies strictly in the interior of the interval.
\end{definition}

 Figure \ref{fig:pwl_example} provides some examples of piecewise linear solutions: (a) and (b) are admissible (in the sense of Definition \ref{def:apwlsolution}), while (c) is not admissible, since it has two knots in the interior of the prediction interval and one located at the right endpoint. As mentioned before, the location of the knot points is associated with the limiting behaviour of the corresponding polynomials $f^j(x)$ and $f_j(x)$. For instance, consider the prediction interval $[x_2,x_3] \in \mathcal{I}$. Then, the configuration of Figure \ref{fig:firstconf} corresponds to the case described in Figure \ref{subfig:1}. In fact, $f^j$ has a negative leading coefficient, and its roots are converging to the end points of the interval. Moreover, $f_j$ has positive curvature and the minimizer is located inside the interval. The same parallel can be drawn between Figure \ref{fig:secondconf} and Figure \ref{subfig:3}. Furthermore, one can verify that the situation described in Figure \ref{fig:thirdconf} cannot be achieved for any configuration of $f^j(x)$ and $f_j(x)$.

We are now ready to state our result concerning the structure of the function obtained from the Gibbs distribution $\rho^{*}_{\tau,m}$.

\begin{theorem}[Free energy minimizer solution is increasingly more piecewise linear]\label{mainT1} Assume that condition \textbf{A1} is satisfied and that $m > e^{K_1\beta}$, where $K_1 > 0$ is a constant independent of $(\tau,m,\beta,\lambda)$. Then, given a set of prediction intervals $\mathcal{I}$, there exists a family of admissible piecewise linear solutions $\{f_{m,\beta,\lambda}\}$ as per Definition \ref{def:apwlsolution}, such that, for any $I\in\mathcal{I}$ and $x \in I$, the following convergence result holds
$$
\lim_{\substack{\beta\lambda^{7/4}\rightarrow+\infty}} \lim_{\tau\rightarrow +\infty}  \Bigg|f_{m,\beta,\lambda}(x) - \int a^{\tau,m}(w^{\tau,m}x + b)^{m}_{\tau}\rho^{*}_{\tau,m}(\btheta)\mathrm{d}\btheta\Bigg| = 0.
$$
\end{theorem}

In words, Theorem \ref{mainT1} means that the solution resulting from the minimization of the free energy \eqref{eq:free_energy} approaches a piecewise linear function, as the noise vanishes. Let us highlight that our result tackles both the regularized case in which $\lambda$ approaches a fixed positive constant and the un-regularized one in which $\lambda$ vanishes (as long as its vanishing rate is sufficiently slow to ensure that $\beta\lambda^{7/4}\rightarrow\infty$). We also note that that the family $\{f_{m,\beta,\lambda}\}$ is well-behaved, i.e., on each linear region the function $f_{m,\beta,\lambda}$ has the following representation:
 $
 f_{m,\beta,\lambda} = u x + v$  for some $u, v \in \mathbb{R},
 $
 and the coefficients $|u|,|v|$ are uniformly bounded in $(m,\beta,\lambda)$.
 
The proof of Theorem \ref{mainT1} crucially relies on the fact that the second moment of $\rho^{*}_{\tau,m}$ is uniformly bounded along the sequence $\beta\lambda^{7/4}\rightarrow\infty$. In fact, the uniform bound on the second moment implies that the \emph{first} derivatives of the predictors w.r.t. the input are uniformly bounded (even for points \textit{inside} the cluster set), and therefore the sequence of predictors is equi-Lipschitz. This, in particular, allows us to show that the limit is well-behaved, as function changes can be controlled via Lipschitz bounds.

 Let us clarify that Theorem \ref{mainT1} does not establish the uniqueness of the limit in $(m, \beta, \lambda)$, i.e., that the limiting piecewise linear function is the same regardless of the subsequence. Our numerical results reported in Figures \ref{fig:intro}, \ref{subfig:plots2}, \ref{fig:6} and \ref{fig:7} suggest that the limit is unique. However, a typical line of argument (see e.g. \cite{jordan1998variational}) would require the lower-semicontinuity of the free energy (which does not hold for $m=\infty$). Furthermore, even the uniqueness of the minimizer for $\beta=\infty$ remains unclear in our setup.
 \rev{Nevertheless, let us point out that the sequence $\{\rho^{*}_{\tau,m}\}$ is tight, since the second moments are uniformly bounded by Lemma \ref{unifboundM_rhomtau}, and Proposition 2.3 in \cite{hu2021mean} suggests that at least the cluster points of the sequence $\{\rho^{*}_{\tau,m}\}$ as $\beta\rightarrow\infty$ coincide with the set of minimizers of the limiting objective ($\beta=\infty$). Another piece of evidence comes from the fact that the annealed dynamics converges to the minimizers of the noiseless objective \cite{chizat2022mean}.}
 We leave for future work the resolution of these issues. 

We remark that providing a quantitative bound on the parameter $\tau$ appears to be challenging. The current analysis relies on a dominated convergence argument which does not lead to an explicit convergence rate. Obtaining such a rate requires understanding the trade-off between the terms in the free energy (\ref{eq:free_energy}) for varying $\tau$, and it is also left for future work. 
 
Finally, by combining Theorem \ref{mainT1} with the mean-field analysis in \cite{mei2018mean}, we obtain the desired result on finite-width networks trained via noisy SGD in the low temperature regime.

\begin{corollary}[Noisy SGD solution is increasingly more piecewise linear]\label{mainT2} Assume that condition \textbf{A1} holds and that $m > e^{K_1\beta}$, where $K_1 > 0$ is a constant independent of $(\tau,m,\beta,\lambda)$. Let $\rho_0$ be absolutely continuous and $K_0$ sub-Gaussian, where $K_0 > 0$ is some numerical constant. Assume also that $M(\rho_0)<\infty$ and $H(\rho_0) > -\infty$. Let $\sigma^{*}(x,\btheta) = a^{\tau,m}(w^{\tau,m}x + b)^{m}_{\tau}$ be the activation function, and let $\btheta^k$ be obtained by running $k=\floor*{t/\varepsilon}$ steps of the noisy SGD algorithm (\ref{eq:SGD}) with data $(\tilde{x}_k,\tilde{y}_k)_{k\geq 0}\sim_{\rm i.i.d.} \mathbb P$ and initialization $\rho_0$. Then, given a set of prediction intervals $\mathcal{I}$, there exists a family of admissible piecewise linear solutions $\{f_{m,\beta,\lambda}\}$ as per Definition \ref{def:apwlsolution}, such that, for any $I\in\mathcal{I}$ and $x \in I$, the following convergence result holds almost surely:
$$
\lim_{\substack{\beta\lambda^{7/4}\rightarrow+\infty}}  \lim_{\tau\rightarrow +\infty} \lim_{t\rightarrow+\infty}\lim_{\substack{\varepsilon\rightarrow0\\N\rightarrow\infty}}\Bigg|f_{m,\beta,\lambda}(x) - \frac{1}{N}\sum_{i=1}^N \sigma^{*}(x,\btheta^k_i)\Bigg| = 0,
$$
where the limit in $N, \varepsilon$ is taken along any subsequence $\{(N, \varepsilon = \varepsilon_N)\}$ with $N/\log\left(N/\varepsilon_N\right)\rightarrow\infty$ and $\varepsilon_N\log\left(N/\varepsilon_N\right)\rightarrow 0$.
\end{corollary}

In words, Corollary \ref{mainT2} means that, at convergence, the estimator implemented by a wide two-layer ReLU network approaches a piecewise linear function, in the regime of vanishingly small noise. In fact, as $\tau, m\to\infty$, the activation function $\sigma^{*}(x,\btheta) = a^{\tau,m}(w^{\tau,m}x + b)^{m}_{\tau}$ converges pointwise to the ReLU activation $a(wx+b)_+$. We also remark that our result holds for any initialization of the weights of the network, as long as some mild technical conditions are fulfilled (absolute continuity, sub-Gaussian tails, finite second moment and entropy).

\rev{Let us clarify some technical aspects of  the statement of Corollary \ref{mainT2}. The result holds for a particular sequence of minimizers, since some of the limits ($t\rightarrow\infty$, $(N,\varepsilon^{-1})\rightarrow\infty$, and $\beta\rightarrow\infty$) are not interchangeable.
Furthermore, it appears to be difficult to prove the same statement directly for the noiseless case ($\beta=\infty$). We also point out that 
the stochasticity of the gradient descent algorithm does not play a role in our analysis, since its impact is seen to be inconsequential by the usual concentration argument \cite{mei2018mean} when passing to its non-stochastic counterpart.
}

As concerns the limit in $t$, describing the dependence of the mixing time of the diffusion dynamics \eqref{eq:PDE} on the temperature parameter $\beta$ is a cumbersome task. In particular, \cite{gayrard2004metastability} show that an exponentially bad dependence could occur if the target function has multiple small risk regions. \rev{However, some recent studies show an exponentially fast convergence of the noisy dynamics under some reasonable but particular conditions on the objective landscape \cite{chizat2022mean, nitanda2022convex}.}

As concerns the limit in $(N, \varepsilon)$, the analyses in \cite{mei2018mean,mei2019mean} lead to an upper bound on the error term that, with probability at least $1-e^{-z^2}$, is given by
\begin{equation}\label{mei_coupling}
    C e^{Ct} \sqrt{1/N \lor \varepsilon} \cdot \left[\sqrt{1 + \log (N (t/\varepsilon \lor 1))}+z\right],
\end{equation}
where $a\lor b$ denotes the maximum between $a$ and $b$. The exponential dependence of \eqref{mei_coupling} in the time $t$ of the dynamics is a common drawback of existing mean-field analyses, and improving it is an open problem which lies beyond the scope of this work. Let us conclude by mentioning that the numerical results presented in Section \ref{section:numsim} suggest that, in practical settings, the convergence to the limit occurs rather quickly in the various parameters.

\section{Proof of the Main Results}\label{section:proofsketch}

\subsection{Roadmap of the Argument}

We start by providing an informal outline of the proof for the main statements. In Section \ref{selected_proofs}, we show that, in the low temperature regime, the curvature of the predictor evaluated at the Gibbs distribution $\rho^{*}_{\tau,m}$ vanishes everywhere except at a small neighbourhood of at most three points per prediction interval $I_j \in \mathcal{I}$ (Theorem \ref{mainT0}). This is done in a few steps. First, in Lemma \ref{convdelta}, we show that, as $\tau\rightarrow\infty$, the density
$\rho^{*}_{\tau,m}$ acts similarly to a delta distribution supported on the lower-dimensional linear subspace $\{b\in\mathbb{R}: b=-w^mx\}$, namely,
\begin{equation}\label{eq:step1}
\lim_{\tau\rightarrow\infty}\frac{\partial^2}{\partial x^2}\int a^{\tau,m}(w^{\tau,m}x + b)^{m}_{\tau}\rho^{*}_{\tau,m}(\btheta)\mathrm{d}\btheta \approx \int a^m (w^m)^2\rho^{*}_m(a,w,-w^mx) \mathrm{d}a\mathrm{d}w.
\end{equation}
To do so, in Lemma \ref{convmin} we prove that, as $\tau\rightarrow\infty$, the sequence $\rho^{*}_{\tau,m}(\btheta)$ of  minimizers of the free energy $\mathcal F^{\tau, m}$ converges pointwise for all $\btheta$  to a minimizer $\rho^{*}_m(\btheta)$ of the free energy $\mathcal F^m$ with truncated ReLU activation. Then, a dominated convergence argument allows us to obtain \eqref{eq:step1}. Next, in Lemma \ref{UBintegral} we show that, as $\beta\rightarrow\infty$, the absolute value of the integral
\begin{equation}\label{sketch_delta}
    \int a^m (w^m)^2\rho^{*}_m(a,w,-w^mx) \mathrm{d}a\mathrm{d}w
\end{equation}
can be made arbitrary small for all $x$ except those in the cluster set. The idea is that the absolute value of (\ref{sketch_delta}) can be bounded by a certain Gaussian integral, and the corresponding covariance matrix is well-defined everywhere except in the cluster set (see Lemmas \ref{welldefquadgeneric} and \ref{welldefquad}). The definition of the cluster set (see \eqref{eq:clustbeg}-\eqref{def:Om}) together with the fact that the partition function of $\rho^{*}_m$ is uniformly bounded in $m$ (see Lemma \ref{unifboundZm}) allows us to show that the cluster set concentrates on at most three points per interval as $\beta\rightarrow\infty$.

In Section \ref{subsec:pfth2}, we show that the predictor evaluated at the Gibbs distribution $\rho^{*}_{\tau,m}$ can be approximated arbitrarily well by an admissible piecewise linear solution (Theorem \ref{mainT1}).
First, via a Taylor argument, since the curvature vanishes, the estimator can be approximated by a linear function on each interval of $\mathcal{I} \setminus \Omega(m,\beta,\lambda)$. Since the cluster set vanishes concentrating on at most three points per prediction interval, the predictor converges to an admissible piecewise linear solution. However, there is one technical subtlety to consider before reaching this conclusion. Namely, we must consider the possibility that the sequence of predictors experiences unbounded oscillations inside the cluster set, which might ultimately result in a discontinuous limit. Fortunately, this scenario is ruled out because the sequence $\rho^{*}_{\tau,m}$ has uniformly bounded second moments. This fact in conjunction with the structure of the \emph{first} derivative of the predictor yields the conclusion that the sequence of predictors is equi-Lipschitz, and therefore the limit is well-behaved.

Finally, the proof of Corollary \ref{mainT2} follows from similar arguments together with the application of the result established in \cite{mei2018mean}. More specifically, first, the truncation of the parameter $w$ ensures that, as $t\rightarrow\infty$, the curvature of the predictor evaluated on the solution $\rho_t$ of the flow (\ref{eq:PDE}) converges pointwise in $x$ to the corresponding evaluation on the Gibbs distribution $\rho^{*}_{\tau,m}$. Next, following \cite{mei2018mean}, we couple the weights obtained after $\floor*{t/\varepsilon}$ steps of the SGD iteration \eqref{eq:SGD} with $N$ i.i.d. particles with distribution $\rho_t$, thus obtaining that the curvature of the SGD predictor converges to the curvature of the flow predictor. By using this coupling again, together with the fact that along the trajectory of the flow $M(\rho_t) < C$ (see \cite{mei2018mean} or \cite{jordan1998variational}), we obtain a uniform bound on the second moment of the empirical distribution $\hat{\rho}^{N}_{\floor*{t/\varepsilon}}$ of the SGD weights. The final result then follows from the same Lipschitz argument described above.

\subsection{Proof of Theorem \ref{mainT0}}\label{selected_proofs}

Let us start with the proof of the vanishing curvature phenomenon. The quantity
\begin{equation}\label{quant1}
    \frac{\partial^2}{\partial x^2}\int a^{\tau,m}(w^{\tau,m}x + b)^{m}_{\tau}\rho^{*}_{\tau,m}(\btheta)\mathrm{d}\btheta
\end{equation}
is hard to analyze directly due to the presence of the $\tau$-smoothing in the soft-plus activation. However, the structure of the activation $(\cdot)_{\tau}^m$ alongside with the pointwise convergence of the minimizers $\rho^{*}_{\tau,m}$ to $\rho^{*}_m$ (cf. Lemma \ref{convmin}) allows us to infer the properties of (\ref{quant1}) through the analysis of the auxiliary object:
\begin{equation}\label{quant2}
    \int a^m (w^m)^2\rho^{*}_m(a,w,-w^mx) \mathrm{d}a\mathrm{d}w.
\end{equation}
Formally, we show that the approximation result below holds.
\begin{lemma}[Convergence to delta]\label{convdelta} Assume that condition \textbf{A1} holds.
Let $\rho^{*}_{\tau,m}$ and $\rho^{*}_{m}$ be the minimizers of the free energy for truncated softplus and ReLU activations, respectively, as defined in (\ref{eq:rhodensitytm})-(\ref{eq:rhodensitytm2}). Then, 
$$
\lim_{\tau \rightarrow \infty}\left|\frac{\partial^2}{\left(\partial x\right)^2}\int a^{\tau,m}\left[ (w^{\tau,m}x+b)^m_{\tau}\right]\rho^{*}_{\tau,m}(\btheta)  \mathrm{d}\btheta - \int a^m (w^m)^2\rho^{*}_m(a,w,-w^mx) \mathrm{d}a\mathrm{d}w\right| \le \frac{C}{m\lambda},
$$
where $C$ is a constant independent of $(m, \tau, \beta, \lambda)$. 
\end{lemma}
\begin{proof}[Proof of Lemma \ref{convdelta}] First, we show that
\begin{align}\label{eq:intclaim}
 \begin{split}
    \lim_{\tau \rightarrow \infty}\Bigg|\int a^{\tau,m}&\left[\frac{\partial^2}{\left(\partial x\right)^2} (w^{\tau,m}x+b)^m_{\tau}\right]\rho^{*}_{\tau,m}(\btheta)  \mathrm{d}\btheta \\ &-  \int a^m (w^m)^2\rho^{*}_m(a,w,-w^mx) \mathrm{d}a\mathrm{d}w\Bigg| \le \frac{C}{m\lambda}.
 \end{split}
\end{align}
Recall the definition of the activation $(\cdot)^{\tau}_m$ provided in (\ref{act_taum}). We can decompose the integral into two pieces with respect to the domain of truncation and obtain
\begin{align}\label{decompose1}
    &\int a^{\tau,m}\left[\frac{\partial^2}{(\partial x)^2}(w^{\tau,m}x+b)_{\tau}^m\right] \rho^{*}_{\tau,m}(\btheta)\mathrm{d}\btheta  \nonumber\\ =& \int_{w^{\tau,m}x+b\leq x_m} a^{\tau,m}\left[\frac{\partial^2}{(\partial x)^2}(w^{\tau,m}x+b)_{\tau}\right] \rho^{*}_{\tau,m}(\btheta)\mathrm{d}\btheta\  \nonumber\\
    +& \int_{w^{\tau,m}x+b > x_m} a^{\tau,m}(w^{\tau,m})^2\left[\frac{\partial^2}{(\partial u)^2}\phi_{\tau,m}(u)\Bigg|_{u=w^{\tau,m}x+b}\right] \rho^{*}_{\tau,m}(\btheta)\mathrm{d}\btheta.
\end{align}
Let us focus on the first term in the RHS of (\ref{decompose1}). The second derivative has the following form
$$
\frac{\partial^2}{\left(\partial x\right)^2} (w^{\tau,m}x+b)_{\tau} = (w^{\tau,m})^2 \cdot \frac{\tau  e^{\tau(w^{\tau,m}x+b)}}{\left(e^{\tau(w^{\tau,m}x+b)} + 1\right)^2} > 0.$$
Thus, the following chain of equalities holds
\begin{align*}
    &\int_{w^{\tau,m}x+b\leq x_m} a^{\tau,m} \left[\frac{\partial^2}{\left(\partial x\right)^2} (w^{\tau,m}x+b)_{\tau}\right] \rho^{*}_{\tau,m}(\btheta)\mathrm{d}\btheta  \\ &=\int_{w^{\tau,m}x+b\leq x_m}  a^{\tau,m}(w^{\tau,m})^2 \cdot \frac{\tau  e^{\tau(w^{\tau,m}x+b)}}{\left(e^{\tau(w^{\tau,m}x+b)} + 1\right)^2}\rho^{*}_{\tau,m}(\btheta) \mathrm{d}\btheta  \\ &=\int \mathbbm{1}_{\{y\leq \tau x_m\}} \cdot  a^{\tau,m}(w^{\tau,m})^2 \cdot \frac{ e^{ y}}{\left(e^{y} + 1\right)^2} \rho^{*}_{\tau,m}\left(a,w,\frac{y}{\tau}-w^{\tau,m}x\right) \mathrm{d}(a,w,y),
\end{align*}
where in the last step we have performed the change of variables $y=\tau(w^{\tau,m}x+b)$.
By Lemma \ref{convmin}, we have that, as $\tau\to\infty$, $\rho^{*}_{\tau,m}(\btheta)$ converges to $\rho^{*}_m(\btheta)$ pointwise in $\btheta$. Furthermore, as $\tau\to\infty$, $a^{\tau,m}$ converges to $a^{m}$ for any $a$, and $w^{\tau,m}$ converges to $w^{m}$ for any $w$. Thus, as the
Gibbs distributions $\rho^{*}_{\tau,m}(\btheta)$ and $\rho^{*}_m(\btheta)$ are continuous with respect to $\btheta$, we have that
\begin{align*}\label{convdelta1}
    &\lim_{\tau\rightarrow \infty} \left[\mathbbm{1}_{\{y\leq \tau x_m\}} \cdot a^{\tau,m}(w^{\tau,m})^2 \cdot \frac{ e^{ y}}{\left(e^{y} + 1\right)^2} \rho^{*}_{\tau,m}\left(a,w,\frac{y}{\tau}-w^{\tau,m}x\right)\right] \\ &=a^{m}(w^m)^2 \cdot \frac{ e^{ y}}{\left(e^{y} + 1\right)^2} \rho^{*}_m\left(a,w,-w^mx\right).
\end{align*}
Furthermore, combining (\ref{Zmt_bound}) and (\ref{eq:bd1}) from Lemma \ref{limitrho_mtau}, we get the following bound
\begin{equation}\label{rhobound1}
  \rho^{*}_{\tau,m}(\btheta) \leq C' \exp\left(-\frac{\beta\lambda \|\btheta\|^2_2}{2}\right),
\end{equation}
for some constant $C' > 0$ independent of $\btheta$ and $\tau$. Thus, we have
\begin{align*}
    &|a^{\tau,m}|(w^{\tau,m})^2 \cdot \frac{ e^{ y}}{\left(e^{y} + 1\right)^2} \rho^{*}_{\tau,m}\left(a,w,\frac{y}{\tau}-w^{\tau,m}x\right) \\ \leq &C' m^3 \cdot \frac{ e^{ y}}{\left(e^{y} + 1\right)^2} \cdot \exp\left(-\frac{\beta\lambda (a^2 + w^2)}{2}\right),
\end{align*}
which is integrable in $(y,a,w)$. 
Hence, by using the Dominated Convergence theorem and integrating out $y$ using Tonelli's theorem, we have
\begin{align}\label{eq:dec2}
 \begin{split}
\lim_{\tau\rightarrow\infty} \Bigg|\int_{w^{\tau,m}x+b\leq x_m} a^{\tau,m}&\left[\frac{\partial^2}{(\partial x)^2}(w^{\tau,m}x+b)_{\tau}\right] \rho^{*}_{\tau,m}(\btheta)\mathrm{d}\btheta\\ &- \int a^m (w^m)^2\rho^{*}_m(a,w,-w^mx) \mathrm{d}a\mathrm{d}w \Bigg| = 0.
\end{split}
\end{align}
Now, by triangle inequality, it remains to show that the absolute value of the second term in the RHS of (\ref{decompose1}) can be upper bounded by $\mathcal{O}\left(\frac{1}{m\lambda}\right)$ as $\tau\rightarrow\infty$. Recall that, by construction,
$$
|\phi''_{\tau,m}(x)|\leq \frac{1}{m^2}, \quad |a^{\tau,m}| \leq m, \quad |w^{\tau,m}| \leq |w|,
$$
for any $x>x_m$ and any $(a,w) \in\mathbb{R}^2$. Thus, the following upper bound holds
\begin{align}
    &\lim_{\tau\rightarrow\infty} \int_{w^{\tau,m}x+b > x_m} |a^{\tau,m}|(w^{\tau,m})^2\left|\frac{\partial^2}{(\partial u)^2}\phi_{\tau,m}(u)\Bigg|_{u=w^{\tau,m}x+b}\right| \rho^{*}_{\tau,m}(\btheta)\mathrm{d}\btheta \nonumber\\&\leq \frac{1}{m}\lim_{\tau\rightarrow\infty}\int w^2 \rho^{*}_{\tau,m}(\btheta)\mathrm{d}\btheta.\label{eq:dec3}
\end{align}
In addition, we have the following pointwise convergence of the integrand
$$
\lim_{\tau\rightarrow\infty} w^2 \rho^{*}_{\tau,m}(\btheta) = w^2 \rho^{*}_{m}(\btheta).
$$
Furthermore, by using (\ref{rhobound1}), we conclude that the integrand can be dominated by an integrable function. Hence, an application of the Dominated Convergence theorem gives that 
\begin{equation}\label{eq:dec4}
\frac{1}{m}\lim_{\tau\rightarrow\infty} \int w^2\rho^{*}_{\tau,m}(\btheta)\mathrm{d}\btheta = \frac{1}{m}\int w^2 \rho^{*}_{m}(\btheta)\mathrm{d}\btheta \leq \frac{C''}{m\lambda},
\end{equation}
where the last inequality follows from Lemma \ref{limitrho_mtau}, which gives that $M(\rho^{*}_{m})<C''/\lambda$ for some $C''>0$ that is independent of $(m,\lambda)$.
By combining (\ref{decompose1}), (\ref{eq:dec2}), (\ref{eq:dec3}) and (\ref{eq:dec4}), we conclude that (\ref{eq:intclaim}) holds. Finally, by using a standard line of arguments, i.e., Mean Value theorem and Dominated Convergence, the derivative can be pushed inside the integral sign, which finishes the proof.
\end{proof}

Next, we study the set on which (\ref{quant2}) might grow unbounded. In particular, in Lemma \ref{worstcasebound}, we provide an upper bound on the measure of the set $\overline{\Omega}_j(m,\beta,\lambda)$ defined in (\ref{clusterset})-(\ref{def:Om}). To do so, we will first show that the partition function of $\rho^{*}_m$ is uniformly bounded in $m$, as stated and proved below.

\begin{lemma}[Uniform bound on partition function]\label{unifboundZm}
Consider $\sigma^{*}(\btheta,x) = a^{\tau,m}(w^{\tau,m}x+b)^m_{\tau}$ or $\sigma^{*}(\btheta,x) = a^m(w^mx+b)^m_{+}$, and let $\rho^{*}_{\sigma^{*}}$ be the Gibbs distribution with activation $\sigma^{*}$. Then, the following upper bound holds for its partition function $Z_{\sigma^{*}}(\beta,\lambda)$:
$$
\ln Z_{\sigma^{*}}(\beta,\lambda) \leq \beta C + 1 + 3 \log \frac{8\pi}{\beta\lambda},
$$
where $C>0$ is a constant independent of $(m,\tau, \beta,\lambda)$. 
\end{lemma}
\begin{proof}[Proof of Lemma \ref{unifboundZm}] Let $R^{\sigma^{*}}_i(\rho^{*}_{\sigma^{*}})$ be defined as follows
$$
R^{\sigma^{*}}_i(\rho^{*}_{\sigma^{*}}) := -\frac{1}{M}\left(y_i - y^{\sigma^{*}}_{\rho^{*}_{\sigma^{*}}}(x_i)\right).
$$
By substituting the form (\ref{eq:gibbs}) of the Gibbs distribution into the free energy functional (\ref{eq:free_energy}), we have that
\begin{align*}
   &\mathcal{F}^{\sigma^{*}}(\rho) = \frac{1}{2M}\sum_{i=1}^M \left(y_i - y^{\sigma^{*}}_{\rho^{*}_{\sigma^{*}}}(x_i) \right)^2 + \frac{\lambda}{2} M(\rho^{*}_{\sigma^{*}}) \\
   &\hspace{5em}- \int  \sum_{i=1}^M \left[R^{\sigma^{*}}_i(\rho_{\sigma^{*}}^{*}) \cdot \sigma^{*}(x_i,\btheta)\right]\rho^{*}_{\sigma^{*}}(\btheta)\mathrm{d}\btheta - \frac{\lambda}{2} \int  \|\btheta\|_2^2\rho^{*}_{\sigma^{*}}(\btheta) \mathrm{d}\btheta - \frac{1}{\beta}\ln Z_{\sigma^{*}}(\beta,\lambda).
 \end{align*}
 Note that, by Fubini's theorem, we can interchange summation and integration in the first integral, since the activation and the labels are bounded. By using also the definition of $R^{\sigma^{*}}_i(\rho_{\sigma^{*}}^{*})$, we have that
 \begin{align*}
   \mathcal{F}^{\sigma^{*}}(\rho)&= \frac{1}{2M}\sum_{i=1}^M y^2_i + \frac{1}{2M}\sum_{i=1}^M \left(y^{\sigma^{*}}_{\rho^{*}_{\sigma^{*}}}(x_i)\right)^2 - \frac{1}{M}\sum_{i=1}^M y_i \cdot y^{\sigma^{*}}_{\rho^{*}_{\sigma^{*}}}(x_i) + \frac{\lambda}{2} M(\rho^{*}_{\sigma^{*}}) \\
   &\hspace{1.2em}- \frac{1}{M}\sum_{i=1}^M \left(y^{\sigma^{*}}_{\rho^{*}_{\sigma^{*}}}(x_i)\right)^2 + \frac{1}{M}\sum_{i=1}^M y_i \cdot y^{\sigma^{*}}_{\rho^{*}_{\sigma^{*}}}(x_i) - \frac{\lambda}{2} M(\rho^{*}_{\sigma^{*}}) - \frac{1}{\beta}\ln Z_{\sigma^{*}}(\beta,\lambda)
   \\&=-\frac{1}{\beta}\ln Z_{\sigma^{*}}(\beta,\lambda) - \frac{1}{2M}\sum_{i=1}^M \left(y^{\sigma^{*}}_{\rho^{*}_{\sigma^{*}}}(x_i)\right)^2 + \frac{1}{2M}\sum_{i=1}^M y_i^2 \\ &\leq -\frac{1}{\beta}\ln Z_{\sigma^{*}}(\beta,\lambda) + \frac{1}{2M}\sum_{i=1}^M y_i^2 \\ &\leq -\frac{1}{\beta}\ln Z_{\sigma^{*}}(\beta,\lambda) + C, 
\end{align*}
where $C>0$ is independent of $(m,\tau, \beta,\lambda)$. From Lemma  10.2 in \cite{mei2018mean}, we obtain that, for any $\rho \in \mathcal{K}$,
$$
\mathcal{F}(\rho) \geq \frac{1}{2} R(\rho) + \frac{\lambda}{4} M(\rho) - \frac{1}{\beta} \left[1 + 3 \log \frac{8\pi}{\beta\lambda}\right] \geq - \frac{1}{\beta} \left[1 + 3 \log \frac{8\pi}{\beta\lambda}\right],
$$
where the last inequality follows from non-negativity of $R(\rho)$ and $M(\rho)$. Combining the upper and lower bounds gives
$$
-\frac{1}{\beta}\ln Z_{\sigma^{*}}(\beta,\lambda) + C \geq - \frac{1}{\beta} \left[1 + 3 \log \frac{8\pi}{\beta\lambda}\right].
$$
After a rearrangement, we have
$$
\ln Z_{\sigma^{*}}(\beta,\lambda) \leq \beta C + 1 + 3 \log \frac{8\pi}{\beta\lambda},
$$
which concludes the proof.
\end{proof}

In order to bound the measure of $\overline{\Omega}_j(m,\beta,\lambda)$, the idea is to combine the upper bound on the partition function of Lemma \ref{unifboundZm} with a lower bound that diverges in $m$ unless $|\overline{\Omega}_j(m,\beta,\lambda)|$ vanishes. In particular, we derive a lower bound with the structure of a Gaussian integral which grows unbounded for a certain set of inputs. This set of inputs corresponds to the scenario when the Gaussian covariance has non-positive eigenvalues, and it can be expressed as the set in which the polynomials $f_j$ and $f^j$ defined in (\ref{eq:secdpoly}) are non-negative.
For brevity, we suppress the dependence of $\Omega_j$ and $\Omega^j$ on $(m,\beta,\lambda)$ in the proofs below.

\begin{lemma}[Bound on measure of cluster set]\label{worstcasebound} Assume that condition \textbf{A1} holds. For $j\in \{0, \ldots, M\}$, let $\Omega^j$ and $\Omega_j$ be defined as in (\ref{def:Om}). Then,
\begin{equation}\label{eq:ubmu}
    |\Omega^j|,\ |\Omega_j| \leq K_1 \frac{e^{\beta K_2}}{m^2},
\end{equation}
where $K_1, K_2>0$ is independent of $(m,\beta,\lambda)$.

\end{lemma}
\begin{proof}[Proof of Lemma \ref{worstcasebound}] We start with the proof for $\Omega^j$. For $j=M$, the corresponding polynomial $f^M(x)$ is equal to $1+x^2$ and therefore $|\Omega^M| = 0$. Let us now consider the case $j\neq M$, and assume that $\mu(\Omega^j) > 0$. (If that's not the case, the claim trivially holds.)

Note that, as $f^j(x)$ is a polynomial of degree at most two in $x$, $\Omega^j$ is the union of at most two intervals. Then, the following set has a non-zero Lebesgue measure in $\mathbb{R}^2$:
$$
\Omega := \{(w,b)\in\mathbb{R}_{+}\times\mathbb{R}: b = -w^mx,\ 0<w<m  ,\ x\in \Omega^j\}.
$$
Now, we can lower bound the partition function as
\begin{equation}\label{eq:lb}
    \begin{split}
Z_m(\beta, \lambda) &\geq \int_{\{|a|<m\}\times\Omega} \exp\left\{-\frac{\beta\lambda}{2}\left[\frac{2}{\lambda}\sum_{i=1}^M R_i^m(\rho^{*}_m)\cdot a^m(w^mx_i+b)^{m}_{+}+\|\btheta\|_2^2\right]\right\}\mathrm{d}\btheta\\
&= \int_{\{|a|<m\}\times\Omega} \exp\left\{-\frac{\beta\lambda}{2}\left[\frac{2}{\lambda}\sum_{i=j+1}^M R_i^m(\rho^{*}_m)\cdot a^m(w^mx_i+b)+\|\btheta\|_2^2\right]\right\}\mathrm{d}\btheta.   
    \end{split}
\end{equation}
Here, the equality in the second line follows from the following observation: if $i\in [j]$ and $(w,b)\in\Omega$, then $w^m x_i+b\le 0$ and therefore $(w^m x_i+b)_+^m=0$; if $i>j$ and $(w,b)\in\Omega$, then $0<w^mx_i+b< m^2$ ($|x|,|x_i|\leq L$, hence $|x_i-x| \leq m$, as $L$ is a numerical constant independent of $m$ and $m$ is sufficiently large by assumption \textbf{A1})
and therefore $(w^m x_i+b)_+^m=w^mx_i +b$ for all $(w, b)\in \Omega$. Thus, after the change of variables $(a,w, b) \mapsto (a,w, -w^mx)$ and an application of Tonelli's theorem, the RHS in (\ref{eq:lb}) reduces to
\begin{equation}\label{eq:lb2}
\int_{x\in\Omega^j}\int_{\{|a|<m\}\times \{0<w<m\}} w\cdot\exp\left\{-\frac{\beta\lambda}{2}\left[2aw(B^j-A^jx)+a^2+w^2(1+x^2)\right]\right\}\mathrm{d}(a,w)\mathrm{d}x.
\end{equation}
Here the coefficients $A^j$ and $B^j$ are defined as per (\ref{eq:defAB}). The term under the exponent can be rewritten as
$$
2aw(B^j-A^jx)+a^2+w^2(1+x^2)=\begin{bmatrix}
a    &
w    
\end{bmatrix}\Sigma^{-1}\begin{bmatrix}
a     \\
w    
\end{bmatrix},
$$
with
$$
\Sigma^{-1} = \begin{bmatrix}
1 & (B^j-A^jx) \\
(B^j-A^jx) & 1+x^2
\end{bmatrix}.
$$
By definition of $\Omega^j$ in conjunction with Sylvester's criterion, we have that $\Sigma^{-1}$ has a non-positive eigenvalue with corresponding eigenvector
$$
\lambda_{-} = \frac{1}{2}\left(-\sqrt{4(B^j-A^jx)^2+x^4} + x^2 + 2\right) \leq 0,\ \ \ v_{-} = \left(-\frac{x^2 + \sqrt{4(B^j-A^jx)^2+x^4}}{2(B^j-A^jx)},1\right).
$$
Furthermore, the other eigenvalue with corresponding eigenvector is given by
$$
\lambda_{+} = \frac{1}{2}\left(\sqrt{4(B^j-A^jx)^2+x^4} + x^2 + 2\right) > 0,\ \ \  v_{+} = \left(-\frac{x^2 - \sqrt{4(B^j-A^jx)^2+x^4}}{2(B^j-A^jx)},1\right).
$$
Note that $v_{-}$ and $v_{+}$ are orthogonal, and consider the following change of variables for the integral
$$
\mathbf{z}  = \begin{bmatrix}
z_1 \\
z_2
\end{bmatrix} = \begin{bmatrix}
v_{-} / \|v_{-}\|_2 \\
v_{+} / \|v_{+}\|_2
\end{bmatrix} 
\begin{bmatrix}
a \\
w
\end{bmatrix} = Q^T \begin{bmatrix}
a \\
w
\end{bmatrix} \Leftrightarrow Q \mathbf{z} = \begin{bmatrix}
a \\
w
\end{bmatrix} \Leftrightarrow \begin{bmatrix}
a(\mathbf{z}) \\
w(\mathbf{z})
\end{bmatrix}:= Q \mathbf{z} .
$$
As the matrix $Q$ is unitary, the quantity in (\ref{eq:lb2}) can be rewritten as
$$
\int_{x\in\Omega^j}\int_{\{|a(\mathbf{z})|<m\}\times \{0<w(\mathbf{z})<m\}} w(\mathbf{z})\cdot\exp\left\{-\frac{\beta\lambda}{2}\left[\lambda_{-} z_1^2+\lambda_{+}z_2^2\right]\right\}\mathrm{d}\mathbf{z} \mathrm{d}x,
$$
as the determinant of the Jacobian is 1 for any unitary linear transformation. As $\lambda_{-} \leq 0$, this quantity is lower bounded by
\begin{equation}\label{eq:lb3}
    \int_{x\in\Omega^j}\int_{\{|a(\mathbf{z})|<m\}\times \{0<w(\mathbf{z})<m\}} w(\mathbf{z})\cdot\exp\left\{-\frac{\beta\lambda}{2}\left[ \lambda_{+}z_2^2\right]\right\}\mathrm{d}\mathbf{z} \mathrm{d}x. 
\end{equation}
Notice that $\|v_{-}\| \geq 1$, $\|v_{+}\| \geq 1$ and
$
w(\mathbf{z}) = z_1 / \|v_{-}\|_2 + z_2 / \|v_{+}\|_2.
$
Thus, picking $z_1 \in (0,m/2]$ and $z_2\in(0,m/2]$ ensures that $0<w(\mathbf{z})<m$. Furthermore, these conditions on $\mathbf{z}$ do not violate the requirement on $a(\mathbf{z})$, since
$
|a(\mathbf{z})| \leq |z_1| + |z_2| \leq m.
$
Consequently, as the integrand is non-negative, the integral in (\ref{eq:lb3}) is lower bounded by
\begin{equation}\label{eq:lb4}
    \int_{x\in\Omega^j}\int_{\{0<z_1<m/2\}\times \{0<z_2<m/2\}} w(\mathbf{z})\cdot\exp\left\{-\frac{\beta\lambda}{2}\left[ \lambda_{+}z_2^2\right]\right\}\mathrm{d}\mathbf{z} \mathrm{d}x. 
\end{equation}
By Lemma \ref{lambdarisk}, $|R_i^m(\rho^{*}_m)|$ is bounded by a constant independent of  $(m, \beta,\lambda)$, since $\lambda<C_3$ from condition \textbf{A1}. Hence, $\lambda|A^jx-B^j|$ is also uniformly bounded in $(m, \beta,\lambda)$. This, in particular, implies that
$$
\lambda \cdot \lambda_{+} \leq K_1,
$$
where $K_1>0$ is independent of  $(m, \beta,\lambda)$. Furthermore, by definition of $\Omega^j$, $|B^j-A^jx|>1$, which implies that $\|v_+\|_2$ and $\|v_-\|_2$ are also upper bounded by a constant $K_2>0$ independent of $(m, \beta,\lambda)$, and therefore
$$
w(z) \leq \frac{z_1+z_2}{K_2}.
$$
With this in mind, we can then further lower bound the integral in (\ref{eq:lb4}) by
\begin{equation}\label{eq:lb5}
\begin{split}
     &\int_{x\in\Omega^j}\int_{\{0<z_1<m/2\}\times \{0<z_2<m/2\}} \frac{1}{K_2}(z_1+z_2)\cdot\exp\left\{-\frac{K_1\beta}{2}\cdot z_2^2\right\}\mathrm{d}\mathbf{z} \mathrm{d}x \\ 
    &=|\Omega^j| \int_{\{0<z_1<m/2\}\times \{0<z_2<m/2\}} \frac{1}{K_2}(z_1+z_2)\cdot\exp\left\{-\frac{K_1\beta}{2}\cdot z_2^2\right\}\mathrm{d}\mathbf{z} \\
    &\geq |\Omega^j|\int_{\{0<z_1<m/2\}\times \{0<z_2<m/2\}} \frac{1}{K_2}z_1\cdot\exp\left\{-\frac{K_1\beta}{2}\cdot z_2^2\right\}\mathrm{d}\mathbf{z} \\
    &= \frac{|\Omega^j|}{K_2} \left[\frac{m^2}{8}\sqrt{\frac{\pi}{2K_1\beta}}\mathrm{erf}\left(\frac{m\sqrt{K_1\beta}}{2\sqrt{2}}\right)\right]\\
    &\ge |\Omega^j| \frac{K_3\, m^2}{\sqrt{\beta}},
\end{split}
\end{equation}
where $K_3>0$ is independent of  $(m,\beta,\lambda)$ and in the last passage we have used that $\mathrm{erf}\left(\frac{m\sqrt{K_1\beta}}{2\sqrt{2}}\right) \geq 1/10$ for sufficiently large $m$ and $\beta$. By combining (\ref{eq:lb5}) with the upper bound on the partition function given by Lemma \ref{unifboundZm}, the desired result immediately follows and the proof for $\Omega^j$ is complete.

In regards to the argument for $\Omega_j$, for $j=0$ the result trivially holds, since $f_0(x) = 1+x^2$ and, thus, $|\Omega_0| = 0$. For $j>0$, the partition function can be lower bounded by
\begin{equation}\label{eq:lb_j}
   \int_{\{|a|<m\}\times\Omega} \exp\left\{-\frac{\beta\lambda}{2}\left[\frac{2}{\lambda}\sum_{i=1}^j R_i^m(\rho^{*}_m)\cdot a^m(w^mx_i+b)+\|\btheta\|_2^2\right]\right\}\mathrm{d}\btheta,   
\end{equation}
where the set $\Omega$ is defined on non-positive $w$ and $x\in\Omega_j$, i.e.,
$$
\Omega := \{(w,b)\in\mathbb{R}_{+}\times\mathbb{R}: b = -w^mx,\ -m<w<0  ,\ x\in \Omega_j\}.
$$
The rest of the argument remains the same by noting that with the change of variable $$(a,w, b) \mapsto (-a,-w, w^mx)$$ the quantity in \eqref{eq:lb_j} is equal to
$$
\int_{x\in\Omega_j}\int_{\{|a|<m\}\times \{0<w<m\}} w\cdot\exp\left\{-\frac{\beta\lambda}{2}\left[2aw(B_j-A_jx)+a^2+w^2(1+x^2)\right]\right\}\mathrm{d}(a,w)\mathrm{d}x,
$$
which is exactly as in \eqref{eq:lb2}, but with $x\in\Omega_j$ and the polynomial $(B_j-A_jx)$ in place of $x\in\Omega^j$ and the polynomial $(B^j-A^jx)$.
\end{proof}

In order to control the magnitude of (\ref{quant2}), it is also necessary to understand the behavior of the polynomials defined in (\ref{eq:secdpoly}). The worst case scenario, in terms of presenting a challenge to bounding the curvature, corresponds to $f^j$ or $f_j$ being arbitrarily close to zero on the whole area outside of cluster set. In fact, this would imply that the Gaussian-like integral arising in the computation of (\ref{quant2}) has arbitrary small eigenvalues. More specifically, our plan is to exploit the following bound for $x \in I_j\setminus \overline{\Omega}_j(m,\beta,\lambda)$:
\begin{equation}\label{eq:newbd}
\begin{split}
|(\ref{quant2})| \leq C \int |a|w^2 \Bigg[\exp&\left\{-\frac{\beta\lambda}{2} \cdot f^j(x) \cdot (a^2+w^2)\right\}\\&+\exp\left\{-\frac{\beta\lambda}{2} \cdot f_j(x) \cdot (a^2+w^2)\right\}\Bigg]\mathrm{d}\mathbf{\theta}.
\end{split}
\end{equation}
Now, the RHS of \eqref{eq:newbd} diverges (and, therefore, the bound is useless), if either of the polynomials is arbitrarily close to zero outside of the cluster set. Fortunately, we are able to prove that this cannot happen: in Lemma \ref{welldefquad} we show that $f^j(x)$ and $f_j(x)$ can be small only when $x$ approaches the cluster set, i.e.,
$$
f^j(x), f_j(x) \geq  \min\{C^j(x), C_j(x), 1\},
$$
where $C^j(x), C_j(x)$ are defined in \eqref{eq:totallbpoly} and, because of the condition on their coefficients $\{K_i\}_{i=1}^4$, they cannot be arbitrarily close to $0$ in any interval $I_j$.

As a preliminary step towards the proof of Lemma \ref{welldefquad}, we  show an auxiliary result for polynomials of a certain form. Fix some interval $I=[I_l,I_r]\subset\mathbb{R}$. Given two quantities $a,b\in\mathbb{R}$, consider the following polynomial of degree at most two
\begin{equation}\label{def:genericpoly}
    P_2(x) := (1-a^2) \cdot x^2 + 2ab  \cdot x + (1-b^2),\quad  x \in I,
\end{equation}
where we suppress the dependence on $(a,b)$, i.e., $P_2(x;a,b)=P_2(x)$, for more compact notation. In addition, let $\Omega_{+}$ be the subset of $I$ on which $P_2$ is strictly positive, i.e.,
$$
\Omega_{+} := \{x\in I: P_2(x) > 0\}.
$$
For a fixed small constant $C_{\Omega} > 0$, define the set of admissible coefficients as follows 
\begin{equation}\label{goodab}
\mathcal{U} := \{(a,b)\in\mathbb{R}^2: |\Omega_+| \geq C_{\Omega}\}.
\end{equation}
Given $(a,b) \in \mathcal{U}$ and $x\in  \Omega_{+}$, we define the critical point $x_c$ of the polynomial $P_2$ associated with $x$ and $\Omega_{+}$ in the same fashion as in Definition \ref{def:critical}, after replacing $f^j(\cdot)$ with $P_2(\cdot)$ and $I_j \setminus \Omega^j(m,\beta,\lambda)$ with $\Omega_{+}$.  Notice that, since $\Omega_{+}$ has strictly positive Lebesgue measure for $(a,b)\in\mathcal{U}$, the critical point is well-defined and, in particular, $x_c \in I$ always holds.

\begin{lemma}[Lower bound on polynomial]\label{welldefquadgeneric} 
Fix some $C_{\Omega}$ such that $\mathcal{U}$, as defined in (\ref{goodab}), is of positive measure.
Pick some interval $(a,b)\in\mathcal{U}$. Let $x\in \Omega_{+}$ and $x_c$ be the critical point associated to $x$. Then, the following holds
\begin{equation}\label{eq:p2def}
P_2(x) \geq \alpha_{2} (x-x_c)^2 + \alpha_{1} |x-x_c| + \alpha_{0},
\end{equation}
where $\alpha_{0}, \alpha_{1}, \alpha_{2} \geq 0$ and at least one of them is lower bounded by a strictly positive constant depending on $C_{\Omega}$ but independent of the choice of $(a,b)\in\mathcal{U}$.
\end{lemma}

We defer the proof of Lemma \ref{welldefquadgeneric} to Appendix \ref{appendix:lbpoly}. Recall the definition of the polynomial $f^j(x)$ given in (\ref{eq:secdpoly}), and notice that expression can be rearranged such that $f^j(x)$ is in the form of (\ref{def:genericpoly}), namely
$$
f^j(x) = 1 + x^2 - (A^jx - B^j)^2 = (1-(A^j)^2)  x^2 + 2A^jB^j x + (1 - (B^j)^2).
$$
In this view, the following result follows from Lemma \ref{welldefquadgeneric}.

\begin{lemma}[Well-defined quadratic form]\label{welldefquad} Assume that $(A^j, B^j)\in\mathcal{U}$, i.e., $|I_j \setminus \Omega^j|$ is lower bounded by a positive constant. Given $x\in I_j \setminus \Omega^j$, let $x_c$ be the critical point associated to $x$. Then, we have that 
\begin{equation}\label{eq:Csupj}
f^j(x) \geq C^j(x):= \gamma_1 (x-x_c)^2 + \gamma_2,
\end{equation}
where $\gamma_1, \gamma_2 > 0$ and either $\gamma_1 > \varepsilon$ or $\gamma_2 > \varepsilon$ for some $\varepsilon > 0$  that is independent of $(A^j,B^j)$ but depending on $C_{\Omega}$ as  appearing in the definition of $\mathcal{U}$.
\end{lemma}
\begin{proof}[Proof of Lemma \ref{welldefquad}] Note that $I_j\setminus\Omega^j$ is the set in which $f^j$ is strictly positive. Hence, since $|I_j \setminus \Omega^j|$ is lower bounded by a positive constant independent of $A^j, B^j$, we can apply Lemma \ref{welldefquadgeneric} to get
$$
f^j(x) \geq \alpha_{2} (x-x_c)^2 + \alpha_{1} |x-x_c| + \alpha_{0},
$$
where $\alpha_{0}, \alpha_{1}, \alpha_{2} \geq 0$ and at least one of them is lower bounded by a strictly positive constant independent of $(A^j, B^j)$. Thus, since each term of the RHS above is non-negative, we get
$$
f^j(x) \geq \alpha_i|x-x_c|^i + \alpha_0,
$$
where $i=\arg\max_{j\in \{1, 2\}}\alpha_j$. Furthermore, as $|x-x_c| \leq |I_j|$, we have
$$
f^j(x) \geq \frac{\alpha_i}{|I_j|^{2-i}}|x-x_c|^2 + \alpha_0.
$$
Now, either $\alpha_i$ or $\alpha_0$ as well as $1/|I_j|$ are lower bounded by strictly positive constants independent of $(A^j, B^j)$. Thus, taking $\gamma_1=\alpha_i/|I_j|^{2-i}$ and $\gamma_2 = \alpha_0$ concludes the proof.
\end{proof}

Let us point out that, although $\varepsilon$ does not depend on the values of $(A^j,B^j) \in \mathcal{U}$, the position of a critical point $x_c$ \emph{depends} on $(A^j,B^j)$.

In a similar fashion, we define $\bar{\mathcal{U}}$ to be the set of 
admissible $(A_j,B_j)$ as in (\ref{goodab}), and given $x\in I_j \setminus \Omega_j$, we let $\bar{x}_c$ be the critical point associated to $x$ and $\Omega_j$. Then,
a result analogous to Lemma \ref{welldefquad} holds for $f_j(x)$: 
\begin{equation}\label{eq:Cundj}
f_j(x) \geq C_j(x):= \gamma_3 (x-\bar{x}_c)^2 + \gamma_4,    
\end{equation}
where $\gamma_3, \gamma_4 > 0$ and either $\gamma_3 > \varepsilon$ or $\gamma_4 > \varepsilon$ for some $\varepsilon > 0$  that is independent of the choice of $(A_j, B_j) \in \bar{\mathcal{U}}$.

The last ingredient for the proof of the vanishing curvature phenomenon is the control of the decay of the partition function $Z_m(\beta,\lambda)$ as $\beta\rightarrow0$.

\begin{lemma}[Lower bound on partition function independent of $m$]\label{finiteZ} Assume that condition \textbf{A1} holds. Then, 
$$
Z_m(\beta,\lambda) \geq \frac{C}{\sqrt{\beta^3\lambda^{3/2}}},
$$
for some $C>0$ that is independent of $(m,\beta,\lambda)$.
\end{lemma}

The proof of Lemma \ref{finiteZ} is deferred to Appendix \ref{appendix:risk}. At this point, we are ready to provide an upper bound on the magnitude of (\ref{quant2}).

\begin{lemma}[Integral upper bound]\label{UBintegral} 
Assume that condition \textbf{A1} holds. Furthermore, assume that $m>e^{\beta K_2}$, where $K_2$ is given in (\ref{eq:ubmu}). Fix $j\in \{0, \ldots, M\}$. Then, for any $x \in I_j \setminus \left(\Omega^j \cup \Omega_{j} \right)$, 
$$
\left|\int a^m(w^m)^2 {\rho}^{*}_{m}(a,w,-w^mx)\mathrm{d}a\mathrm{d}w\right| \leq \frac{K}{\beta\lambda^{7/4}(\bar{C}^j(x))^2},
$$
where $K>0$ is independent of $(m,\beta,\lambda)$, $\bar{C}^j(x) := \min\left\{C_{j}(x), C^j(x), 1\right\}$, and $C^j(x), C_j(x)$ are given by (\ref{eq:Csupj}) and  (\ref{eq:Cundj}), respectively.
\end{lemma}
\begin{proof}[Proof of Lemma \ref{UBintegral}] 
Note that the following upper bound holds
$$
\left|\int a^m(w^m)^2 {\rho}^{*}_{m}(a,w,-w^mx)\mathrm{d}a\mathrm{d}w\right| \leq I(x):=\int |a^m|(w^m)^2 \rho^{*}_{m}(a,w,-w^mx)\mathrm{d}a\mathrm{d}w.
$$
Let us now decompose the integral $I(x)$ depending on the sign of $w$, i.e., $$Z_m(\beta,\lambda) \cdot I(x) = I^j(x) + I_j(x),$$ where
\begin{align*}
   &I^j(x) := \int_{\{a\in\mathbb{R}\}\times\{w\geq0\}} |a^m|(w^m)^2 \exp\left\{-\beta\Psi^j(a,w,\rho^{*}_m)\right\}\mathrm{d}a\mathrm{d}w,\\
    &I_j(x) := \int_{\{a\in\mathbb{R}\}\times\{w<0\}} |a^m|(w^m)^2 \exp\left\{-\beta\Psi_j(a,w,\rho^{*}_m)\right\}\mathrm{d}a\mathrm{d}w,
\end{align*}
and, recalling the form of $\rho^{*}_{m}(a,w,-w^mx)$ from \eqref{eq:rhodensitytm2}, the corresponding potentials are given by
\begin{align*}
    &\Psi^j(a,w,\rho) = \sum\limits_{i=j+1}^M R^m_i(\rho) \cdot a^mw^m (x_i - x) + \frac{\lambda}{2}\left\{a^2 + w^2 + (w^m)^2x^2\right\}, \\
    &\Psi_j(a,w,\rho) = \sum\limits_{i=1}^j R^m_i(\rho) \cdot a^mw^m (x_i - x) + \frac{\lambda}{2}\left\{a^2 + w^2 + (w^m)^2x^2\right\}.
\end{align*}
By recalling from (\ref{eq:defAB}) the definitions of $A^j,A_j,B^j$ and $B_j$, we obtain the following upper bounds.
\begin{align}
  &I^j(x) \le 2\int_{\{a\geq 0\}\times\{w\geq0\}} aw^2 \exp\left\{-\frac{\beta\lambda}{2}\left[-2aw^m|B^j-A^jx| + a^2 + w^2 + (w^m)^2x^2\right]\right\}\mathrm{d}a\mathrm{d}w,\label{1.9.1}\\
    &I_j(x) \le 2\int_{\{a\geq 0\}\times\{w<0\}} aw^2 \exp\left\{-\frac{\beta\lambda}{2}\left[2aw^m|B_j-A_jx| + a^2 + w^2 + (w^m)^2x^2\right]\right\}\mathrm{d}a\mathrm{d}w.\label{1.9.3}
\end{align}

\noindent Let us analyze the RHS of (\ref{1.9.1}). This term can be rewritten as
\begin{align}\label{eq:int1}
    \begin{split}
    2\int_{\{a\geq 0\}\times\{w\geq0\}} aw^2 &\exp\left\{-\frac{\beta\lambda}{2}\left[-2aw^m|A^jx-B^j| + a^2 + (w^m)^2(A^jx-B^j)^2  \right]\right\}  \\
    \cdot & \exp\left\{-\frac{\beta\lambda}{2}\left[w^2 + (w^m)^2x^2 - (w^m)^2(A^jx-B^j)^2 \right]\right\}\mathrm{d}a\mathrm{d}w.
    \end{split}
\end{align}
Note that
$$
|\Omega^j|\le \frac{K_1\, e^{\beta K_2}}{m^2}\le \frac{K_1}{e^{\beta K_2}}, 
$$
where the first inequality follows from Lemma \ref{worstcasebound}, and the second inequality uses that $m>e^{\beta K_2}$. Therefore, for sufficiently large $\beta$, 
$|\Omega^j|$ is smaller than $|I_j|/2$, and therefore $|I_j\setminus\Omega^j|$ is lower bounded by $|I_j|/2$. At this point, 
we can apply Lemma \ref{welldefquad}
which gives that $1+x^2 -  (A^jx-B^j)^2 \geq C^j(x)\ge \bar{C}^j(x):=\min\left\{ C^j(x), C_j(x), 1\right\}$. Thus, (\ref{eq:int1}) is upper bounded by
\begin{equation}\label{eq:ineqlm1}
\begin{split}
    &2\int_{\{a\geq 0\}\times\{w\geq0\}} aw^2 \exp\left\{-\frac{\beta\lambda}{2}\left(a - |B^j-A^jx|w^m\right)^2\right\} \\
&\hspace{5.5em}\cdot \exp\left\{-\frac{\beta\lambda}{2}\left[w^2 - (w^m)^2(1-\bar{C}^j(x)) \right]\right\}\mathrm{d}a\mathrm{d}w\\
&=2\int_{\{w\geq 0\}} w^2  \exp\left\{-\frac{\beta\lambda}{2}\left[w^2 - (w^m)^2(1-\bar{C}^j(x)) \right]\right\}\sqrt{\frac{2\pi}{\beta\lambda}}\mathbb E\left[(A)_+\right]\mathrm{d}w,
\end{split}
\end{equation}
where $A\sim \mathcal{N}(|B^j-A^jx|w^m,(\beta\lambda)^{-1})$. Furthermore, the following chain of inequalities hold:
\begin{equation}\label{eq:ineqlm2}
\mathbb{E} \left[(A)_+\right] \le \mathbb{E} \left[|A|\right]\le \sqrt{\mathbb{E} \left[A^2\right]} = \sqrt{|B^j-A^jx|^2(w^m)^2 + \frac{1}{\beta\lambda}},
\end{equation}
where the second passage follows from Jensen's inequality. By using (\ref{eq:ineqlm2}), the RHS of (\ref{eq:ineqlm1}) is upper bounded by
\begin{align*} \frac{2\sqrt{2\pi}}{\sqrt{\beta\lambda}}\int_{\{w\geq0\} } &\sqrt{(B^j-A^jx)^2(w^m)^2 + \frac{1}{\beta\lambda}} \\
\cdot\ w^2&\exp\left\{-\frac{\beta\lambda}{2}\left[w^2 - (w^m)^2(1-\bar{C}^j(x)) \right]\right\}\mathrm{d}w.
\end{align*}
Applying Lemma \ref{welldefquad} again to obtain $(A^jx-B^j)^2\le 1+x^2-\bar{C}^j(x)\le 1+x^2$ and noting by definition that $(w^m)^2 \leq w^2$, we now upper bound this last term by
\begin{equation}
\begin{split}
& 2\sqrt{2\pi}\int_{\{w\geq 0\}}\sqrt{\frac{w^2(1+x^2)}{\beta\lambda} + \frac{1}{\beta^2\lambda^2}}\
\cdot\ w^2\exp\left\{-\frac{\beta\lambda}{2}\left[w^2 - (w^m)^2(1-\bar{C}^j(x)) \right]\right\}\mathrm{d}w\\
&\le 2\sqrt{2\pi}\int_{\{w\in\mathbb{R}\}}\sqrt{\frac{w^2(1+x^2)}{\beta\lambda} + \frac{1}{\beta^2\lambda^2}}\
\cdot w^2\exp\left\{-\frac{\beta\lambda}{2}\left[\bar{C}^j(x) \cdot w^2 \right]\right\}\mathrm{d}w\\
&\le 2\sqrt{2\pi}\int_{\{w\in\mathbb{R}\}}\left(\sqrt{\frac{w^2(1+x^2)}{\beta\lambda}} + \sqrt{\frac{1}{\beta^2\lambda^2}}\right)\
\cdot w^2\exp\left\{-\frac{\beta\lambda}{2}\left[\bar{C}^j(x) \cdot w^2 \right]\right\}\mathrm{d}w,
\end{split}
\end{equation}
where in the second line we use that $1-\bar{C}^j(x)\ge 0$ and again that $(w^m)^2 \leq w^2$, and in the third line we use that $\sqrt{u+v} \leq \sqrt{u} + \sqrt{v}$.

Finally, computing explicitly the last integral gives the following upper bound on the RHS of (\ref{1.9.1}) and consequently on $I^j(x)$:
$$
I^j(x) \leq 4\pi\sqrt{\frac{1}{\beta^2\lambda^2}}\
\cdot\ \sqrt{\frac{1}{(\bar{C}^j(x))^3\beta^3\lambda^3}} + 2\sqrt{2\pi} \sqrt{\frac{1+x^2}{\beta\lambda}} \sqrt{\frac{1}{(\bar{C}^j(x))^4\beta^4\lambda^4}}.
$$
By following the similar passages, we obtain the same upper bound for $I_j(x)$. By using the lower bound on the partition function shown in Lemma \ref{finiteZ}, we conclude that
$$
I(x) =\frac{I^j(x)+I_j(x)}{Z_m(\beta, \lambda)} \leq \frac{K}{\beta\lambda^{7/4}(\bar{C}^j(x))^2},
$$
where $K>0$ is independent of $(m,\beta,\lambda)$, and the proof is complete.
\end{proof}

The proof of Theorem \ref{mainT0} is an immediate consequence of the results presented so far.

\begin{proof}[Proof of Theorem \ref{mainT0}]
The proof of \eqref{eq:decay} follows from Lemmas \ref{convdelta} and \ref{UBintegral}, and the proof of \eqref{eq:controlnorm} follows from Lemma \ref{worstcasebound}.
\end{proof}

\subsection{Proof of Theorem \ref{mainT1}}\label{subsec:pfth2}

To summarize, at this point we have shown that as $\beta\to\infty$ the second derivative of the predictor vanishes outside the cluster set, and that the size of the cluster set shrinks to concentrate on at most 3 points per prediction interval. With these results in mind, we are ready to provide the proof for Theorem \ref{mainT1}.

\begin{proof}[Proof of Theorem \ref{mainT1}]
The predictor evaluated at the Gibbs distribution is given by
$$
y_n(x) = \int a^{\tau,m}(w^{\tau,m}x + b)^{m}_{\tau}\rho^{*}_{\tau,m}(\btheta)\mathrm{d}\btheta, 
$$
where $n=(\tau,m,\beta,\lambda)$ denotes the aggregated index and we suppress the dependence on $(\beta,\lambda)$ in $\rho^{*}_{\tau,m}$ for convenience. By Lemma \ref{unifboundM_rhomtau}, there exists $\tau(m,\beta,\lambda)$ such that, for any $\tau>\tau(m, \beta, \lambda)$,
\begin{equation}\label{eq:secmombd}
 M(\rho^{*}_{\tau,m})\le C,   
\end{equation}
for some $C>0$ independent of $(\tau,m,\beta,\lambda)$. We start by showing that the family of predictors $\{y_n\}$ is equi-Lipschitz for $\infty > \tau > \tau(m,\beta,\lambda)$. First, note that
\begin{equation}\label{eq:push1}
\frac{\partial }{\partial x} y_n(x) = \int \frac{\partial }{\partial x} \Big[ a^{\tau,m}(w^{\tau,m}x + b)^{m}_{\tau}\Big]\rho^{*}_{\tau,m}(\btheta)\mathrm{d}\btheta,
\end{equation}
since the derivative can be pushed inside by the same line of arguments as given in the proof of Lemma \ref{convdelta}. Next, we have that, by construction of the activation, the following holds
$$
\int \frac{\partial }{\partial x} \Big[ a^{\tau,m}(w^{\tau,m}x + b)^{m}_{\tau}\Big]\rho^{*}_{\tau,m}(\btheta)\mathrm{d}\btheta \leq C_1 \int |a^{\tau,m} w^{\tau,m}| \rho^{*}_{\tau,m}(\btheta)\mathrm{d}\btheta, 
$$
where, from here on, $C_1 > 0$ denotes a generic constant which might change from line to line, but is independent of $(\tau,m,\beta,\lambda)$. By construction, for any $u\in\mathbb{R}$, it holds that $|u^{\tau,m}| \leq |u|$. Thus, we have that
$$
\int \frac{\partial }{\partial x} \Big[ a^{\tau,m}(w^{\tau,m}x + b)^{m}_{\tau}\Big]\rho^{*}_{\tau,m}(\btheta)\mathrm{d}\btheta \leq C_1 \int |a w| \rho^{*}_{\tau,m}(\btheta)\mathrm{d}\btheta.
$$
Using the Cauchy-Schwartz inequality and \eqref{eq:secmombd}, we obtain that
\begin{equation}\label{eq:push2}
    \int \frac{\partial }{\partial x} \Big[ a^{\tau,m}(w^{\tau,m}x + b)^{m}_{\tau}\Big]\rho^{*}_{\tau,m}(\btheta)\mathrm{d}\btheta \leq C_1 M(\rho^{*}_{\tau,m}) \leq C_1.
\end{equation}
By combining \eqref{eq:push1} and \eqref{eq:push2}, we have shown that the family $\{y_n\}$ for $\tau > \tau(m,\beta,\lambda)$ is equi-Lipschitz, as the derivatives are uniformly bounded. By using a similar argument, we can show that the same result holds for the predictor itself, i.e., for all $x\in \bigcup_{j=0}^M I_j$, $y_n(x)$ is uniformly bounded.

Note that Theorem \ref{mainT0} considers the curvature of points outside the cluster set, and it gives an upper bound which diverges when $\bar{C}^j(x)$ approaches $0$ for some $j\in [M]$. Thus, our next step is to develop the analytical machinery to make this scenario impossible.
Let us recall Definitions \eqref{eq:totallbpoly} and \eqref{eq:barCdef}. Then, by Lemma \ref{welldefquad}, we have that 
$$
\bar{C}^j(x) \geq \min\{\gamma_1(x-x_c)^2 + \gamma_2,\ \gamma_3(x-\bar{x}_c)^2 + \gamma_4\},
$$
where $\gamma_1,\gamma_2,\gamma_3,\gamma_4 > 0$ and $\min\{\max\{\gamma_1, \gamma_2\}, \max\{\gamma_3, \gamma_4\}\}>\varepsilon$, for some $\varepsilon>0$ that is independent of $(m, \beta, \lambda)$. Let us focus on the term $\gamma_1(x-x_c)^2 + \gamma_2$. If $\gamma_2=0$ or it approaches $0$ (as $m, \beta\to\infty$), then we extend $\Omega^j(m,\beta,\lambda)$ as
$$
\mathrm{ext}_{\delta}(\Omega^j(m,\beta,\lambda)) := \left\{x \in I_j: \min_{x'\in \Omega^j(m,\beta,\lambda)\cup \{x_c\}}|x-x'| < \delta\right\}.
$$
Note that adding the singleton $\{x_c\}$ to the argument of the $\min$ allows us to also cover the case in which $\Omega^j(m,\beta,\lambda)$ is empty. Otherwise, i.e., if $\gamma_2>\varepsilon$ for some $\varepsilon>0$ that is independent of $(m, \beta, \lambda)$, the upper bound on the curvature does not diverge and we set $\mathrm{ext}_{\delta}(\Omega^j(m,\beta,\lambda)):=\Omega^j(m,\beta,\lambda)$. In a similar fashion, we define the extension  of $\Omega_j(m,\beta,\lambda)$ by $\mathrm{ext}_{\delta}(\Omega_j(m,\beta,\lambda))$.

Let $\bar{\Omega}^j_{\rm ext}$ be the union of $\mathrm{ext}_{\delta}(\Omega^j(m,\beta,\lambda))$ and $\mathrm{ext}_{\delta}(\Omega_j(m,\beta,\lambda))$, where we drop the explicit dependence of $\bar{\Omega}^j_{\rm ext}$ on $(\delta,m,\beta,\lambda)$ for convenience. Then, since $f^j$ and $f_j$ are polynomials of degree two, the extended set $\bar{\Omega}^j_{\rm ext}$ (just like $\bar{\Omega}^j$) is the union of at most three disjoint open intervals, i.e., 
$$
\bar{\Omega}^j_{\rm ext} = A_1^j \cup A_2^j \cup A_3^j,
$$
where $\{A_i^j\}_{i=1}^3$ denote such (possibly empty) open intervals. Furthermore, $I_j \setminus \bar{\Omega}^j_{\rm ext}$ is the union of 
at most three disjoint closed intervals, i.e.,  
$$
I_j \setminus \bar{\Omega}^j_{\rm ext} = B_1^j \cup B_2^j \cup B_3^j,
$$
where $\{B_i^j\}_{i=1}^3$ denote such (possibly empty) closed intervals.

At this point, we are ready to show that, for all closed intervals $\{B_i^j\}_{i=1}^3$, the predictor $y_n$ can be approximated arbitrarily well by a linear function (which may be different in different closed intervals). Note that $y_n$ is twice continuously differentiable for $\tau < \infty$, and fix $\tilde{x} \in B_i^j$. Then, by combining Taylor's theorem with the result of Theorem \ref{mainT0}, we obtain that, for any $x \in B_i^j$,
\begin{equation}\label{eq:Tay1}
\lim_{\tau\to\infty}|y_n(x) - y_n(\tilde{x}) - y'_n(\tilde{x}) (x - \tilde{x})| \leq \mathcal{O}\left(\frac{1}{m\lambda} + \frac{1}{\delta^4 \cdot \beta\lambda^{7/4}}\right),
\end{equation}
where we use that $|x - x_c| \geq \delta$ by construction of the extended set $\bar{\Omega}^j_{\rm ext}$. Let us define
$$
f_n^i(x) = y_n(\tilde{x}) - y'_n(\tilde{x}) (x - \tilde{x}).
$$
Then, by picking a sufficiently small $\delta$, \eqref{eq:Tay1} implies that, as $m\lambda \rightarrow \infty$ and $ \beta\lambda^{7/4} \rightarrow \infty$, for all $x\in B_i^j$,
\begin{equation}\label{eq:fniy}
|y_n(x) - f_n^i(x)| \rightarrow 0.
\end{equation}
We remark that, as shown previously, the coefficients $y_n(\tilde{x})$ and $y'_n(\tilde{x})$ are uniformly bounded in absolute value.

Let us now consider the open intervals $\{A_i^j\}_{i=1}^3$. For any $x \in A_i^j$, let
$$
x' = \argmin_{y \not \in A_i^j} |x-y|,
$$
and note that, by definition, $x' \in B_{\tilde{\imath}}^j$ for some $\tilde{\imath} \in \{1, 2, 3\}$. By picking the linear approximation $f_n^i$ that corresponds to $B_i^j$ and by using the triangle inequality, we obtain that
\begin{align}\label{tm1:eq1}
    |y_n(x) - f^i_n(x)| &\leq |y_n(x) - y_n(x')| + |y_n(x') - f_n^i(x')| + |f_n^i(x') - f_n^i(x)|
    \nonumber\\&\leq \mathcal{O}\left(|x-x'| + |y_n(x') - f_n^i(x')| \right),
\end{align}
where the second inequality is due to the fact that the families $\{y_n\}$ and $\{f_n^i\}$ are equi-Lipschitz. From \eqref{eq:fniy} the second term in the RHS in (\ref{tm1:eq1}) vanishes. As for the first term, by construction of the extension, together with the result of Lemma \ref{worstcasebound}, we have that
$$
|x-x'| \leq \mathcal{O}\left(\frac{e^{\beta K_2}}{m^2} + \delta\right),
$$
for some $K_2>0$ independent of $(m, \beta, \lambda)$. Thus, by picking a sufficiently small $\delta$ and $m > e^{\beta K_2}$, we conclude that the first term in the RHS in (\ref{tm1:eq1}) also vanishes. 

So far, we have showed that, both inside and outside of the extension of the cluster set, the predictor $y_n$ is well approximated by linear functions. It remains to prove that the linear pieces connect, i.e., there exists $\hat{x}\in \bar{\Omega}^j_{\rm ext}$ such that, for two neighboring linearities $f_n^{i}$ and $f_n^{i+1}$ (possibly belonging to different intervals), the following holds
$$
f_n^i(\hat{x}) - f_n^{i+1}(\hat{x}) = 0.
$$
This claim follows from Lipschitz arguments similar to those presented above, and the proof is complete.
\end{proof}

\subsection{Proof of Corollary \ref{mainT2}}\label{subsec:pfcor}

At this point, we have proved a result about the structure of the predictor coming from the minimizer of the free energy \eqref{eq:free_energy}. By using the mean-field analysis in \cite{mei2018mean}, we finally show that this structural result holds for the predictor obtained from a wide two-layer ReLU network. 

\begin{proof}[Proof of Corollary \ref{mainT2}] 
First, we show that, as $t\to\infty$, the second derivative of the predictor evaluated on the solution $\rho_t$ of the flow  (\ref{eq:PDE}) converges to the same quantity evaluated on the Gibbs minimizer $\rho^*_{\tau, m}$. To do so, we decompose the integral involving $\rho_t$ as in Lemma \ref{convdelta} (cf. \eqref{decompose1}):
\begin{align}\label{cor:decompose1}
    \int a^{\tau,m}&\left[\frac{\partial^2}{(\partial x)^2}(w^{\tau,m}x+b)_{\tau}^m\right] \rho_t(\btheta)\mathrm{d}\btheta  \nonumber\\ =& \int_{w^{\tau,m}x+b\leq x_m} a^{\tau,m}\left[\frac{\partial^2}{(\partial x)^2}(w^{\tau,m}x+b)_{\tau}\right] \rho_{t}(\btheta)\mathrm{d}\btheta\  \nonumber\\
    &\hspace{2em}+ \int_{w^{\tau,m}x+b > x_m} a^{\tau,m}(w^{\tau,m})^2\left[\frac{\partial^2}{(\partial u)^2}\phi_{\tau,m}(u)\Bigg|_{u=w^{\tau,m}x+b}\right] \rho_{t}(\btheta)\mathrm{d}\btheta.
\end{align}

Next, we show that a technical condition bounding the free energy at initialization appearing in the statement of Theorem 4 in \cite{mei2018mean} is satisfied under the assumption $M(\rho_0) < \infty$ and $H(\rho_0) > -\infty$. Recalling the sandwich bound for the truncated soft-plus activation \eqref{eq:sandwich} and the fact that that $\tau \geq 1$ by condition \textbf{A1}, an application of Cauchy-Schwarz inequality gives
$$
R^{\tau,m}(\rho_0) < C M(\rho_0) + C' < \infty,
$$
where $C,C' > 0$ are some numerical constants independent of $(\tau,m)$. 
This readily implies that
$$
\mathcal{F}^{\tau,m}(\rho_0) < \infty,
$$
since $\lambda$ and $\beta^{-1}$ are upper-bounded by assumption \textbf{A1}. 

Now we can apply Theorem 4 in \cite{mei2018mean} to conclude that, as $t\rightarrow\infty$, 
$$
\rho_t \rightharpoonup \rho^{*}_{\tau,m}.
$$

Thus, as the terms inside the integrals in \eqref{cor:decompose1} are all bounded for fixed $(\tau, m, \beta, \lambda)$, by definition of weak convergence, we get that, as $t\rightarrow\infty$,
$$
\int a^{\tau,m}\left[\frac{\partial^2}{(\partial x)^2}(w^{\tau,m}x+b)_{\tau}^m\right] \rho_t(\btheta)\mathrm{d}\btheta \rightarrow \int a^{\tau,m}\left[\frac{\partial^2}{(\partial x)^2}(w^{\tau,m}x+b)_{\tau}^m\right] \rho^{*}_{\tau,m}(\btheta)\mathrm{d}\btheta.
$$
Consequently, since the derivative operator can be pushed inside by the same arguments as in Lemma \ref{convdelta}, we have that, as $t\rightarrow\infty$, the following pointwise convergence holds
\begin{equation}\label{eq:conv2der}
\frac{\partial^2}{(\partial x)^2}\int a^{\tau,m}(w^{\tau,m}x+b)_{\tau}^m \rho_t(\btheta)\mathrm{d}\btheta \rightarrow \frac{\partial^2}{(\partial x)^2}\int a^{\tau,m}(w^{\tau,m}x+b)_{\tau}^m \rho^{*}_{\tau,m}(\btheta)\mathrm{d}\btheta.
\end{equation}

Next, we show that the second derivative of the predictor obtained from the two-layer ReLU network also converges to the same limit. Recall that $\sigma^{*}(x,\btheta) = a^{\tau,m}(w^{\tau,m}x+b)^{m}_{\tau}$. Then, by Theorem 3 in \cite{mei2018mean}, we have that, almost surely, as $N\rightarrow\infty, \,\,\varepsilon_N\rightarrow0$
\begin{equation}\label{eq:conv2der2}
    \frac{\partial^2}{(\partial x)^2}\left[\frac{1}{N}\sum_{i=1}^N \sigma^{*}\left(x,\btheta_i^{\floor*{t/\varepsilon}}\right)\right]\rightarrow\frac{\partial^2}{(\partial x)^2}\int a^{\tau,m}(w^{\tau,m}x+b)_{\tau}^m\rho_t(\btheta)\mathrm{d}\btheta
\end{equation}
along any sequence $\{\varepsilon_N\}$ such that $\varepsilon_N\log(N/\varepsilon_N)\rightarrow 0$ and $N/\log(N/\varepsilon_N)\rightarrow\infty$. By combining \eqref{eq:conv2der} and \eqref{eq:conv2der2}, we obtain that the desired convergence result holds for the LHS of \eqref{eq:conv2der}.

Another application of Theorem 3 of \cite{mei2018mean}, together with the fact that the second moment of the flow solution $\rho_t$ is uniformly bounded along the sequence $t\rightarrow\infty$ (cf. Lemma 10.2 in \cite{mei2018mean}, following Proposition 4.1 in \cite{jordan1998variational}), gives that the gradients
$$
\frac{\partial}{\partial x}\left[\frac{1}{N}\sum_{i=1}^N \sigma^{*}\left(x,\btheta_i^{\floor*{t/\varepsilon}}\right)\right]
$$
are almost surely uniformly bounded. This fact, in turn, implies that the corresponding predictor is almost surely equi-Lipschitz.
In a similar fashion, we also have that the predictor itself is almost surely uniformly bounded in absolute value. 

At this point, the desired result follows from the same line of arguments as in the proof of Theorem \ref{mainT1}.
\end{proof}

\section{Knots Inside the Interval}\label{section:knotsarethere}

In this section, we provide an explicit example of a 2-point dataset such that the SGD solution exhibits a change of tangent (or ``knot'') \emph{inside} the training interval. To do so, we will show that neural networks implementing a linear function without knots on the prediction interval \emph{cannot} minimize the free energy \eqref{eq:free_energy}. To simplify the analysis, throughout the section we omit the limits in $(\tau,m)$, i.e., we consider directly ReLU activations (this corresponds to taking $\tau = m = \infty$). Similar arguments apply to the case of sufficiently large parameters $\tau$ and $m$.

\subsection{Noiseless Regime}

\begin{figure}[t!]
    \centering
    \subfloat[\label{subfig:plots}]{\includegraphics[width=.399\columnwidth]{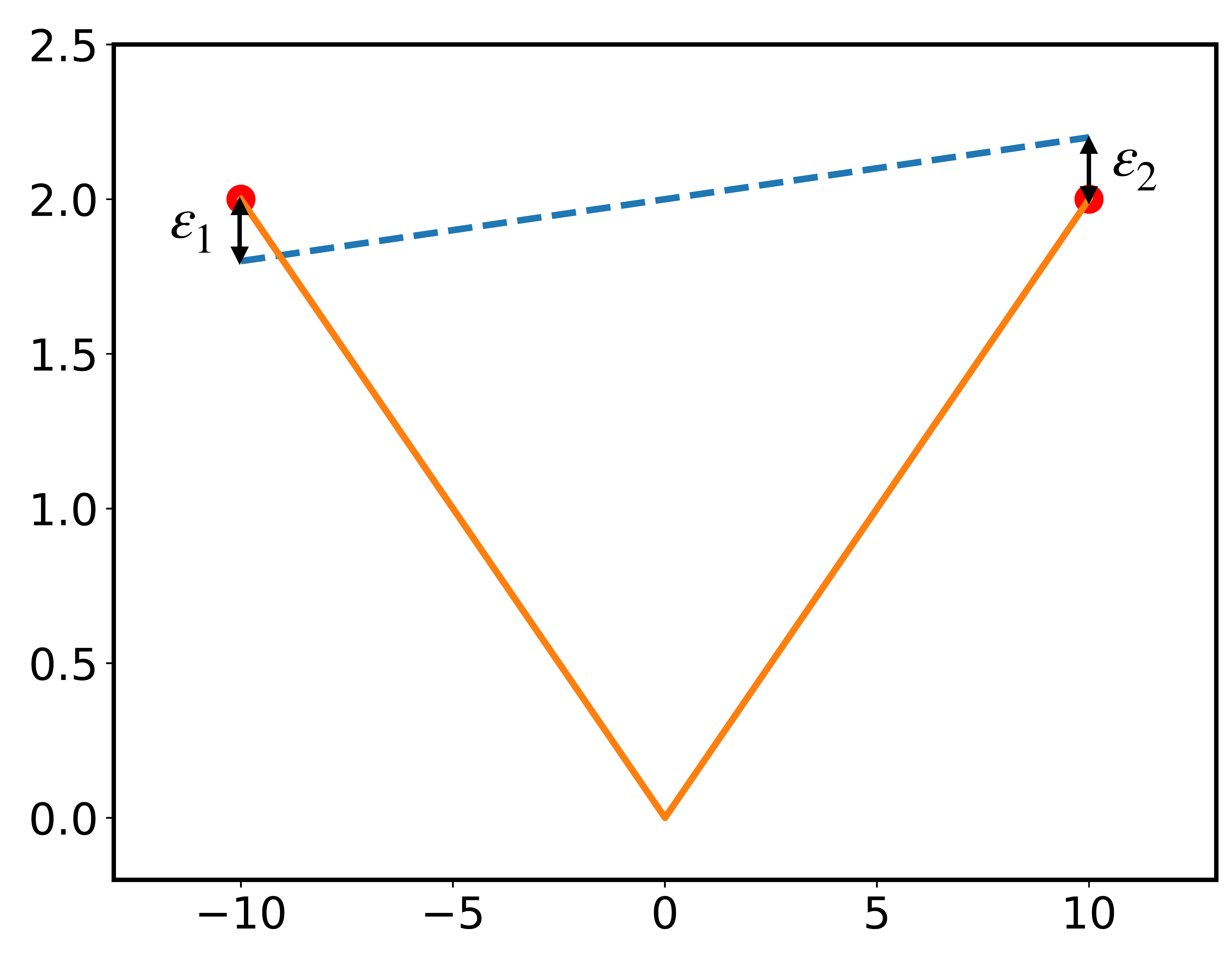}}
    \hspace{3em}
    \subfloat[\label{subfig:plots2}]{\includegraphics[width=.4\columnwidth]{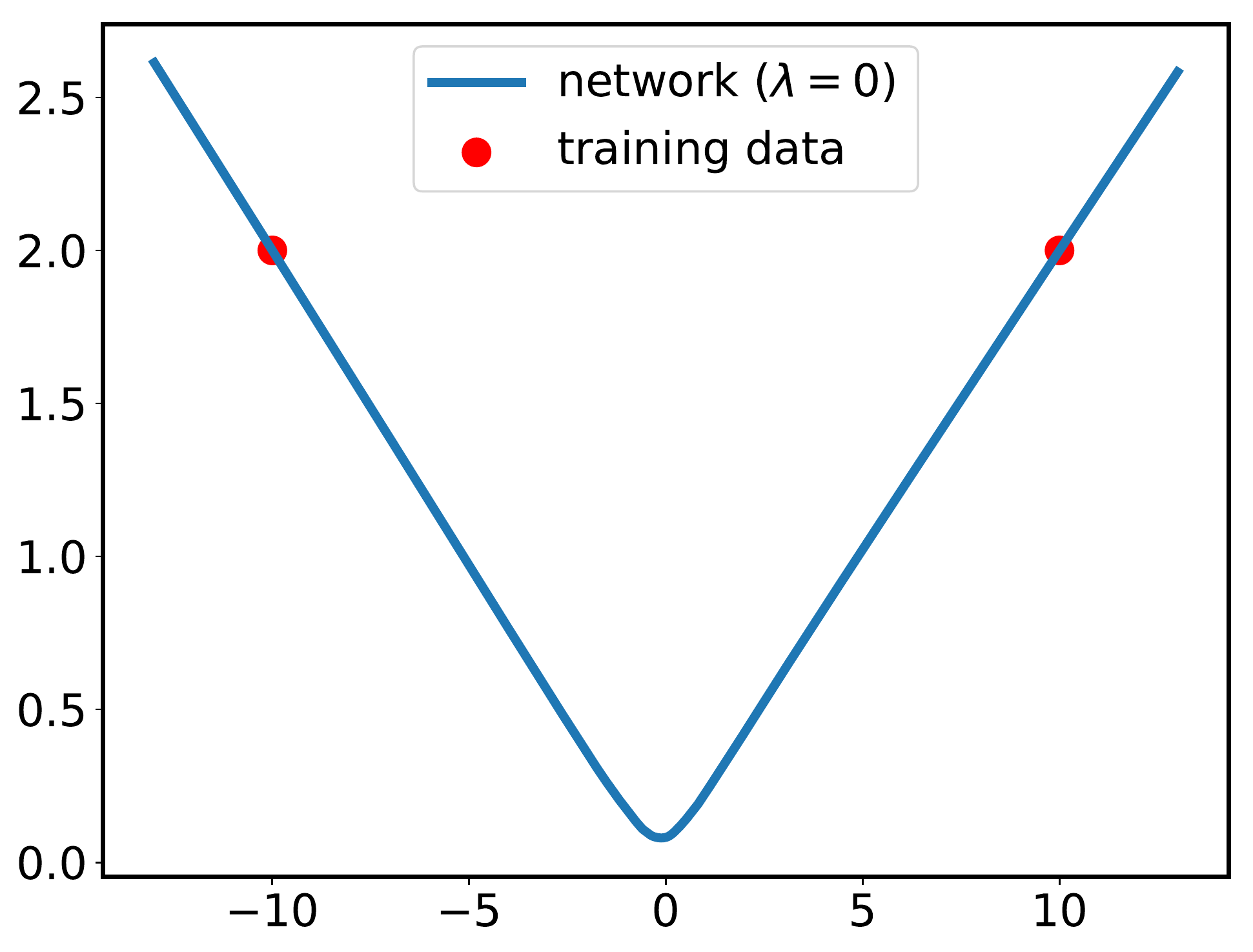}}
\caption{(a) The orange curve represents the function $f^*(x)$ which interpolates the training data (red dots) and exhibits a knot at the point $(0, 0)$; the blue dashed curve is linear in the interval between the training points ($\varepsilon_1 = -0.2$, $\varepsilon_2 = 0.2$). (b) We run noiseless SGD ($\beta=\infty$) with no regularization ($\lambda = 0$) for a two-layer ReLU network with $N=500$ neurons, trained on the dataset \eqref{eq:dataset}. The resulting estimator (in blue) approaches the piecewise linear function $f^*(x)$ with a knot between the two training data points. }\label{fig:counter}
\end{figure}

We start with the case of noiseless SGD training, i.e., $\beta=+\infty$. Here, the free energy has no entropy penalty and it can be expressed as 
\begin{equation}\label{eq:fennop}
    \mathcal{F}_\infty(\rho) = \frac{1}{2}R(\rho) + \frac{\lambda}{2}M(\rho).
\end{equation}
We consider the following dataset which consists of two points:
\begin{equation}\label{eq:dataset}
\mathcal{D} = \{(-\bar{x}, \bar{y}), (\bar{x}, \bar{y})\} = \{(-10,2),(10,2)\}.
\end{equation}
Let $f^{*}(x)$ be the piecewise linear function that interpolates the training data $\{(-\bar{x}, \bar{y}), (\bar{x}, \bar{y})\}$ and passes through the point $(0, 0)$, where it exhibits a knot (see the orange curve in Figure \ref{subfig:plots}).
Note that 
$$
    f^{*}(x)=\int a(wx+b)_{+} \rho^*(a,w,b)\mathrm{d}a\mathrm{d}w\mathrm{d}b,
$$
where
\begin{align}\label{eq:rhostardef}
    &\rho^{*}(a,b,w) =  \frac{1}{2}\left[\delta_{\left(\sqrt{2\frac{\bar{y}}{\bar{x}}}, -\sqrt{2\frac{\bar{y}}{\bar{x}}}, 0\right)}(a,w,b) + \delta_{\left(\sqrt{2\frac{\bar{y}}{\bar{x}}}, \sqrt{2\frac{\bar{y}}{\bar{x}}}, 0\right)}(a,w,b) \right],
\end{align}
and $\delta_{(a_0,w_0,b_0)}$ denotes the Dirac delta function centered at $(a_0,w_0,b_0)$.
Note that $R(\rho^{*}) = 0$ and $M(\rho^{*}) = \frac{2}{5}$. Thus, the free energy is given by
\begin{equation}\label{eq:freepwl}
    \mathcal{F}_\infty(\rho^{*}) = \frac{1}{2}R(\rho^{*}) + \frac{\lambda}{2}M(\rho^{*}) =  \frac{\lambda}{5}.
\end{equation}

Let $f(x)$ be a linear function on the interval $[-\bar{x},\bar{x}]$ such that $f(-\bar{x}) = \bar{y} + \varepsilon_1$ and $f(\bar{x}) = \bar{y}+\varepsilon_2$ (see the blue dashed line in Figure \ref{subfig:plots}), and let $\rho$ be the corresponding distribution of the parameters, i.e., 
\begin{equation}\label{eq:deffrho}
f(x) = \int a(wx+b)_{+} \rho(a,w,b)\mathrm{d}a\mathrm{d}w\mathrm{d}b. 
\end{equation}
In the rest of this section, we will show that, for all $\lambda\le 1$, 
\begin{equation}\label{eq:fenmin}
    \min_{\varepsilon_1, \varepsilon_2} \mathcal{F}_\infty(\rho)> \mathcal{F}_\infty(\rho^{*}).
\end{equation}
In words, the minimizer of the free energy cannot be a linear function on the interval $[-\bar{x},\bar{x}]$. As $f$ is linear, we have that
$$
f(x) = \frac{\varepsilon_2-\varepsilon_1}{2\bar{x}}(x-\bar{x}) + \bar{y} + \varepsilon_2,
$$
which implies that 
\begin{align}\label{zeroeval}
    f(0) = \bar{y} + \frac{\varepsilon_1+\varepsilon_2}{2} = \int a(b)_{+}\rho(a,b)\mathrm{d}a\mathrm{d}b.
\end{align}

First, we consider the case $f(0) = 0$. From \eqref{zeroeval}, we have that $\varepsilon_1+\varepsilon_2=-2\bar{y}$. Hence,
\begin{equation}\label{eq:fencan}
\mathcal{F}_\infty(\rho) \geq \frac{1}{2}R(\rho) = \frac{1}{4}(\varepsilon_1^2 + \varepsilon_2^2) \geq \frac{1}{8}(\varepsilon_1+\varepsilon_2)^2=\frac{\bar{y}^2}{2} = 2.
\end{equation}
By combining \eqref{eq:fencan} and \eqref{eq:freepwl}, we conclude that \eqref{eq:fenmin} holds for all $\lambda \leq 1$ (under the additional restriction $f(0) = 0$).

Next, we consider the case $f(0) \neq 0$. By using (\ref{zeroeval}) and applying Cauchy-Schwarz inequality, we have that 
$$
    |f(0)| = \left|\int a(b)_{+}\rho(a,b)\mathrm{d}a\mathrm{d}b\right| = \left|\mathbb{E} [a(b)_{+}]\right| \leq \sqrt{\mathbb{E} [a^2]\, \mathbb{E} [(b)^2_{+}]} \quad\Longrightarrow \quad\mathbb{E} [a^2] \geq \frac{(f(0))^2}{ \mathbb{E} [(b)^2_{+}]}.
$$
With this in mind, we can lower bound the regularization term as
$$
M(\rho) \geq \mathbb{E} [a^2] + \mathbb{E} [b^2] \geq \frac{(f(0))^2}{ \mathbb{E} [(b)^2_{+}]} + \mathbb{E} [b^2] \geq \frac{(f(0))^2}{ \mathbb{E} [(b)^2_{+}]} + \mathbb{E} [(b)_{+}^2] \geq 2 |f(0)| = 2 \left|\bar{y} + \frac{\varepsilon_1+\varepsilon_2}{2}\right|,
$$
where the last inequality follows from the fact that $g(t)=(f(0))^2/t+t$ is minimized over $t\ge 0$ by taking $t=|f(0)|$. Therefore, we have that
$$
\mathcal{F}_\infty(\rho) \geq \frac{1}{4}(\varepsilon_1^2 + \varepsilon_2^2) + \lambda \left|\bar{y} + \frac{\varepsilon_1+\varepsilon_2}{2}\right|.
$$
Note that, for a fixed value of the sum $\varepsilon_1+\varepsilon_2$, the quantity $\varepsilon_1^2 + \varepsilon_2^2$ is minimized when $\varepsilon_1 = \varepsilon_2$. Thus, by recalling that $\bar{y}=2$, we have
\begin{equation}\label{eq:tmp1min}
\mathcal{F}_\infty(\rho) \geq \min_{\varepsilon}\left\{ \frac{1}{2}\varepsilon^2 + \lambda |2 + \varepsilon|\right\}.
\end{equation}
One can readily verify that, for any $\lambda \leq 2$, the minimizer is given by $\varepsilon^* = -\lambda$. Thus, 
\begin{equation}\label{eq:argprev}
\mathcal{F}_\infty(\rho) \geq 2\lambda-\frac{\lambda^2}{2}\ge \frac{3\lambda}{2}> \frac{\lambda}{5}=\mathcal{F}_\infty(\rho^*),
\end{equation}
where the first inequality uses \eqref{eq:tmp1min} and that the minimizer is $\varepsilon^* = -\lambda$, and the next two inequalities use that $\lambda\ge 1$.
Merging two cases regarding $f(0)$, we conclude that \eqref{eq:fenmin} holds, as desired.

\subsection{Low Temperature Regime}

We now focus on the case of noisy SGD with temperature $\beta^{-1}$. Here, the free energy can be expressed as
\begin{equation}\label{eq:fennoisy}
    \mathcal{F}_{\beta}(\rho) = \frac{1}{2}R(\rho) + \frac{\lambda}{2} M(\rho) - \beta^{-1}H(\rho).
\end{equation}
We consider the two-point dataset \eqref{eq:dataset} and we recall that $f^{*}(x)$ has a knot inside the training interval. In this section we will show that the following two results hold for all $\lambda\le 1$: 
\begin{itemize}
    \item[\emph{(i)}] There exists a sequence of distributions $\{\rho^*_\beta\}_\beta$ such that, for any $x\in [-\bar{x}, \bar{x}]$, 
    \begin{equation}\label{eq:dist1}
\lim_{\beta\to\infty}        \int a(wx+b)_{+}\rho^{*}_{\beta}(a,w,b)\mathrm{d}a\mathrm{d}w\mathrm{d}b = f^*(x),
    \end{equation}
    and
    \begin{equation}\label{eq:dist2}
     \limsup_{\beta\to\infty}\mathcal{F}_\beta(\rho^*_\beta)\le \frac{\lambda}{5}.   
    \end{equation}
    
    \item[\emph{(ii)}] Let $\rho$ be a distribution such that the function $f(x)$ given by \eqref{eq:deffrho} is linear in the interval $[-\bar{x},\bar{x}]$.
    Pick a sequence of distributions $\{\rho_\beta\}_\beta$ such that $\rho_\beta\rightharpoonup \rho$ and for any $x\in [-\bar{x}, \bar{x}]$, 
    \begin{equation}\label{eq:dist3}
\lim_{\beta\to\infty}        \int a(wx+b)_{+}\rho_{\beta}(a,w,b)\mathrm{d}a\mathrm{d}w\mathrm{d}b = f(x).
    \end{equation}
Then, we have that
    \begin{equation}\label{eq:dist4}
        \liminf_{\beta\to\infty}\mathcal{F}_\beta(\rho_\beta)>\frac{\lambda}{5}.
    \end{equation}
\end{itemize}
Combining these two results gives that, for sufficiently large $\beta$, the minimizer of the free energy \eqref{eq:fennoisy} cannot yield a linear estimator on the interval between the two data points. In Figure \ref{subfig:plots2}, we represent the function obtained by training via SGD a two-layer ReLU network with 500 neurons on the dataset \eqref{eq:dataset}. Clearly, the blue curve approaches the piecewise linear function $f^*(x)$, which contains a knot inside the interval $[-10, 10]$. The plot represented in the Figure corresponds to the case with no regularization ($\lambda=0$), but similar results are obtained for small (but non-zero) regularization.

\paragraph{Proof of \emph{(i)}.} Let $\rho_{\beta}^{*}$ be defined as
$$
\rho_\beta^{*} = \frac{1}{2}\left[\mathcal{N}\left(\left[\sqrt{2\frac{\bar{y}}{\bar{x}}}, -\sqrt{2\frac{\bar{y}}{\bar{x}}}, 0\right], \beta^{-1} I_{3\times3}\right) + \mathcal{N}\left(\left[\sqrt{2\frac{\bar{y}}{\bar{x}}}, \sqrt{2\frac{\bar{y}}{\bar{x}}}, 0\right], \beta^{-1} I_{3\times3}\right) \right],
$$
where $\mathcal{N}(\mu,\Sigma)$ denotes the multivariate Gaussian distribution with mean $\mu$ and covariance $\Sigma$. As $\beta \rightarrow \infty$, we have that $\rho^{*}_{\beta} \rightharpoonup \rho^{*}$, where $\rho^*$ is given by \eqref{eq:rhostardef}. However, weak convergence does not suffice for pointwise convergence of the corresponding estimators, since 
the function $\sigma^*(x)=a(wx+b)_{+}$ is unbounded (in $x$). To solve this issue, we observe that the fourth moment of $\rho^{*}_{\beta}$ is uniformly bounded as $\beta \rightarrow \infty$. Thus, by the de la Vallée Poussin criterion (see e.g. \cite{hu2011note}), we have that the sequence of random variables $\{\|X_\beta\|_2^2\}_\beta$ is uniformly integrable, with $X_\beta\sim \rho^*_\beta$.
Consider a ball $B_r=\{\mathbf{v}\in\mathbb{R}^3: \|\mathbf{v}\|_2 \leq r\}$, for $r > \sqrt{4\bar{y}/\bar{x}}$. Then, we have
\begin{align}\label{RHSpwcounter}
    \begin{split}
    \Bigg|\int_{\mathbb R^3} a(wx+b)_{+}&(\rho^{*}_{\beta}(a,w,b)-\rho^{*}(a,w,b))\mathrm{d}a\mathrm{d}w\mathrm{d}b\Bigg| \\ \ \ \leq  &\Bigg|\int_{B_r} a(wx+b)_{+}(\rho^{*}_{\beta}(a,w,b)-\rho^{*}(a,w,b))\mathrm{d}a\mathrm{d}w\mathrm{d}b\Bigg| \\  &\ \hspace{3.64em}+ \left|\int_{\mathbb{R}^3\setminus B_r} a(wx+b)_{+}\rho^{*}_{\beta}(a,w,b)\mathrm{d}a\mathrm{d}w\mathrm{d}b\right|,
    \end{split}
\end{align}
where we have used that the support of $\rho^*$ lies inside the ball $B_r$. The first term in the RHS of (\ref{RHSpwcounter}) vanishes as $\beta\rightarrow \infty$ by weak convergence, since the function $a(wx+b)_{+}$ is bounded inside $B_r$. For the second term, we have that, for any $x\in [-\bar{x}, \bar{x}]$,
\begin{align*}
  \left|\int_{\mathbb{R}^3\setminus B_r} a(wx+b)_{+}\rho^{*}_{\beta}(a,w,b)\mathrm{d}a\mathrm{d}w\mathrm{d}b\right| &\leq \int_{\mathbb{R}^3\setminus B_r} (|aw| \cdot |x| +|ab|)\rho^{*}_{\beta}(a,w,b)\mathrm{d}a\mathrm{d}w\mathrm{d}b \\
  &\leq C \int_{\mathbb{R}^3\setminus B_r} (a^2+b^2+w^2)\rho^{*}_{\beta}(a,w,b)\mathrm{d}a\mathrm{d}w\mathrm{d}b,  
\end{align*}
where $C>0$ is a constant independent of $(\beta,r)$. Since the sequence $\{\|X_\beta\|_2^2\}_\beta$ is uniformly integrable, we can make the RHS arbitrary small by picking a sufficiently large $r$ (uniformly for all $\beta$). As a result, \eqref{eq:dist1} readily follows.
\begin{figure}[t!]
    \centering
    \subfloat[$\beta^{-1}=0.005$]{\includegraphics[width=.33\columnwidth]{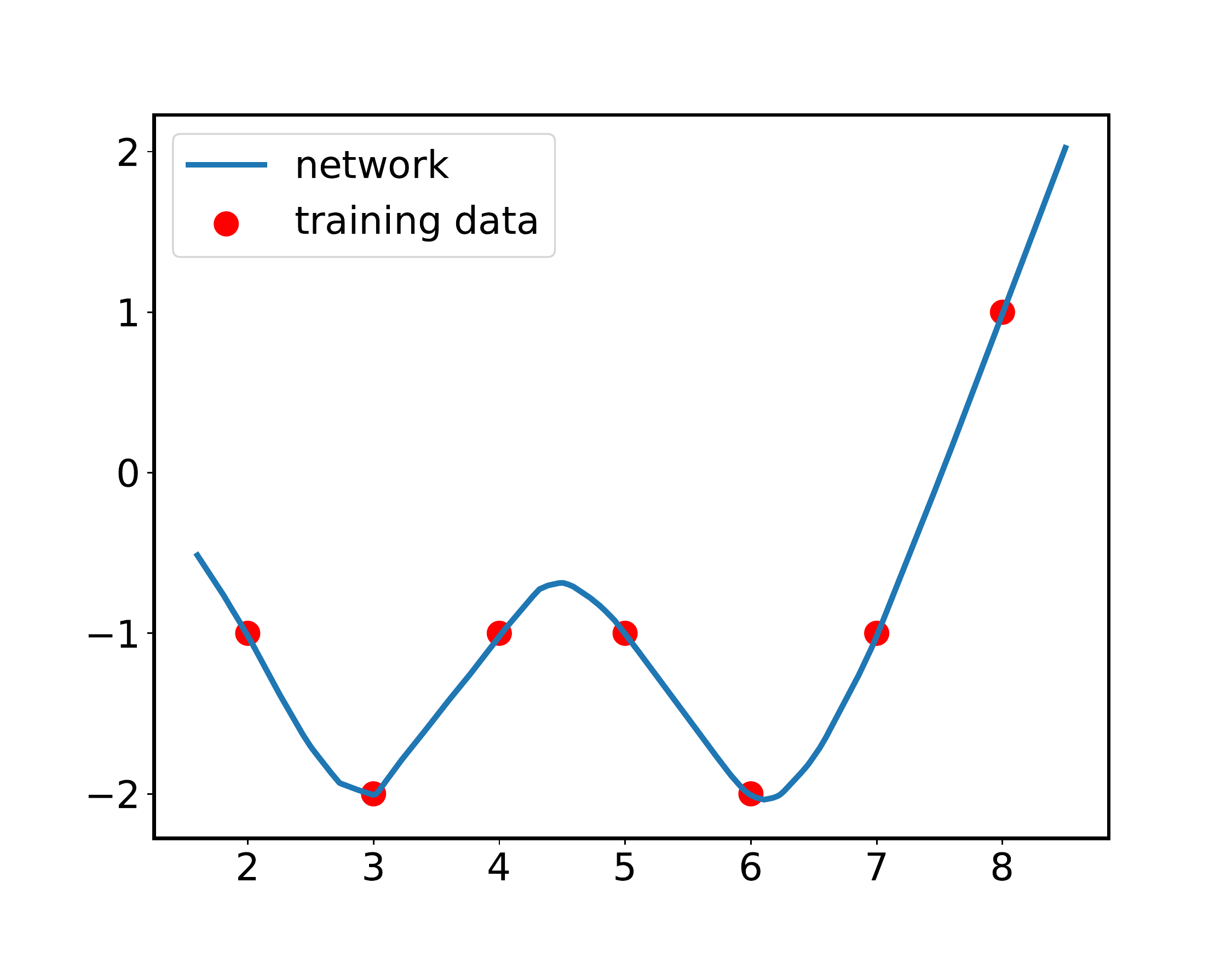}}
    \subfloat[$\beta^{-1}=0.0001$]{\includegraphics[width=.33\columnwidth]{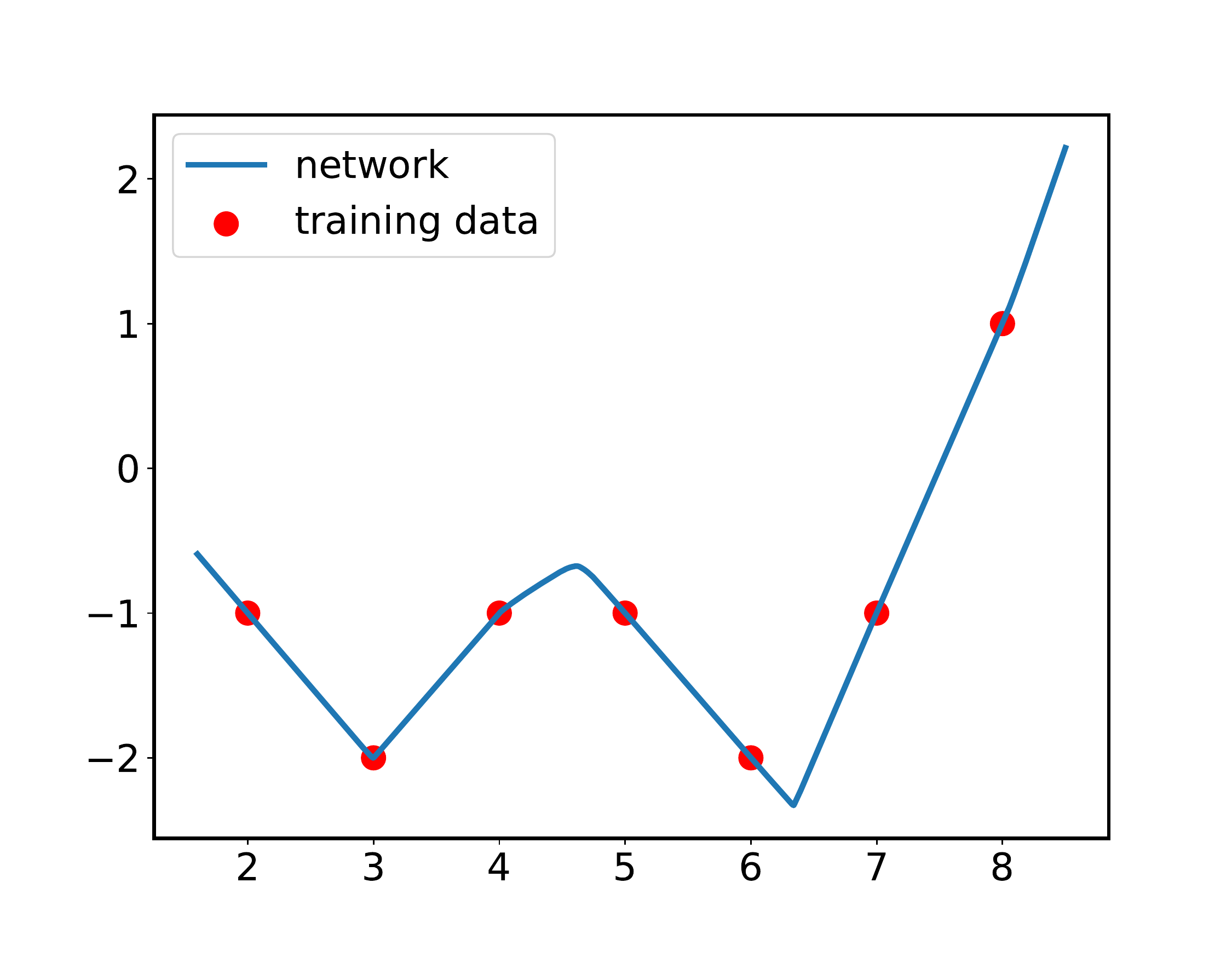}}
    \subfloat[$\beta^{-1}=0$]{\includegraphics[width=.33\columnwidth]{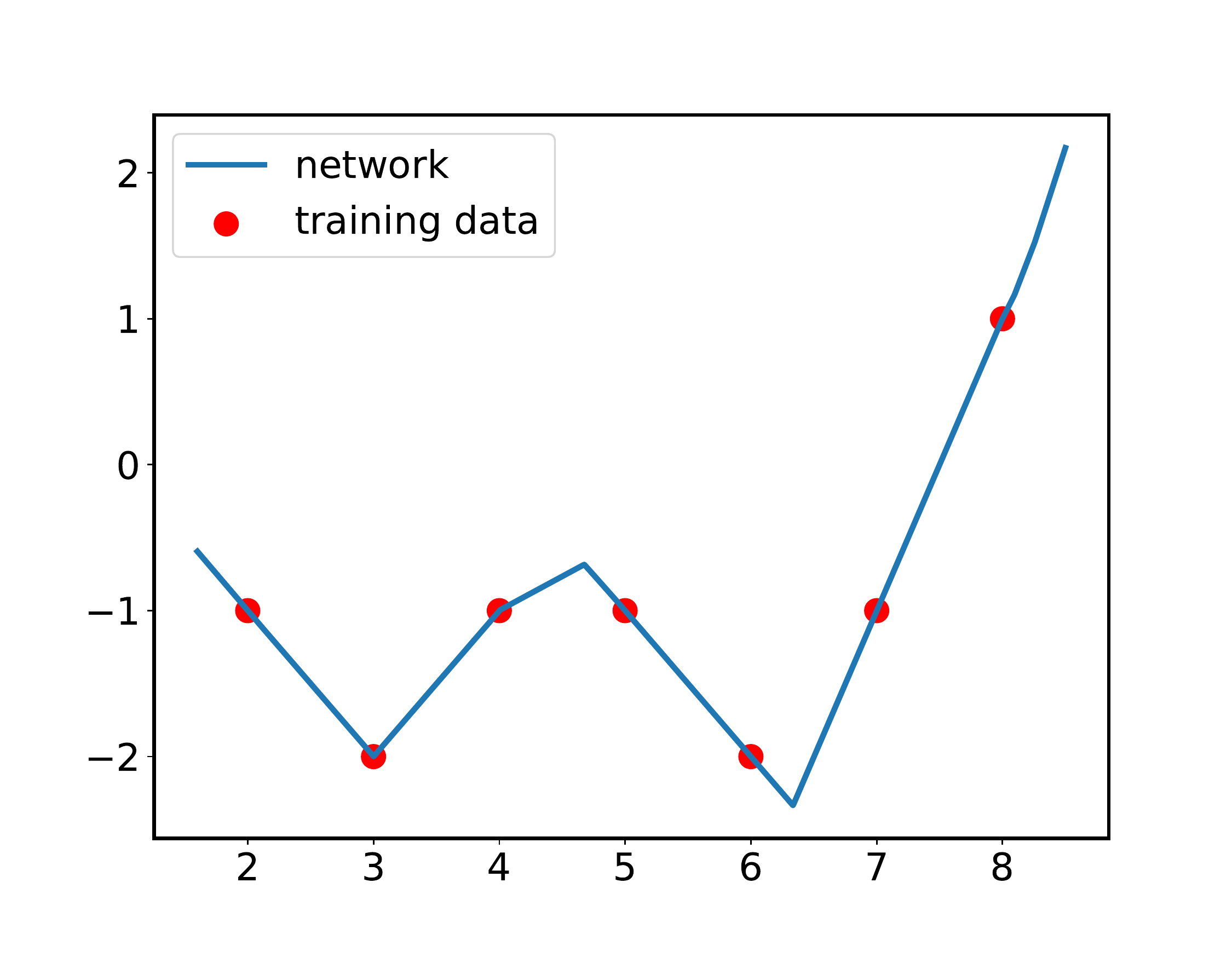}}
\caption{Functions learnt by a two-layer ReLU network with $N=500$ neurons, for different values of the temperature parameter $\beta^{-1}$. The regularization coefficient $\lambda$ is set to zero.}\label{fig:6}
\end{figure}
Note that \eqref{eq:dist1} immediately implies that, as $\beta \rightarrow \infty$, $R(\rho^{*}_{\beta}) \rightarrow R(\rho^{*}) = 0$.
Furthermore, with similar arguments we obtain that, as $\beta\to\infty$, $M(\rho^{*}_{\beta})\to M(\rho^{*})$. By convexity of the differential entropy, we have that $H(\frac{1}{2} \rho_1 + \frac{1}{2} \rho_2) \geq \frac{1}{2}H(\rho_1) + \frac{1}{2}H(\rho_2)$. Hence, $
H(\rho_\beta^{*}) \geq  C \log (2\pi e/\beta)
$, where $C>0$ is independent of $\beta$. By combining these bounds on $R(\rho^{*}_{\beta})$, $M(\rho^{*}_{\beta})$ and $H(\rho^{*}_{\beta})$, we conclude that 
$$
\limsup_{\beta\to\infty}\mathcal{F}_{\beta}(\rho^{*}_{\beta}) \leq \mathcal{F}_{\infty}(\rho^{*}), 
$$
which, combined with \eqref{eq:freepwl}, completes the proof of \eqref{eq:dist2}. 

\paragraph{Proof of \emph{(ii)}.} From \eqref{eq:dist3}, we obtain that $\lim_{\beta\to\infty} R(\rho_{\beta}) = R(\rho)$. As the second moment is lower-semicontinuous and bounded from below, we have that  $\liminf_{\beta\to\infty} M(\rho_\beta)\ge M(\rho)$. Furthermore, 
Lemma 10.2 in \cite{mei2018mean} implies that
$$
\mathcal{F}_{\beta}(\rho_{\beta}) \geq \frac{1}{2}R(\rho_{\beta}) + \frac{\lambda}{4}M(\rho_{\beta}) - \beta^{-1} (1 + 3\log8\pi) + \beta^{-1}\log(\beta\lambda). 
$$
By combining these bounds, we have that 
\begin{equation}\label{eq:combobds}
    \liminf_{\beta\to\infty}\mathcal{F}_{\beta}(\rho_{\beta})\ge \frac{1}{2}R(\rho) + \frac{\lambda}{4}M(\rho).
\end{equation}
By replicating the argument leading to \eqref{eq:argprev} (but now with regularization coefficient $\lambda/2$ instead of $\lambda$), we obtain that the RHS of \eqref{eq:combobds} can be lower bounded as  
\begin{equation}\label{eq:combobds2}
\frac{1}{2}R(\rho) + \frac{\lambda}{4}M(\rho) \geq \lambda-\frac{\lambda^2}{8}\ge \frac{7\lambda}{8}> \frac{\lambda}{5},
\end{equation}
for all $\lambda\le 1$. Then, the desired result follows from \eqref{eq:combobds} and \eqref{eq:combobds2}.

\section{Numerical Simulations}\label{section:numsim}

We consider training the two-layer neural network \eqref{eq:NN} with $N$ neurons and ReLU activation functions, i.e., $\sigma^*(x, \btheta)=a(wx+b)_+$, with $\btheta=(a, w, b)$. We run the SGD iteration \eqref{eq:SGD} (no momentum or weight decay, batch size equal to $1$), and we plot the resulting predictor once the algorithm has converged. The results for two different unidimensional datasets are reported in Figures \ref{fig:6} and \ref{fig:7}. In these experiments, we set $N=500$ and we remark that the plots for wider networks ($N\in\{1000,2000,5000\}$) look identical.
We also point out that the shape of the predictor does not change for different runs of the SGD algorithm (with different initializations, and order of the training samples). \rev{This is in agreement with the mean-field predictions when $\beta<\infty$, $\lambda>0$ and the variance of the initialization does not depend on $N$.}
The same setup is employed to obtain the numerical results of Figure \ref{fig:intro} and \ref{subfig:plots2}, discussed in Section \ref{sec:intro} and \ref{section:knotsarethere}, respectively.

In Figure \ref{fig:6}, we plot the shape of the function learnt by the network for different values of the temperature parameter $\beta^{-1}$. The learning rate is $s_k = 1$, the total number of training epochs required for SGD to converge is roughly $5 \times 10^4$, and no $\ell_2$ regularization is enforced ($\lambda=0$).
As predicted by our theoretical findings, the predictor approaches a piecewise linear function whose number of tangent changes (or knots) is proportional to the number of training samples (and not to the width of the network): if $\beta^{-1} = 0.005$, the predictor is still rather smooth; if $\beta^{-1}=10^{-4}$, the predictor sharpens, except for a smoother tangent change in the interval $[4,5]$; and finally if $\beta=0$, the predictor is piecewise linear. Let us highlight that the knots sometimes do not coincide with the training data points, as suggested by the results of Section \ref{section:mainres} and demonstrated in the example of Section \ref{section:knotsarethere}.

\begin{figure}[t!]
    \centering
    \subfloat[$\beta^{-1}=0.01,\ \lambda=0.003$]{\includegraphics[width=.5\columnwidth]{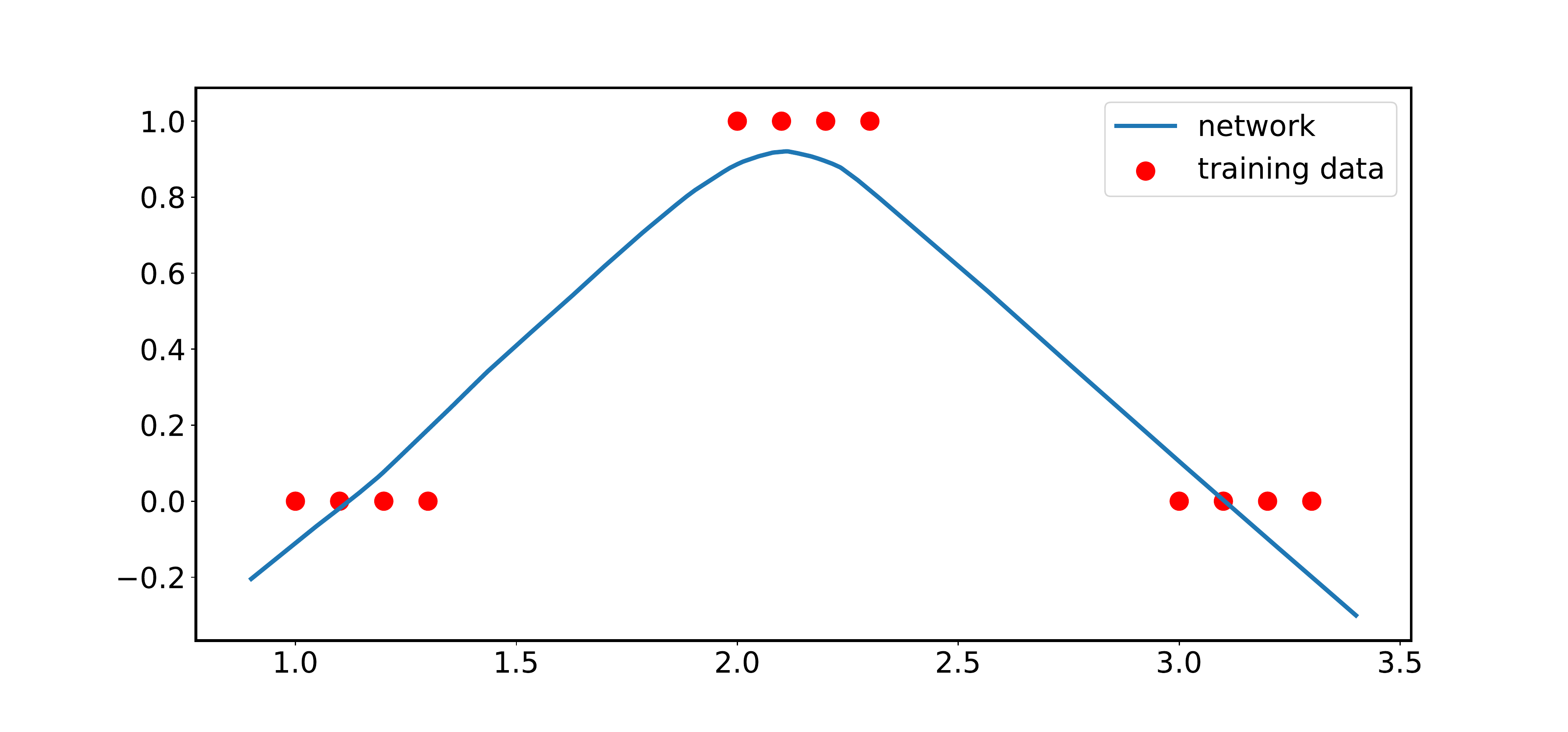}}
    \subfloat[$\beta^{-1}=0.001,\ \lambda=0$]{\includegraphics[width=.5\columnwidth]{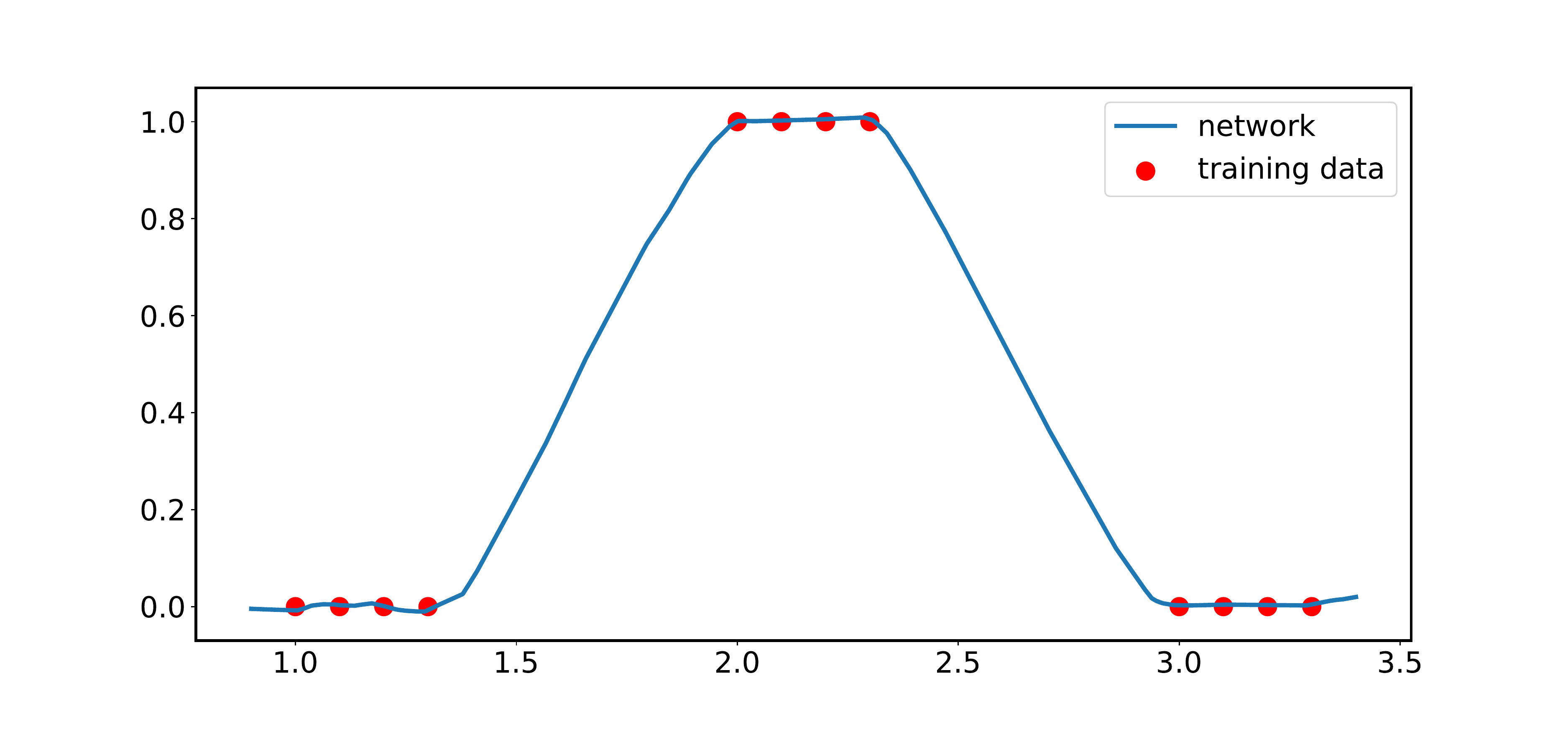}}
    
    \bigskip
    \subfloat[$\beta^{-1}=0,\ \lambda=0.003$]{\includegraphics[width=.5\columnwidth]{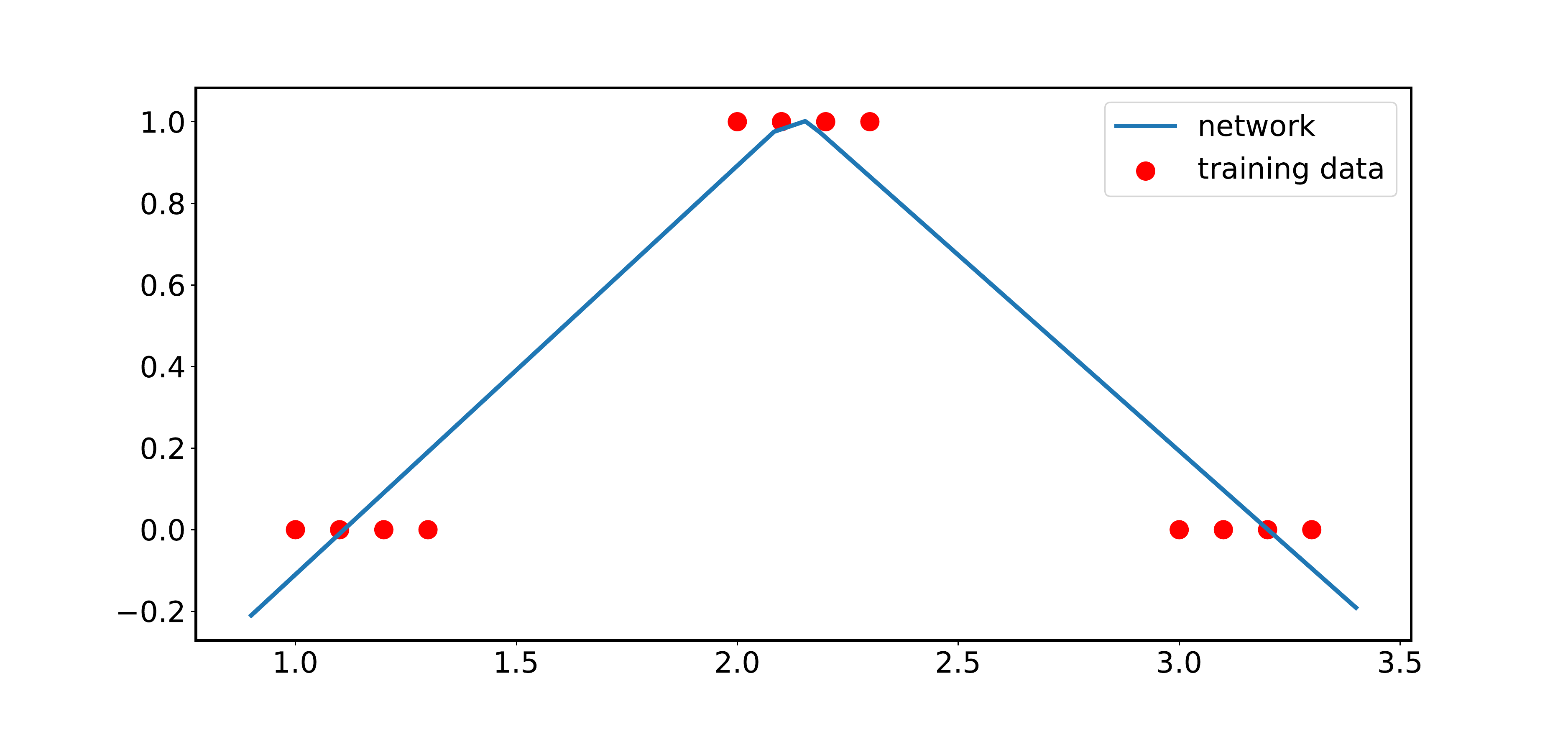}}
    \subfloat[$\beta^{-1}=0,\ \lambda=0$]{\includegraphics[width=.5\columnwidth]{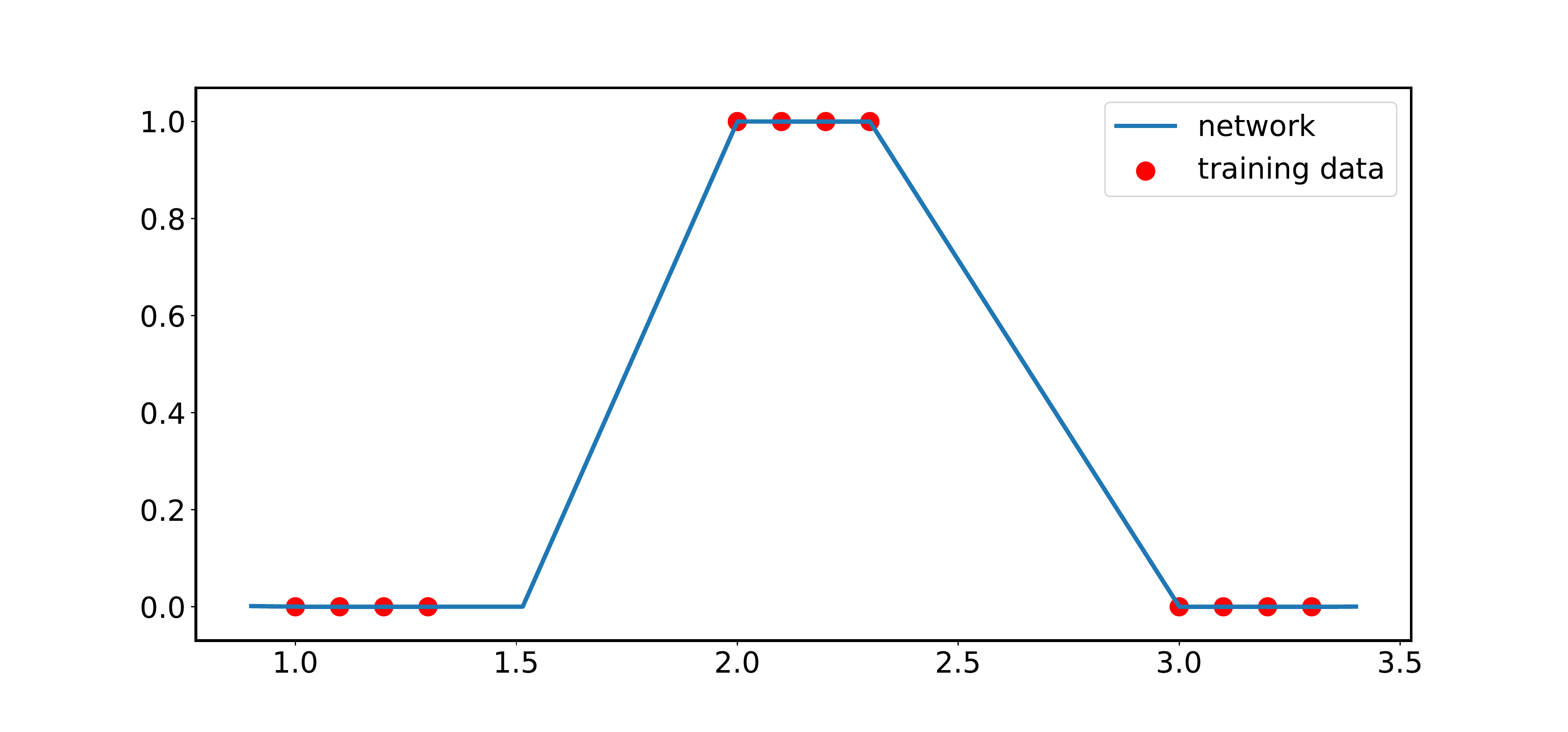}}
\caption{Functions learnt by a two-layer ReLU network with $N=500$ neurons, for different values of the temperature parameter $\beta^{-1}$ and the regularization coefficient $\lambda$.}\label{fig:7}
\vspace{-2mm}
\end{figure}

In Figure \ref{fig:7}, we consider another dataset and plot the neural network predictor for four different pairs of $(\beta^{-1}, \lambda)$. By comparing (a) with (b) and with the bottom plots (c)-(d), it is clear that the solution becomes increasingly piecewise linear as the noise decreases. Furthermore, the effect of regularization can be noticed by comparing plots (a)-(c) on the left with plots (b)-(d) on the right: adding an $\ell_2$ penalty implies that the network does not fit the data and therefore the location of the knots changes.

\section{Comparison with Related Work}\label{sec:discussion}

The line of works \cite{savarese2019infinite,ergen2020convex,ongie2019function,parhi2020banach} studies the properties of the minimizers of certain optimization objectives, and therefore these results are not directly connected to the dynamics of gradient descent algorithms. On the contrary, the goal of this paper is to understand the implicit bias due to gradient descent, namely, to characterize the structure of the neural network predictor once the algorithm has converged. Another important difference lies in the fact that our $\ell_2$ regularization involves all the parameters, including the bias $b$, while existing work does not regularize the biases of the network. This fact may lead to the qualitatively different behavior unveiled by our study.
Going into detail, \cite{ergen2020convex} show that the network that minimizes a regularized objective implements a linear spline. In contrast, our analysis suggests that the knots (i.e., abrupt changes in the tangent of the predictor) can occur at points different from the training samples. Let us also mention that \cite{savarese2019infinite} and \cite{ongie2019function} give an explicit form of the functional regularizer of the neural network solution, but it is not clear how to characterize the function class to which the solution belongs, e.g., whether the function implemented by the neural network is a cubic or linear spline. 
Furthermore, the upper bound on the number of knot points appearing in \cite{parhi2020banach} depends on the null space of a certain operator, and computing  the dimension of this null space explicitly appears to be difficult.

The work by \cite{williams2019gradient} considers a noiseless setting with no regularization, and it studies the properties of gradient flow on the space of reduced parameters. In particular, 
the initial ReLU neurons depending on three parameters ($a$, $b$ and $w$, in our notation) are mapped to a two-dimensional space, where each neuron 
is defined by its magnitude and angle. Then, it is proven that the Wasserstein gradient flow on this reduced space drives the activation points of the ReLU neurons to the training data. As a consequence,  the solution found by SGD is piecewise linear and the knot points are located at a subset of the training samples. \cite{blanc2020implicit} consider SGD with label noise and no regularization, and show that, once the squared loss is close to zero, the algorithm minimizes an auxiliary quantity, i.e., the sum of the squared norms of the gradients evaluated at each training point. By instantiating this result in the case of a two-layer ReLU network with a skip connection, the authors show that the solution found by SGD is piecewise linear with the minimum amount of knots required to fit the data.

While our result shares some similarities with \cite{williams2019gradient} and \cite{blanc2020implicit}, let us highlight some crucial differences. First, we note that
\cite{blanc2020implicit} consider a two-layer network with a skip connection which fits the training data perfectly. In contrast, our 
two-layer model is standard (no skip connections) and the analysis does not require a perfect fit of the data, as we allow for non-vanishing $\ell_2$ regularization. Furthermore, even when the regularization term is vanishing, our characterization does not lead to the minimum number of knots required to fit the data (as in \cite{blanc2020implicit}), and the knots are not necessarily located at the training points (as in \cite{williams2019gradient}). In fact, our theoretical results suggest the presence of additional knot points, a feature that is confirmed in numerical simulations. The novel behavior that we unveil appears to be due to the differences in the setting and to the addition of (a possibly vanishing) $\ell_2$ regularization term in the optimization. Concerning the proof techniques, the work by \cite{blanc2020implicit} exploits an Ornstein-Uhlenbeck like analysis, while this work tackles the increasingly popular mean-field regime. Our key technical contribution is to analyze the Gibbs minimizer of a certain free energy, while \cite{williams2019gradient} consider the gradient flow on reduced parameters and connect it to the flow on the full parameters via a specific type of initialization. Our analysis directly establishes a result for the full parameters, and it requires mild technical assumptions on the initialization. Finally, let us point out that it is an open problem to extend the approach of \cite{williams2019gradient} to a regularized objective, because of the non-injectivity of the mapping to the canonical parameters.

\section{Concluding Remarks}\label{conclusion}

We develop a new technique to characterize the implicit bias of gradient descent methods used to train overparameterized neural networks. In particular, we consider training a wide two-layer ReLU network via SGD for a univariate regression task and, by taking a mean field view, we show that the predictor obtained at convergence has a simple piecewise linear form. Our results hold in the regime of vanishingly small noise added to the SGD gradients, and handle both constant and vanishing $\ell_2$ regularization. The analysis leads to an exact characterization of the number and location of the tangent changes (or knots) in the predictor: 
on each interval between consecutive training inputs, the number of knots is at most three. To obtain the desired result, we relate the distribution of the weights of the network once SGD has converged to the minimizer of a certain free energy. Then, we prove that the curvature of the predictor resulting from this minimizer vanishes everywhere except in a \emph{cluster set}, which concentrates on at most three points per prediction interval. This novel strategy opens the way to several interesting directions. We discuss them below.

We focus on ReLU networks. However, only the following two properties of the activation appear to be crucial for the analysis: \emph{(i)} its second derivative behaves like a Dirac delta, and \emph{(ii)} its growth is at most linear. In fact, the first property reduces the computation of the curvature to an integral over a lower-dimensional subspace; and the second property leads to a uniform bound on the second moment of the network parameters. Hence, our approach may be extendable to a more general class of piecewise linear activations, although this would come at the cost of a more intricate structure for the cluster set containing the location of the tangent changes.

We focus on univariate regression. The natural ordering on one-dimensional features allows for a convenient characterization of the activation regions that correspond to each input conditioned on the sign of $w$. For larger input dimension, such a characterization appears to be cumbersome, as the structure of these regions is induced by the intersection of hyperplanes. Furthermore, in the setting considered in this work, the cluster set is the union of intervals where certain second-degree polynomials are non-positive. For multivariate regression, we expect the cluster set to be connected to the non-positive set of quadratic forms. Hence, the structure of the cluster set may be highly non-linear, and its concentration can occur on subspaces which are hard to define explicitly. 

\rev{We provide an upper bound on the number of tangent changes of the predictor. The numerical simulations of Section \ref{section:knotsarethere} suggest that one and two knots between consecutive training inputs can occur. Showing whether our theoretical bound of three knots is tight by providing an explicit example, or by proving a tighter bound of two, is an open question  for possible future work. We also remark that, given the errors $R_i$ of the neural network estimator at the data points, one can deduce the location of the knot points. Such implicit characterization 
is similar in spirit to the attractive/repulsive condition on the training points of \cite{williams2019gradient}.}

\rev{In conclusion, in this work we demonstrate how to exploit the Gibbs form of the minimizer in order to accurately characterize a functional property of the predictor learnt by the neural network using limiting arguments of the training process. The general spirit of this technique could potentially be informative in additional ways. For instance, utilizing the properties of the Gibbs distribution reached at convergence may be of additional interest for future study. We conjecture that this could yield insight into the stability of the predictor with respect to perturbations in the training data at \emph{finite} temperature $\beta$.}

\section*{Acknowledgements}

We would like to thank Mert Pilanci for several exploratory discussions in the early stage of the project, Jan Maas for clarifications about \cite{jordan1998variational}, and Max Zimmer for suggestive numerical experiments. A. Shevchenko and M. Mondelli are partially supported by the 2019 Lopez-Loreta Prize. V. Kungurtsev acknowledges support to the OP VVV project
CZ.02.1.01/0.0/0.0/16\_019/0000765 Research Center for Informatics.

\bibliographystyle{amsalpha}
\bibliography{refs}

\appendix
\section{Technical Results}\label{techlemmas}

In this appendix, we prove a few technical results which are used in the arguments of Section \ref{selected_proofs}.
More specifically, in Section \ref{appendix:convmin} we show that, as $\tau\to\infty$, the minimizer $\rho_{\tau, m}^*(\btheta)$ of the free energy $\mathcal F^{\tau, m}$ converges pointwise in $\btheta$ to the minimizer $\rho_{m}^*(\btheta)$ of the free energy $\mathcal F^{m}$. This pointwise convergence is needed to establish the result of Lemma \ref{convdelta}. In Section \ref{appendix:risk}, we derive upper bounds on the risk of the minimizer (used in Lemma \ref{worstcasebound}) and on its second moment (which implies that the sequence of predictors is equi-Lipschitz), and we also prove the lower bound on the partition function in Lemma \ref{finiteZ}. Finally, in Section \ref{appendix:lbpoly} we give the proof of Lemma \ref{welldefquadgeneric}, which lower bounds the growth of the polynomials $f^j$ and $f_j$.

\subsection{Convergence of Minimizers}\label{appendix:convmin}

\begin{lemma}[Convergence of densities]\label{DP} Let $\{\rho_n\}_n$ be a sequence of densities in $\mathcal{K}$ with uniformly bounded truncated entropy, that is
$$
\int \max\left\{\rho_n(\btheta)\log \rho_n(\btheta), 0\right\} \mathrm{d}\btheta \leq C, \quad \forall n,
$$
for some $C>0$ that is independent of $n$, and uniformly bounded second moment, i.e., $M(\rho_n) \leq C$ for all $n$. Then, there exists a subsequence $\{\rho_{n'}\}_{n'}$ of $\{\rho_n\}_n$ and $\rho \in \mathcal{K}$ such that
$
\rho_{n'} \rightharpoonup \rho 
$
and
$$
C \ge  \liminf_{n'\rightarrow\infty} M(\rho_{n'}) \geq M(\rho) \geq 0.
$$
\end{lemma}
\begin{proof}[Proof of Lemma \ref{DP}] Since $z \mapsto \max\{z\log z,0\}$, $z\in[0,+\infty)$, has super-linear growth, this result in conjunction with the de la Vallée Poussin criterion (see for instance \cite{hu2011note})
guarantees that the sequence of densities $\{\rho_n\}_{n}$ is uniformly integrable. By Dunford-Pettis Theorem (for $\sigma$-finite measure spaces, see for instance \cite{laurenccot2015weak}), relative weak compactness in $L_1$ is equivalent to uniform integrability. Hence, there exists a density $\rho$ and a subsequence $\{\rho_{n'}\}_{n'}$ of $\{\rho_n\}_n$ such that $\rho_{n'} \rightharpoonup \rho$. 

As $M(\cdot)$ is lower-semicontinuous with respect to the topology of weak convergence in $L_1$ and bounded from below, we have that
$
\liminf_{n'\rightarrow\infty} M(\rho_{n'}) \geq M(\rho).
$
Furthermore, as $M(\rho_{n}) \le C$, we get that $M(\rho) \le C$ and, thus, $\rho\in\mathcal{K}$.
\end{proof}

\begin{lemma}[Uniformly bounded $M(\rho^{*}_{\tau,m})$ and limit of $\rho^{*}_{\tau,m}$]\label{limitrho_mtau} Assume that condition \textbf{A1} holds. Consider the sequence of minimizing Gibbs distributions $\{\rho^{*}_{\tau,m}\}_{\tau}$. The following results hold:
\begin{enumerate}
    \item $M(\rho^{*}_{\tau,m})$ is uniformly bounded in $(\tau,m)$. Moreover, if $\beta\lambda > 1$,
    $$
    M(\rho^{*}_m),\ M(\rho^{*}_{\tau,m}) \leq \frac{C_3}{\lambda},\quad \forall \tau \in (0,+\infty),
    $$
    where $C_3>0$ is independent of $(\tau,m,\beta,\lambda)$.
    \item Given any $m$ consistent with \textbf{A1}, there exists $\rho_m \in \mathcal{K}$ and a subsequence $\{\rho^{*}_{\tau',m}\}_{\tau'}$ (which with an abuse of notation we identify with  $\{\rho^{*}_{\tau,m}\}_{\tau}$) such that $\rho^{*}_{\tau,m} \rightharpoonup \rho_m$ as $\tau\to\infty$.
      \item Given any $m$ consistent with \textbf{A1}, $\lim_{\tau\to\infty}R^{\tau,m}_i(\rho^{*}_{\tau,m}) = R^m_i(\rho_m)$ for all $i\in [M]$, and $\liminf_{\tau\rightarrow \infty} \mathcal{F}^{\tau,m}(\rho^{*}_{\tau,m}) \geq \mathcal{F}^{m}(\rho_{m})$.
\end{enumerate}

\end{lemma}
\begin{proof}[Proof of Lemma \ref{limitrho_mtau}] We provide the proof of the first result for $\rho^{*}_{\tau,m}$. The arguments for $\rho^{*}_m$ are the same after changing the notation from $\rho^{*}_{\tau,m}$ to $\rho^{*}_m$. Let $\rho = \mathcal{N}(0,\mathbb{I}_{3\times3})$. Then, we have that
\begin{equation}\label{eq:RM}
    R^{\tau,m}(\rho) = \frac{1}{M} \sum_{i=1}^M y_i^2,\ M(\rho) = 3,\ H(\rho) = \frac{3}{2}\ln(2\pi e).
    \end{equation}
Note that for this $\rho$,  $R^{\tau,m}(\rho)$, in fact, does not depend on $(\tau,m,\beta,\lambda)$.

From Lemma 10.2 in \cite{mei2018mean}, since $\rho^{*}_{\tau,m}$ is the unique minimizer of the free energy $\mathcal{F}^{\tau,m}$, we have that the following inequalities hold
\begin{equation}\label{1.4.1}
    \mathcal{F}^{\tau,m}(\rho) \geq \mathcal{F}^{\tau,m}(\rho^{*}_{\tau,m}) \geq  R^{\tau,m}(\rho^{*}_{\tau,m})+\lambda / 4 \cdot M(\rho^{*}_{\tau,m})-1 / \beta \cdot[1+3 \cdot \log (8 \pi /(\beta \lambda))].
\end{equation}
Furthermore, by using (\ref{eq:RM}) and the fact that $\beta > C_1$ and $\lambda < C_2$, we obtain
\begin{equation}\label{eq:ubfree}
    \mathcal{F}^{\tau,m}(\rho) \leq K_1 + K_1 \lambda - \beta^{-1} K_1 \leq K_2,
\end{equation}
for some $K_1, K_2>0$ that are independent of $(\tau,m,\beta,\lambda)$. By combining (\ref{eq:ubfree}) and (\ref{1.4.1}) and using that $R^{\tau,m}(\rho^{*}_{m})\ge 0$, we conclude that
$$
\lambda \cdot  M(\rho^{*}_{\tau,m}) \leq K_3 + 1 / \beta \cdot[1+3 \cdot \log (8 \pi /(\beta \lambda))],
$$
where $K_3 > 0$ is independent of $(\tau,m,\beta,\lambda)$. As $\beta\lambda > 1$, the first claim immediately follows.

Since the activation and the labels are uniformly bounded in $\tau$ and $\{i\}_{i\in [M]}$ is finite, $|R_i^{\tau,m}(\rho^{*}_{\tau,m})|$ is uniformly bounded in $(\tau,i)$. Hence, the following lower bound on the partition function 
$Z_{\tau,m}(\beta,\lambda)$ holds
\begin{align}\label{Zmt_bound}
    Z_{\tau,m}(\beta,\lambda) & = \int\exp\left\{-\beta\left[\sum\limits_{i=1}^M R^{\tau,m}_i(\rho_{\tau,m}^{*}) \cdot a^{\tau,m} (w^{\tau,m}x_i+b)_{\tau}^m + \frac{\lambda}{2}\|\bm\theta\|_2^2\right]\right\}\mathrm{d}\btheta\nonumber\\&\geq\nonumber  
    \int\exp\left\{-\beta\left[\sum\limits_{i=1}^M |R^{\tau,m}_i(\rho_{\tau,m}^{*})| \cdot 2m^3 +\frac{\lambda}{2}\|\bm\theta\|_2^2\right]\right\}\mathrm{d}\btheta \\&\geq
    K_4 \int \exp\left\{- \frac{\beta\lambda}{2}\|\bm\theta\|_2^2\right\} \mathrm{d}\btheta =  \frac{K_5}{\sqrt{\beta^3\lambda^3}} \geq K_6,
\end{align}
for some $K_4,K_5,K_6>0$ independent of $\tau$ (but dependent on $(m, \beta, \lambda)$). In the same way, one can upper bound $\rho^{*}_{\tau,m}\cdot Z_{\tau,m}(\beta,\lambda)$ as
\begin{equation}\label{eq:bd1}
\exp\left\{-\beta\left[\sum\limits_{i=1}^M R^{\tau,m}_i(\rho_{\tau,m}^{*}) \cdot a^{\tau,m} (w^{\tau,m}x_i+b)_{\tau}^m + \frac{\lambda}{2}\|\bm\theta\|_2^2\right]\right\} \leq K_{7} \exp\left\{- \frac{\beta\lambda}{2}\|\bm\theta\|_2^2\right\},
\end{equation}
where $K_{7} > 0$ is independent of $\tau$ (but dependent on $(m, \beta, \lambda)$). Notice that we can increase $K_{7}$ to be arbitrarily large and still satisfy~\eqref{eq:bd1}, and in particular, increase it to satisfy $K_{7}/K_6 > 1$. Thus, by combining (\ref{Zmt_bound}) and (\ref{eq:bd1}), we get
 \begin{align*}
      &\int \max\{\rho^{*}_{\tau,m}(\btheta)\ln \rho^{*}_{\tau,m}(\btheta), 0\} \mathrm{d}\btheta \\ \leq &\int \max\left\{\frac{K_{7}}{K_6} \exp\left\{- \frac{\beta\lambda}{2}\|\bm\theta\|_2^2\right\} \cdot \left(\ln \frac{K_{7}}{K_6} - \frac{\beta\lambda}{2}\|\bm\theta\|_2^2\right), 0\right\} \mathrm{d}\btheta \\
      = &\int_{\Omega} \frac{K_{7}}{K_6} \exp\left\{- \frac{\beta\lambda}{2}\|\bm\theta\|_2^2\right\} \cdot \left(\ln \frac{K_{7}}{K_6} - \frac{\beta\lambda}{2}\|\bm\theta\|_2^2\right) \mathrm{d}\btheta \leq \int_{\Omega} \frac{K_{7}}{K_6} \ln \frac{K_{7}}{K_6}\mathrm{d}\btheta,
\end{align*}
where $$\Omega = \left\{\btheta\in\mathbb{R}^3: \|\btheta\|_2^2 \leq \ln \left(\frac{K_{7}}{K_6}\right) \frac{2}{\beta\lambda}\right\}.$$ Since $\mathrm{vol}(\Omega) < K_{8}$ for some $K_{8} \geq 0$ independent of $\tau$, we get that
$$
\int \max\{\rho^{*}_{\tau,m}(\btheta)\ln \rho^{*}_{\tau,m}(\btheta), 0\} \mathrm{d}\btheta \leq K_{8} \cdot \frac{K_{7}}{K_6} \cdot \ln\frac{K_{7}}{K_6},
$$
where the RHS is independent of $\tau$. As $M(\rho^{*}_{\tau,m})$ is uniformly bounded in $\tau$, we can invoke Lemma \ref{DP} to finish the proof of the second statement.

We now prove the third statement. By the triangle inequality, we have that, for all $i \in [M]$,
\begin{align*}
    &\lim_{\tau\rightarrow\infty}\left|\int a^{\tau,m} (w^{\tau,m}x_i + b)^{m}_{\tau} \rho^{*}_{\tau,m}(\btheta)\mathrm{d}\btheta - \int a^m (w^mx_i + b)^{m}_{+} \rho_{m}(\mathrm{d}\btheta)\right| \nonumber\\\leq 
    &{\lim_{\tau\rightarrow\infty}\left|\int a^{\tau,m} (w^{\tau,m}x_i + b)^{m}_{\tau} \rho^{*}_{\tau,m}(\btheta)\mathrm{d}\btheta - \int a^m (w^mx_i + b)^{m}_{+} \rho^{*}_{\tau,m}(\btheta)\mathrm{d}\btheta\right|} \\
    +&\lim_{\tau\rightarrow\infty}\left|\int a^m (w^mx_i + b)^{m}_{+} \rho^{*}_{\tau,m}(\btheta)\mathrm{d}\btheta - \int a^m (w^mx_i + b)^{m}_{+} \rho_{m}(\mathrm{d}\btheta)\right|:= A_1 + A_2.
\end{align*}
By upper bounding $\rho^{*}_{\tau,m}$ as in (\ref{Zmt_bound})-(\ref{eq:bd1}), we have 
$$
A_1 \leq K_{9} \lim_{\tau\rightarrow\infty} \int |a^{\tau,m} (w^{\tau,m}x_i + b)^{m}_{\tau} - a^m (w^mx_i + b)^{m}_{+}|
\exp\left\{- \frac{\beta\lambda}{2}\|\bm\theta\|_2^2\right\}\mathrm{d}\btheta,
$$
where $K_{9} > 0$ is independent of $\tau$. Thus, an application of the Dominated Convergence theorem gives that the term $A_1$ vanishes.
Furthermore, the term $A_2$ vanishes by weak convergence of $\rho^{*}_{\tau,m}$ to $\rho_m$. This proves that, as $\tau\to\infty$,  $y^{\sigma^{*}}_{\rho^{*}_{\tau,m}}(x_i)\rightarrow y^{\sigma^{*}}_{\rho_{m}}(x_i)$ and so $R^{\tau,m}_i(\rho^{*}_{\tau,m}) \rightarrow R^m_i(\rho_m)$.

Note that $-H(\cdot)$ and $M(\cdot)$ are lower-semicontinuous in ${\mathcal{K}}$. Furthermore, $M(\cdot)$ is lower bounded and $-H(\cdot)$ is lower bounded by Lemma 10.1 in \cite{mei2018mean} on the subsequence $\{\rho^{*}_{\tau,m}\}_{\tau}$, as $M(\rho^{*}_{\tau,m})$ is uniformly bounded in $\tau$. Hence, as $\rho^{*}_{\tau,m}$ converges weakly to $\rho_m\in\mathcal{K}$, we conclude that
$$
\liminf_{\tau\rightarrow\infty} -H(\rho^{*}_{\tau,m}) \geq -H(\rho_m),\quad \liminf_{\tau\rightarrow\infty} M(\rho^{*}_{\tau,m}) \geq M(\rho_m),
$$
which, combined with $R^{\tau,m}_i(\rho^{*}_{\tau,m}) \rightarrow R^m_i(\rho_m)$, implies the desired result.
\end{proof}

\begin{lemma}[Pointwise convergence of free-energies]\label{convenergy} Fix some distribution $\rho \in \mathcal{K}$, then we have the following pointwise convergence:
$$
\lim_{\tau\rightarrow\infty} \mathcal{F}^{\tau,m}(\rho) = \mathcal{F}^{m}(\rho).
$$
\end{lemma}
\begin{proof}[Proof of Lemma \ref{convenergy}] 
By construction, we have that $(x)_{\tau}^m$ converges to $(x)_{+}^m$, for all $x\in\mathbb{R}$. It is clear that 
$$|a^{\tau,m}|(w^{\tau,m}x+b)_{\tau}^m\rho(\btheta) \leq 2m^3\rho(\btheta),$$
and the RHS is integrable.
Thus, an application of the Dominated Convergence theorem gives that
$$
\lim_{\tau\rightarrow\infty} R^{\tau,m}(\rho) = R^{m}(\rho).
$$This concludes the proof since $M(\rho)$ and $H(\rho)$ are independent of  $\tau$.
\end{proof}

\begin{lemma}[Pointwise convergence of minimizers]\label{convmin} Assume that condition \textbf{A1} holds and consider any satisfactory $m$. Then, as $\tau\to \infty$, the minimizer $\rho^{*}_{\tau,m}$ of the free energy $\mathcal F^{\tau, m}$ converges pointwise in $\btheta$ to the minimizer $\rho^{*}_{m}$ of the free energy $\mathcal F^{m}$, i.e.,
$$
\lim_{\tau\rightarrow\infty}\rho^{*}_{\tau,m} (\btheta) = \rho^{*}_{m} (\btheta), \qquad \forall\, \btheta\in\mathbb R^3. 
$$
\end{lemma}
\begin{proof}[Proof of Lemma \ref{convmin}] From Lemma \ref{limitrho_mtau}, we have that there exists a subsequence $\{\rho^{*}_{\tau,m} \in {\mathcal{K}}\}$ and $\rho_m \in \mathcal{K}$ such that the following holds
\begin{equation}\label{convengseq}
    \liminf_{\tau\rightarrow\infty} \mathcal{F}^{\tau,m}(\rho^{*}_{\tau,m}) \geq \mathcal{F}^{m}(\rho_m).
\end{equation}
Since $\rho^{*}_{\tau,m} \in {\mathcal{K}}$ minimizes $\mathcal{F}^{\tau,m}$, we have
$$
\mathcal{F}^{\tau,m}(\rho^{*}_{\tau,m}) \le \mathcal{F}^{\tau,m}(\rho^{*}_{m}).
$$
By taking the liminf on both sides, using Lemma \ref{convenergy} and (\ref{convengseq}), we have
$$
\mathcal{F}^{m}(\rho_m) \leq \liminf_{\tau\rightarrow\infty} \mathcal{F}^{\tau,m}(\rho^{*}_{\tau,m})  \leq \liminf_{\tau\rightarrow\infty}\mathcal{F}^{\tau,m}(\rho^{*}_{m}) =  \mathcal{F}^{m}(\rho^{*}_{m}).
$$
Since $\rho^{*}_{m}$ is the unique minimizer of $\mathcal{F}^m$ (see Lemma 10.2 of \cite{mei2018mean}), $\rho^{*}_{m}$ and $\rho_{m}$ coincide almost everywhere, which implies that
$$
R^m_i(\rho_m) = R^m_i(\rho^{*}_m).
$$
Hence, by Lemma \ref{limitrho_mtau}, we have that
$$
\lim_{\tau\rightarrow\infty} R^{\tau,m}_i(\rho^{*}_{\tau,m}) = R^m_i(\rho^{*}_m).
$$
Recall that, by construction, for any parameter $v\in\mathbb R$, the $\tau$-smooth $m$-truncation $v^{\tau, m}$ converges to $v^m$ as $\tau\to\infty$. Furthermore, as $\tau\to\infty$, the smooth $m$-truncation $(\cdot)_\tau^m$ of the softplus activation converges pointwise to the smooth $m$-truncation $(\cdot)_+^m$ of the ReLU activation. Thus,
$$
\lim_{\tau\rightarrow\infty} \Psi_\tau(\btheta) = \lim_{\tau\rightarrow\infty}\sum\limits_{i=1}^M R^{\tau,m}_i(\rho_{\tau,m}^{*}) \cdot a^{\tau,m} (w^{\tau,m}x_i+b)_{\tau}^m = \sum\limits_{i=1}^M R^{m}_i(\rho_{m}^{*}) \cdot a^m (w^mx_i+b)_{+}^m = \Psi(\btheta),
$$
where the convergence is intended to be pointwise in $\btheta$. Note that $\Psi_{\tau}(\btheta)$ is uniformly bounded in $\tau$, hence
$$
\lim_{\tau\rightarrow\infty} \exp\left\{-\beta\Psi_{\tau}(\btheta) - \frac{\beta\lambda}{2}\|\btheta\|^2_2\right\} = \exp\left\{-\beta\Psi(\btheta) - \frac{\beta\lambda}{2}\|\btheta\|^2_2\right\},
$$
which implies that $Z_{\tau, m}\rho_{\tau,m}^{*}(\btheta)$ converges pointwise to  $Z_m\rho_{m}^{*}(\btheta)$. Furthermore, as $\tau\to\infty$, $Z_{\tau, m}$ converges to $Z_{m}$ by Dominated Convergence, which concludes the proof. 
\end{proof}

\subsection{Bounds on Risk of Minimizer, Second Moment and Partition Function}\label{appendix:risk}

\begin{lemma}[Bound on risk of the minimizer]\label{lambdarisk} Assume that condition \textbf{A1} holds. Then, 
$$
R^m(\rho^{*}_m) \leq C \lambda,
$$
where $C > 0$ is a constant independent of $(m,\beta,\lambda)$. In addition, for any $\varepsilon > 0$, there exists $\bar{\tau}(\varepsilon,m,\beta,\lambda)$ such that for any $\tau > \bar{\tau}(\varepsilon,m,\beta,\lambda)$ we have
$$
R^{\tau,m}(\rho^{*}_{\tau,m}) \leq C \lambda + \varepsilon.
$$
\end{lemma}
\begin{proof}[Proof of Lemma \ref{lambdarisk}]
Consider a ``saw-tooth'' function centered at $x_i$ with height $y_i$ and width $\varepsilon > 0$, namely,
$$
\textrm{ST}_{x_i,y_i}(x) := \begin{cases}
0,& x < x_i - \varepsilon \text{ or } x > x_i + \varepsilon, \\
\frac{y_i}{\varepsilon} (x - x_i + \varepsilon), &  x_i - \varepsilon \leq x \leq x_i, \\
\frac{y_i}{\varepsilon} (x_i - x + \varepsilon), & x_i < x \leq x_i + \varepsilon,
\end{cases}
$$
Notice that this function can be implemented by the following $\hat{\rho}_i$:
$$
\hat{\rho}_i = \frac{1}{3} \left(\delta_{\left(\frac{3y_i}{\varepsilon},1,\varepsilon-x_i\right)} + \delta_{\left(-\frac{6y_i}{\varepsilon},1,-x_i\right)} + \delta_{\left(\frac{3y_i}{\varepsilon},1,-\varepsilon-x_i\right)}\right),
$$
in the sense that
$$
\textrm{ST}_{x_i,y_i}(x) =\int a(wx+b)_{+}\hat{\rho}_i(\mathrm{d}\btheta),
$$
where $\delta_\btheta$ stands for a delta distribution centered at the point $\btheta=(a, w, b)\in\mathbb{R}^3$. Let us pick $\varepsilon$ such that $\varepsilon < \min_{i\in[M-1]} \left\{|x_i-x_{i+1}|/{2}\right\}$. This condition on $\varepsilon$ guarantees that
$$
\left\{x\in\mathbb{R}: \int a(wx+b)_{+}\hat{\rho}_i(\mathrm{d}\btheta) \neq 0 \right\} \cap \left\{x\in\mathbb{R}: \int a(wx+b)_{+}\hat{\rho}_j(\mathrm{d}\btheta) \neq 0 \right\} = \emptyset, \ \forall i \neq j,
$$
which ensures that the ``saw-tooth'' functions are not intersecting. Define
\begin{align*}
\hat{\rho}=    \frac{1}{3M}\sum_{i=1}^M \left[\delta_{\left(\frac{3My_i}{\varepsilon},1,\varepsilon-x_i\right)} + \delta_{\left(-\frac{6My_i}{\varepsilon},1,-x_i\right)} + \delta_{\left(\frac{3My_i}{\varepsilon},1,-\varepsilon-x_i\right)}\right].
\end{align*}
Then, one immediately has that, for all $i\in [M]$,
$$
\int a(wx_i+b)_{+}\hat{\rho}(\mathrm{d}\btheta)=y_i.
$$
Furthermore, by taking a sufficiently large $m$, in particular, taking $m>\max_i\{6M|y_i|/\varepsilon\}+ 3|x_M|+ 3|x_1|+ 2$ suffices, we get that, for all $x\in [x_1, x_M]$,
$$
\int a^m(w^mx+b)^m_{+} \hat{\rho}(\mathrm{d}\btheta) = \int a(wx+b)_{+} \hat{\rho}(\mathrm{d}\btheta),
$$
which implies that $R^m(\hat{\rho})=0$.

Let $\mathcal{N}(\mu,\sigma^2)$ denote a Gaussian distribution with mean $\mu \in \mathbb{R}$ and variance $\sigma^2\in\mathbb{R}$, and let $U(\mu,\sigma^2)$ denote the uniform distribution with mean $\mu$ and variance $\sigma^2/12$. Given $(\mu_1, \mu_2, \mu_3)\in\mathbb{R}^3$ and $\sigma^2\in\mathbb{R}$, let $\rho_{((\mu_1, \mu_2,\mu_3),\sigma^2)}$ denote the following product distribution
\begin{align*}
    \rho_{((\mu_1, \mu_2, \mu_3),\sigma^2)} := U(\mu_1,\sigma^2)\times\mathcal{N}(\mu_2,\sigma^2)\times\mathcal{N}(\mu_3,\sigma^2),
\end{align*}
and define
\begin{align}\label{eq:rhotildedef}
    \tilde{\rho} &= \frac{1}{3M}\sum_{i=1}^M \left[\rho_{\left(\left(\frac{3My_i}{\varepsilon},1,\varepsilon-x_i\right),\sigma^2 \right)} + \rho_{\left(\left(-\frac{6My_i}{\varepsilon},1,-x_i\right),\sigma^2 \right)} + \rho_{\left(\left(\frac{3My_i}{\varepsilon},1,-\varepsilon-x_i\right),\sigma^2 \right)}\right].
\end{align}
Note that, for $\sigma^2 < 1$ and $m$ chosen sufficiently large as mentioned previously, 
$$
\int a^m(w^mx+b)^{m}_{+}\tilde{\rho}(\mathrm{d}\btheta)=\int a(w^mx+b)^{m}_{+}\tilde{\rho}(\mathrm{d}\btheta).
$$
Thus, by computing the integral w.r.t. $a$, we have that
\begin{align}
    &\int a^m(w^mx+b)^{m}_{+} \hat{\rho}(\mathrm{d}\btheta) - \int a^m(w^mx+b)^{m}_{+}\tilde{\rho}(\mathrm{d}\btheta) \nonumber\\ &= \sum_{i=1}^M\left[ \frac{y_i}{\varepsilon} \left(\int (w^mx+b)^m_{+} \delta_{\left(1,\varepsilon-x_i\right)}(\mathrm{d}w\,\mathrm{d}b) - \int (w^mx+b)^m_{+}\rho_{\left((1,\varepsilon-x_i), \sigma^2\right)}(\mathrm{d}w\,\mathrm{d}b)\right)\right] \nonumber\\
    &-\sum_{i=1}^M\left[ \frac{2y_i}{\varepsilon} \left(\int (w^mx+b)^m_{+} \delta_{\left(1,-x_i\right)}(\mathrm{d}w\,\mathrm{d}b) - \int (w^mx+b)^m_{+}\rho_{\left((1,-x_i), \sigma^2\right)}(\mathrm{d}w\,\mathrm{d}b)\right)
    \right]\nonumber\\
    &+\sum_{i=1}^M\left[ \frac{y_i}{\varepsilon} \left(\int (w^mx+b)^m_{+} \delta_{\left(1,-\varepsilon-x_i\right)}(\mathrm{d}w\,\mathrm{d}b) - \int (w^mx+b)^m_{+}\rho_{\left((1,-\varepsilon-x_i), \sigma^2\right)}(\mathrm{d}w\,\mathrm{d}b)\right)
    \right],\label{1.8.1}
\end{align}
where, with an abuse of notation, we denote by $\rho_{((\mu_2, \mu_3), \sigma^2)}$ the marginal of $\rho_{((\mu_1, \mu_2, \mu_3), \sigma^2)}$ with respect to the last two components. 
By applying to Kantorovich-Rubinstein theorem (see, for instance, \cite{villani2009optimal}), we have that
\begin{equation}
    K\cdot W_1(p,q) = \sup_{\|f\|_{\textrm{Lip}} \leq K} |\mathbb{E}_{x\sim p}f(x) - \mathbb{E}_{y\sim q}f(y)|,\label{KR}
\end{equation}
for two densities $p$ and $q$, where $W_1$ is the 1-Wasserstein distance and $\|f\|_{\textrm{Lip}} $ denotes the Lipschitz constant of $f$. Notice that $(w^mx+b)^m_{+}$ is Lipschitz in $(w,b)$ with Lipschitz constant upper bounded by $\max(|x|,1)$. Hence, combining (\ref{1.8.1}) and (\ref{KR}), we have that
\begin{equation}
\begin{split}
   & \left(\int a^m(w^mx+b)^{m}_{+} \hat{\rho}(\mathrm{d}\btheta) - \int a^m(w^mx+b)^{m}_{+}\tilde{\rho}(\mathrm{d}\btheta)\right)^2 \\
    &\hspace{4em} \leq K_1 \bigg(\sum_{i=1}^M  W_1(\delta_{\left(1,\varepsilon-x_i\right)},\rho_{\left((1,\varepsilon-x_i), \sigma^2\right)})
    + W_1(\delta_{\left(1,-x_i\right)},\rho_{\left((1,-x_i), \sigma^2\right)})\\
    &\hspace{21em}+W_1(\delta_{\left(1,-\varepsilon-x_i\right)},\rho_{\left((1,-\varepsilon-x_i), \sigma^2\right)})\bigg)^2,\label{1.8.2}
\end{split}
\end{equation}
where $K_1>0$ is a constant independent of $m$.
Recalling the form of the 2-Wasserstein distance between a delta and a Gaussian
 distribution, we have that
\begin{equation}
    \label{eq:W2bd}
W^2_2(\delta_{(w, b)},\rho_{((w, b), \sigma^2)}) \leq K_2 \sigma^2,
\end{equation}
for some constant $K_2>0$. As the $W_1$ distance is upper bounded by the $W_2$ distance (via H\"older's inequality), by combining (\ref{1.8.2}) and (\ref{eq:W2bd}), we conclude that
$$
\left(\int a^m(w^mx+b)^{m}_{+} \hat{\rho}(\mathrm{d}\btheta) - \int a^m(w^mx+b)^{m}_{+}\tilde{\rho}(\mathrm{d}\btheta)\right)^2 \leq K_3 \sigma^2,
$$
where $K_3>0$ is a constant independent of $m$.
Hence, by taking $\sigma^2 = \min(\lambda, 1/2)$, we have
\begin{equation*}
R^{m}(\tilde{\rho}) \leq K_4 \lambda,
\end{equation*}
where $K_4>0$ is a constant independent of $m$. 

Now recall that the differential entropy is a concave function of the distribution. Hence, by using the fact that $\rho_{((\mu_1, \mu_2,\mu_3),\sigma^2)}$ is a product distribution and by explicitly computing the entropy of a Gaussian and a uniform random variable, we conclude that
\begin{equation*}
H(\tilde{\rho}) \geq  K_5 (-1 + \log \lambda),
\end{equation*}
where $K_5>0$ is a constant independent of $m$. As $M(\tilde{\rho})$ is upper bounded by a constant independent of $m$, we conclude that
\begin{equation}\label{eq:Hcandidate}
\mathcal F^m(\tilde{\rho})\le K_6 \lambda +\frac{K_5}{\beta} (1 - \log \lambda),
\end{equation}
with $K_6>0$ independent of $m$.
Hence, since $\rho^{*}_m$ is the minimizer of the free energy, by using the bound from Lemma 10.2 in \cite{mei2018mean}, we get that
\begin{equation}\label{eq:lastbd}
\frac{1}{2}R^{m}(\rho^{*}_m) \leq K_6 \lambda + \frac{K_5}{\beta} \left(1 - \log \lambda\right) + \frac{1}{\beta}\left[1 + 3\log \frac{8\pi}{\beta\lambda}\right].
\end{equation}
Since $\beta > -\frac{1}{\lambda} \log \lambda$ and $\beta\lambda > 1$, (\ref{eq:lastbd}) implies that  
\begin{equation}\label{eq:bdR}
    R^{m}(\rho^{*}_m) \leq K_7\lambda,
\end{equation}
for $K_7>0$ independent of $(m,\beta,\lambda)$. This finishes the proof of the first part of the statement. The second part of the statement follows by combining (\ref{eq:bdR}) with Lemma \ref{convmin}.
\end{proof}

\begin{lemma}[Second moment is uniformly bounded]\label{unifboundM_rhomtau} Assume that condition \textbf{A1} holds. It holds that there exists $\tau(m,\beta,\lambda)$ such that for any $\tau > \tau(m,\beta,\lambda)$ the following upper bound holds:$$M(\rho^{*}_{\tau,m}) \leq C,$$
for some $C>0$ that is independent of $(\tau,m,\beta,\lambda)$.
\end{lemma}

\begin{proof}[Proof of Lemma \ref{unifboundM_rhomtau}] Let $\tilde{\rho}$ be defined as in \eqref{eq:rhotildedef}. Then, by combining \eqref{eq:Hcandidate} with Lemma \ref{convenergy}, we have that, for  $\tau > \tau(m,\beta,\lambda)$,
$$
\mathcal F^{\tau, m}(\tilde{\rho})\le K_1 \lambda+\frac{K_2}{\beta}(1-\log\lambda),
$$
where $K_1, K_2>0$ are independent of $m$.
Hence, by using \eqref{1.4.1} with $\tilde{\rho}$ in place of $\rho$ and by recalling that $R^{\tau, m}(\rho_{\tau, m}^*)\ge 0$ and the existence of constants $C_1$ and $C_2$ such that $\beta > C_1$ and $\lambda < C_2$, the result readily follows. 
\end{proof}

We conclude this part of the appendix by providing the proof of Lemma \ref{finiteZ}.

\begin{proof}[Proof of Lemma \ref{finiteZ}]
Consider the following lower bound
\begin{equation}\label{eq:lbZm}
    \begin{split}
    Z_m(\lambda,\beta) &\geq \int \exp\left\{-\beta\left[\sum\limits_{i=1}^M |R_i^m(\rho_m^*)| \cdot |a^m| (w^mx_i+b)_{+}^m + \frac{\lambda}{2}\|\bm\theta\|_2^2\right]\right\}\mathrm{d}\bm\theta  \\
    &\geq \int \exp\left\{-\beta\left[\sum\limits_{i=1}^M |R_i^m(\rho_m^*)| \cdot |a| (w^mx_i+b)_{+} + \frac{\lambda}{2}\|\bm\theta\|_2^2\right]\right\}\mathrm{d}\bm\theta  \\
    &\geq \int \exp\left\{-\beta\left[\sum\limits_{i=1}^M |x_iR_i^m(\rho_m^*)| \cdot |aw| + \sum\limits_{i=1}^M |R_i^m(\rho_m^*)| \cdot |ab| +  \frac{\lambda}{2}\|\bm\theta\|_2^2\right]\right\}\mathrm{d}\bm\theta.
\end{split}
\end{equation}
Define $A = \sum_{i=1}^M |x_iR_i^m(\rho_m^*)|$ and $B=\sum_{i=1}^M |R_i^m(\rho_m^*)|$. By Lemma \ref{lambdarisk}, $|R_i^m(\rho_m^*)| \leq K_1 \sqrt{\lambda}$, where $K_1>0$ is independent of $(m,\beta,\lambda,i)$. Therefore, $|A|,|B| \leq K_2 \sqrt{\lambda}$ for some $K_2>0$ independent of $(m,\beta,\lambda)$.
Using the inequalities $2|aw| \leq a^2 + w^2$ and $2|ab| \leq a^2 + b^2$, the RHS of (\ref{eq:lbZm}) can be lower bounded by 
$$
\int\limits_{\mathbb{R}^3} \exp\left\{ -\frac{\beta}{2}
\left[ (2K_2\sqrt{\lambda} + \lambda)\cdot a^2 + (K_2\sqrt{\lambda} + \lambda)\cdot w^2 + (K_2\sqrt{\lambda} + \lambda)\cdot b^2
\right]\right\}d\bm\theta.
$$
By explicitly computing the integral above, the desired result immediately follows.
\end{proof}

\subsection{Lower Bound on Polynomials}\label{appendix:lbpoly}

\begin{proof}[Proof of Lemma \ref{welldefquadgeneric}]
We start by rewriting $P_2$ as
\begin{align}\label{eq:decomposefj}
    P_2(x) & = (1-a^2) \cdot (x-x_c)^2 + \left[2(1-a^2)x_c + 2ab\right] \cdot (x-x_c)\nonumber\\
    &\hspace{9.5em}+(1-a^2)x_c^2 + 2ab x_c + 1-b^2 \nonumber\\
    &= \frac{1}{2}P_2''(x_c)\cdot (x-x_c)^2 +P_2'(x_c)\cdot (x-x_c)+P_2(x_c).
\end{align}
By definition of $x_c$, one can immediately verify that $P_2(x_c)\geq 0$. 
Notice that, if $P_2''(x_c)$ is close to $0$, then $a^2$ is close to $1$, which implies that (since $x_c\in I$ and, thus, bounded in absolute value) $|P_2'(x_c)|$ is close to $2|b|$  and $P_2(x_c)$ is close to $\textrm{sign}(a) \cdot 2bx_c + 1-b^2$. Therefore, at least one of the coefficients $P_2(x_c), |P_2'(x_c)|, |P''_2(x_c)|$ is lower bounded by a constant that is independent of $(a,b)$.

Next, we distinguish two cases depending on the sign of $P_2''(x_c)$. First, assume that $P_2''(x_c)\ge 0$. We now show that $P'_2(x_c)\cdot (x-x_c)\geq 0$. 

In case of a degenerate polynomial, i.e., $P''_2(x_c) = 0$, we distinguish two sub-cases: either $P'_2(x_c) > 0$ or $P_2'(x_c) < 0$ holds. (The case corresponding to $P'_2(x_c)=0$ is trivial.)  If $P'_2(x_c) >0$, then by definition of $x_c$ and the fact that $x\in\Omega_{+}$ by assumption, we have that $(x-x_c)>0$. In fact, recalling the definition of $x_r$ in Definition~\ref{def:critical}, as $\Omega_{+}$ has non-zero Lebesgue measure, $x_c$ is either the left extreme of $I$ (i.e., $x_c=\inf\limits_{\tilde x \in I}  \tilde x$) or $x_c=x_r\in I$ (i.e., in the interior) and, hence, $\Omega_{+} = (x_r,I_r]$ with $x_r < I_r$. This gives that $P_2'(x_c)(x-x_c)\geq 0$. The case $P_2'(x_c) < 0$ follows from similar arguments.

Now assume that $P_2''(x_c) > 0$ and let $x_{\textrm{min}}$ be the minimizer of $P_2$ on the interval $I$. If $x \geq x_{\textrm{min}}$ then, by definition of a critical point, $x_c \geq x_{\textrm{min}}$ which means that $x_c$ is located on the \textit{right} branch of the parabola and, hence, $P_2'(x_c) \geq 0$. Furthermore, $x$ belongs to the interval $[x_c, I_r]$ by definition of $x_c$. These facts imply that $P_2'(x_c)\cdot (x-x_c) \geq 0$. The case $x < x_{\textrm{min}}$ is treated in a similar fashion.

As it was shown, at least one of the coefficients $P_2(x_c), |P_2'(x_c)|, P''_2(x_c)$ is lower bounded by a constant that is independent of $(a,b)$, and
$P'_2(x_c)(x-x_c)\geq 0$, hence, choosing
$$
\alpha_2 = P''_2(x_c), \quad \alpha_1 = |P'_2(x_c)|, \quad \alpha_0 = P_2(x_c),
$$
concludes the proof for the case of non-negative curvature.

Assume now that $P_2''(x_c) < 0$. 
As $|\Omega_{+}|$ is lower bounded by a strictly positive constant, we can pick $\tilde{x}\in \Omega_{+}$ such that $|\tilde{x}-x_c| = C$, for some $C>0$ which is independent of $(a,b)$. As $\tilde{x}\in \Omega_{+}$, we have that $P_2(\tilde{x}) \geq 0$. Furthermore, by rewriting $P_2(\tilde{x})$ as in (\ref{eq:decomposefj}), we obtain that
$$
\frac{1}{2}P_2''(x_c)\cdot (\tilde{x}-x_c)^2 +P_2'(x_c)\cdot (\tilde{x}-x_c)+P_2(x_c)\ge 0,
$$
which implies that
\begin{equation}\label{eq:ineq}
|P_2'(x_c)||\tilde{x}-x_c| + P_2(x_c)\ge - \frac{1}{2}P_2''(x_c)(\tilde{x}-x_c)^2.
\end{equation}
As $|\tilde{x}-x_c| = C$, (\ref{eq:ineq}) is equivalent to 
\begin{equation}\label{eq:ineq2}
|P_2'(x_c)|\cdot C + P_2(x_c)\ge -\frac{1}{2} P_2''(x_c) \cdot C^2.
\end{equation}
Now, if both $|P_2'(x_c)|$ and $P_2(x_c)$ are close to $0$, then (\ref{eq:ineq2}) immediately implies that $P_2''(x_c)$ is also close to $0$. However, following the argument above, it is not possible that $-P_2''(x_c)$, $|P_2'(x_c)|$ and $P_2(x_c)$ are simultaneously close to $0$. This proves that $\max(|P_2'(x_c)|,P_2(x_c))$ is lower bounded by a constant that is independent of $(a,b)$.

Let $x_{\textrm{max}}$ be the maximizer of $P_2$ and, without loss of generality, assume that $x_c<x_{\textrm{max}}$ (the case $x_c\ge x_{\textrm{max}}$ is handled in a similar way). Note that, by definition of $x_c$, the point $x$ lies in the interval $[x_c, x_{\textrm{max}}]$. To show this, let us assume the contrary, i.e., $x > x_{\max}$ (the case $x<x_c<x_{\textrm{max}}$ is ruled out by the assumption that $x\in\Omega_+$). Then, the root of $P_2$ which is the closest in Euclidean distance to $x$ is located to the right of $x_{\textrm{max}}$, hence $x_c<x_{\textrm{max}}$ cannot be a critical point for $x$, which leads to a contradiction. This proves that $x$ lies in the interval $[x_c, x_{\textrm{max}}]$ and in particular, $x\leq x_{\textrm{max}}$. Furthermore, by concavity, the parabola $P_2(\tilde{x})$ is lower bounded by the line that connects $(x_c,P_2(x_c))$ and $(x_{\textrm{max}}, P_2(x_{\textrm{max}}))$ for  $\tilde{x} \in [x_c,x_{\textrm{max}}]$. By the focal property of the parabola, this line has angular coefficient $|P_2'(x_c)|/2$. Therefore,
$$
P_2(\tilde{x})\ge (\tilde{x}-x_c)\cdot|P_2'(x_c)|/2 + P_2(x_c), \quad \tilde{x}\in[x_c,x_{\textrm{max}}].
$$
Picking $\tilde{x} = x$ and
$$
\alpha_2 = 0,\quad \alpha_1 = |P_2'(x_c)|/2, \quad \alpha_0 = P_2(x_c),
$$
gives the desired result in the case $P_2''(x_c) < 0$ and concludes the proof.
\end{proof}

\end{document}